\documentclass[preprint,12pt]{elsarticle}

\usepackage{amssymb}
\usepackage{amsmath}
\usepackage{graphicx}
\usepackage{subcaption}
\usepackage{url} 
\usepackage{algorithm}
\usepackage{hyperref}
\usepackage{algpseudocode}
\usepackage{multirow}
\usepackage{xcolor}
\usepackage{amsthm}
\newtheorem{theorem}{Theorem}

\usepackage{xspace}

\newcommand{\MSESHAP}{MSE SHAP\xspace}

\journal{Arxiv}

\begin{document}

\begin{frontmatter}



\title{Explainability of Machine Learning Models under Missing Data}

\author[label3]{Tuan L. Vo$^*$}
\author[label1]{Thu Nguyen$^*$\footnote{$^*$denotes equal contribution}}
\author[label1]{Luis M. Lopez-Ramos}
\author[label2,label1]{Hugo L. Hammer}
\author[label1,label2]{Michael A. Riegler}
\author[label1,label2]{Pål Halvorsen}

\affiliation[label1]{organization={SimulaMet},
            city={Oslo},
            country={Norway}}

\affiliation[label2]{organization={Oslo Metropolitan University},
            city={Oslo},
            country={Norway}}
\affiliation[label3]{organization={LTCI, Télécom Paris, Institut Polytechnique de Paris},
            city={Paris},
            country={France}}



\begin{abstract}
Missing data is a prevalent issue that can significantly impair model performance and explainability. This paper briefly summarizes the development of the field of missing data with respect to  Explainable Artificial Intelligence and experimentally investigates the effects of various imputation methods on  SHAP (SHapley Additive exPlanations), a popular technique for explaining the output of complex machine learning models. 
Next, we compare different imputation strategies and assess their impact on feature importance and interaction as determined by Shapley values. Moreover, we also theoretically analyze the effects of missing values on Shapley values. Importantly, our findings reveal that the choice of imputation method can introduce biases that could lead to changes in the Shapley values, thereby affecting the explainability of the model. Moreover, we also show that a lower test prediction MSE (Mean Square Error) does not necessarily imply a lower MSE in Shapley values and vice versa. Also, while XGBoost (eXtreme Gradient Boosting) is a method that could handle missing data directly, using XGBoost directly on missing data can seriously affect explainability compared to imputing the data before training XGBoost. This study provides a comprehensive evaluation of imputation methods in the context of model explanations, offering practical guidance for selecting appropriate techniques based on dataset characteristics and analysis objectives. The results underscore the importance of considering imputation effects to ensure robust and reliable insights from machine learning models.
\end{abstract}




\end{frontmatter}

\section{Introduction}\label{sec-intro}

Missing data is a common issue that can significantly affect model performance and explainability. This problem can occur due to various reasons, such as data collection errors, privacy concerns, or intentional omission. Addressing missing data is crucial, and one effective approach is imputation, i.e., filling in the missing data points to create a more complete dataset for analysis, thereby improving the overall reliability of the results.

Various imputation methods exist, ranging from simple techniques like mean imputation to more sophisticated approaches such as multiple imputation by chained equations (MICE) and k-nearest neighbors (KNN). Each method has its strengths and weaknesses, influencing not only the performance of predictive models but also their explainability. For example, in \cite{pham2024correlation}, the authors studied the effects of missing data on the correlation heatmap and found out that the technique that computes a correlation matrix with the highest accuracy (in terms of root mean square error, RMSE) does not always produce correlation plots that closely resemble those derived from complete data. However, the effects of missing data on various aspects of explainable AI (XAI) have not yet been fully studied.

In terms of explainability, Shapley values, a concept derived from cooperative game theory, have gained prominence as a robust method for interpreting complex models \cite{molnar2022}. Shapley values attribute the contribution of each feature to the final prediction, offering insights into feature importance and interaction. However, the accuracy and reliability of Shapley values can be affected by the choice of the imputation method, as the imputed values can introduce biases or distortions. Despite the importance of this issue, little attention has been paid to the effect of imputation on the Shapley values of the downstream model. For example, in \cite{lipovetsky2013modeling}, the authors examine various imputation methods to choose the best one and then use Shapley values to explain the prediction of the downstream task. However, the effects of imputation methods on Shapley values have not been thoroughly examined.

This paper aims to explore the effects of various imputation methods on the calculation of Shapley values of the downstream classification/regression task. By comparing different imputation strategies, we seek to understand how they influence the explainability of machine learning models and the robustness of the insights derived from Shapley values. The study provides a comprehensive evaluation of imputation methods in the context of model interpretation, offering guidance for practitioners on selecting appropriate techniques for their specific datasets and objectives.

In summary, the contributions of this paper are: (i) we explore the effects of various imputation methods on Shapley values of the downstream regression/classification task; (ii) we demonstrate that the choice of imputation method can significantly impact Shapley values, influencing the explainability of the downstream machine learning models; (iii) we highlight that imputation methods achieving the highest prediction accuracy do not necessarily produce the most reliable feature importance values, suggesting the need for method selection aligned with dataset characteristics and analysis goals.

The rest of the paper is organized as follows. In Section \ref{sec-related}, we review related works on XAI and missing data to gain insights into the development of missing data in relation to XAI and highlight open issues for future consideration. Next, in Section \ref{sec-methods}, we describe the methods being examined in this work. Following that, we theoretically analyze the effects of missing data on Shapley values in Section \ref{sec-theo} and conduct and analyze the experiment results in Section \ref{sec-experiments}. Finally, the paper concludes with a discussion of our findings and suggestions for future work in Section \ref{sec-discuss} and Section \ref{sec-conclusion}.

\section{Related works}\label{sec-related}
In this section, we review various related works on missing data handling techniques and their explainability. 
\subsection{Explainable AI}
Explainable Artificial Intelligence (XAI) has garnered significant attention in recent years due to its critical role in ensuring transparency, trustworthiness, and accountability in AI systems. As AI continues to be integrated into various domains, the need for models that can not only perform well but also provide human-understandable explanations for their predictions has become paramount.

A significant development in XAI is the emergence of intrinsic explainability approaches, which involve designing models that are interpretable by nature. Decision trees \cite{quinlan1986induction}, rule-based models \cite{kuhn2013classification}, and linear models \cite{hastie2009elements} are classic examples of intrinsically interpretable models. More recently, models such as Generalized Additive Models (GAMs) \cite{hastie2017generalized} and attention mechanisms in neural networks (NNs) \cite{vaswani2017attention} have been explored for their ability to provide insights into the decision-making process. However, their explanation follows their specific explanation style, making it difficult to compare with other models that do not follow the same explanation schemes during model selection.

Recent advancements have also seen the integration of XAI with deep learning models. Techniques such as Deep SHAP \cite{lundberg2017unified} and Integrated Gradients \cite{sundararajan2017axiomatic} provide explanations for neural network predictions by attributing the output to the input features in a manner consistent with the model's internal workings. These methods bridge the gap between the high performance of deep learning models and the need for explainability.

One prominent approach in XAI is the use of post-hoc explanation methods, such as LIME (Local Interpretable Model-agnostic Explanations) \cite{ribeiro2016should} and SHAP (SHapley Additive exPlanations) \cite{lundberg2017unified}, offer explanations by approximating the behavior of complex models locally around a given prediction. LIME creates interpretable local linear surrogate models, while SHAP leverages Shapley values from cooperative game theory to fairly distribute the contribution of each feature to the prediction. Both methods have been widely adopted due to their model-agnostic nature, allowing them to be applied to any black-box model.

Overall, compared to other techniques, the Shapley value offers several distinct advantages over other Explainable AI (XAI) techniques. Unlike methods that provide only local explanations or are limited to specific model architectures, the Shapley value delivers a unified and theoretically grounded approach applicable to any machine learning model. It ensures a mathematical notion of fairness by distributing the model's output among the features based on their marginal contributions over different patterns of availability of the other features, offering an additive feature attribution that respects interaction effects. This fairness and consistency make Shapley values particularly valuable in high-stakes domains like medicine or finance \cite{fryer2021shapley, watson2022rational, heuillet2022collective, storaas2023using, storaas2022explainability,nguyen2023multimedia}, where understanding the relative importance and interplay of various features is crucial for trust and accountability. Additionally, Shapley values provide a global explainability framework, making it easier to understand the model's behavior across different scenarios, thus enhancing transparency and facilitating more informed decision-making. However, the treatment of missing data can significantly influence the robustness and reliability of XAI techniques, including Shapley values.

\subsection{Missing data imputation techniques}

The most common approach to handling missing values is to use imputation methods to fill in the gaps. Techniques such as matrix decomposition or matrix completion, including missing Gaussian processes (MGP) \cite{jafrasteh2023gaussian}, Polynomial Matrix Completion \cite{fan2020polynomial}, matrix completion via alternating least squares (ALS) \cite{hastie2015matrix}, and a congestion imputation model (CIM) based on joint matrix factorization \cite{jia2021missing}, allow for continuous data to be completed and subsequently analyzed using standard data analysis procedures. Additionally, regression or clustering-based methods, such as cumulative Bayesian ridge with less NaN (CBRL) and cumulative Bayesian ridge with high correlation (CBRC) \cite{m2020cbrl}, utilize Bayesian Ridge Regression and cluster-based local least square methods \cite{keerin2013improvement} for imputation. For large datasets, deep learning imputation techniques have gained popularity due to their performance \cite{choudhury2019imputation, lall2022midas, mohan2021graphical}. It is important to note that different imputation methods may produce varying values for each missing entry, making the modeling of uncertainty for each missing value crucial. Bayesian and multiple imputation techniques, such as Multiple Imputation with Denoising Autoencoders (MIDAS) \cite{lall2022midas}, Bayesian principal component analysis-based imputation \cite{audigier2016multiple} and multiple imputations using Deep Denoising Autoencoders \cite{gondara2017multiple}, are particularly useful in these scenarios. Furthermore, certain tree-based methods can naturally handle missing data through prediction, including missForest \cite{stekhoven2012missforest}, the decision tree missing-value imputation (DMI) algorithm \cite{rahman2013missing}, and decision trees and iterative fuzzy clustering (DIFC) \cite{nikfalazar2020missing}. Methods that can manage mixed data types have also emerged, such as Single Center Imputation from Multiple Chained Equation (SICE) \cite{khan2020sice}, 
and Clustering Mixed Numerical and Categorical Data with Missing Values (k-CMM) \cite{dinh2021clustering}.

Although various missing data methods have been developed, 
most of them do not have a built-in explanation mechanism for imputed values. However, the way a value is imputed can have a profound impact on the performance and explainability of downstream machine learning models. If one uses a prediction model such as a regression or classification model, then one can use Shapley values or some other XAI technique to predictions obtained from models trained after imputing training data. However, many times, one may want to use a more complicated imputation method that requires building regression or classification models and looping through the data multiple times to improve imputation accuracy. This can make it more challenging to explain the imputed values than imputing the data without using loops such as based on a regression model, for example. 

Sometimes, the explanation for imputed values of some imputation techniques can be easily achieved even though the original algorithm does not provide an explanation. K-nearest-neighbor imputation may offer some explanation of the imputed values based on the $K$ nearest neighbors. Next, since tree-based techniques are explainable, the explanation for tree-based imputation methods such as missForest \cite{stekhoven2012missforest} can be achieved by building a tree to explain the classification results. However, a potential ambiguity can be the prediction of one missing entry based on the imputed values of the other entries during the iteration process. Similar concerns apply to MICE imputation \cite{buuren2010mice}, a method for handling missing data through iterative predictions based on the relationships between observed and missing variables.

Attention to the explainability of imputation techniques under missing data has intensified in recent years. For example, \cite{vu2023conditional} introduces Distribution-based Imputation of Missing Values with Regularization (DIMV), an interpretable imputation method that characterizes the uncertainty of imputation based on the conditional distribution of a multivariate Gaussian model, with regularization and feature selection; or \cite{pham2024correlation} analyzes the effects of various imputation and direct parameter-estimation methods to the correlation plot. Next, a technique for obtaining local explanations when the point to be explained and/or the reference dataset is affected by data missingness is proposed in \cite{cinquini2023handling}. More specifically, they propose an extension of LIME that can be used to explain the output of models that handle missing data without imputation, such as the ones introduced in the ensuing section.

However, to our knowledge, so far, there has not been much work on the effects of missing data on Shapley values. In fact, while \cite{lyngdoh2022prediction} implements five imputation methods (elimination, mean, mode, KNN, MICE) to deal with missing data, they made predictions using 
various prediction models to conclude that XGBoost combined with KNN imputation gives the best results; they only examine SHAP for this combination.

\subsection{Direct missing data handling techniques without imputation}

Different from the imputation approaches, methods that directly handle missing data can have clearer implications in terms of explainability.
Specifically, Nguyen et al.~\cite{NGUYEN20211} introduced an efficient Parameter Estimation for Multiple Class Monotone Missing Data (EPEM) algorithm to estimate the maximum likelihood estimates (MLEs) for multiple class monotone missing data when the covariance matrices of all classes are assumed to be equal. Additionally, Direct Parameter Estimation for Randomly missing data (DPER) \cite{NGUYEN2022108082} addresses a more general case where missing data can occur in any feature by using pairs of features to estimate the entries in the mean and covariance matrices.
The implication to model explainability of using direct parameter estimation methods, like EPEM and DPER, includes improved transparency and explainability of the model's behavior, as these methods provide clear estimates of means and covariances, which can be directly examined and understood. In fact, \cite{pham2024correlation} analyzes the effects of various imputation and direct parameter estimation methods on the correlation plot, and finds that DPER is not only scalable but also better for this task than imputation techniques such as SOFT-IMPUTE \cite{mazumder2010spectral}, KNN-based imputation (KNNI), generative adversarial imputation network (GAIN) \cite{yoon2018gain}, graph imputer neural network (GINN) \cite{spinelli2020missing}, and missForest \cite{stekhoven2012missforest}.

Note that long short-term memory (LSTM) networks \cite{hochreiter1997long} can handle missing data directly. In fact, it is used in the work of Ghazi et al.~\cite{ghazi2018robust} to model disease progression while handling missing data in both the inputs and targets. The bi-clustering problem~\cite{castanho2024biclustering}, a statistical learning methodology that enables simultaneous partitioning of rows and columns of a rectangular data array into homogeneous subsets, is tackled by Li et al.~\cite{LI2020304} while handling missing data. Learning directly from the data may offer advantages in speed and reduced storage costs by eliminating the need for separate models for imputation and the target task. 

Using techniques that directly handle missing data can significantly enhance model explainability. By learning directly from the data, the model avoids the added complexity of managing separate imputation and prediction models, simplifying the overall structure and making it easier to understand. Additionally, it helps avoid potential biases that can be introduced by filling in missing values, as in imputation methods.

\subsection{Studies on the impact of imputation on model interpretation}

In the work of Shadbahr et al., it is observed that poor imputation quality can lead to misleading feature importance assessments in ML classifiers \cite{shadbahr2023impact}. More specifically, they design a simulated dataset where they expect the distribution of the Shapley values to be symmetric, as the clusters are separated by a fixed distance and are normally distributed within those clusters. Using the skewness, they measure the symmetry of the feature importance values for each feature for all models. According to their experiments, models fit poorly imputed data giving rise to misleading feature importance metrics, demonstrating the adverse impact of poor imputation quality on model explainability.

As highlighted by Payrovnaziri et al., changes in the underlying data due to imputation can alter model characteristics and feature importance, necessitating a careful examination of how imputed values influence predictions \cite{payrovnaziri2021assessing}. Their results show that the level of missingness and the imputation method used can have a significant impact on the interpretation of the models.
To assess model explainability, they focused on the stability of feature importance rankings across multiple imputations. Their findings revealed that the choice of imputation method significantly influences the explainability of predictive models, with MICE providing more stable feature importance rankings compared to mean imputation and k-NN.

Assessing the impact of imputation in the model explainability is particularly relevant in clinical settings, where accurate interpretation of model outputs is essential for decision-making. The effectiveness of interpretable ML models under missing data was also studied in \cite{stempfle2024expert} from a subjective point of view (surveying clinicians' perception).

The impact of missing data imputation in the performance and explainability of prediction rule ensembles (PREs) was studied in \cite{schroeder2024interpretable}. A simulation study is conducted where a PRE model is trained using data imputed using several imputation schemes, and the model size and performance are evaluated, observing a trade-off between performance and model complexity.






\section{Methods}\label{sec-methods}

In this paper, we will also examine the effects of various imputation techniques on Shapley values. Therefore, in this section, we will briefly summarize the basics of Shapley values, the types of plots, and the imputation techniques that will be used to examine Shapley values in the experiments.
\subsection{Shapley values}
Shapley values, originally derived from cooperative game theory~\cite{shapley:book1952}, have been adapted as a powerful tool for interpreting machine learning models~\cite{lundberg2017unified}. They provide a systematic way to attribute the contribution of each feature to a model's prediction, in accordance with a mathematical notion of fairness in the distribution based on the interaction of features.
Consider a machine learning model $v$ that takes an input vector $\mathbf{x} = (x_1, x_2, ..., x_p)$ and produces a prediction $v(\mathbf{x})$. The Shapley value for the value $x_i$ of corresponding feature $\boldsymbol{f}_i$ quantifies its contribution to the prediction $v(\mathbf{x})$.
Formally, let $P = \{1,2,\dots,p\}$ be the complete index set of features and $S$ a coalition of feature indices (a subset of $P \setminus \{i\}$). For $i=1,2,\dots,p$, the Shapley value $\phi_i$ for the feature value $x_i$ is defined as
\begin{equation}
    \phi_i = \sum_{S \subseteq P \setminus \{i\}} \frac{|S|! (p - |S| -1)!}{p!} \left[v(S \cup \{i\}) - v(S)\right],
\end{equation}
where $|S|$ is the number of elements in subset $S$, and $v(S \cup \{i\})$, $v(S)$ denote the model prediction using feature indices in $S$ plus $i^{th}$ feature and only in $S$, respectively.
In the context of classification problems, the model outputs a probability distribution over classes. The Shapley values can be calculated for each class's probability, providing insights into how each feature influences the likelihood of each class. Specifically, calculating Shapley values in classification problems can be broken down into Algorithm~\ref{alg:shapleyvalues}.
\begin{algorithm}
\caption{\textbf{Shapley value calculation}}\label{alg:shapleyvalues}
    \hspace*{\algorithmicindent} \textbf{Input:} a classification model $v$ that predicts the probabilities for each class based on a sample $\mathbf{x} = (x_1, x_2, \ldots, x_p)$ and $P = \{1, 2, \ldots, p\}$ represents the index set of all features.\\
    \hspace*{\algorithmicindent} \textbf{Output:} The Shapley value $\phi_i$ for each feature value $x_i$, for $i=1,2,\dots,p$.
\begin{algorithmic}[1]
    \For{S $\subset$ P}
            \State $v(S) \leftarrow $ trained classification model $v$ with index in $S$.
            \EndFor
    \For{i $\in$ P}
        \For{S $\subset P \setminus \{i\}$}
            \State $\text{Marginal Contribution}_{(i)} \leftarrow v(S \cup \{i\}) - v(S).$
            \State $\text{Weight}(S) \leftarrow \frac{|S|! (p - |S| -1)!}{p!}.$
            \EndFor
        \State $\phi_i \leftarrow \sum_{S \subseteq P \setminus \{i\}} \text{Weight}(S) \times \text{Marginal Contribution}_{(i)}.$
    \EndFor
\end{algorithmic}
\end{algorithm}

The types of plot that will be examined in this paper include:
\begin{itemize}
    \item A global feature importance plot allows visualizing the impact of individual features within a predictive model. More specifically, it highlights the features that contribute the most to the model's predictions, enabling a deeper understanding of the model's behavior. Here, the global importance of each feature is determined by calculating the mean absolute value of Shapley values of that feature across all the given samples. 
\item A beeswarm plot is a useful visualization tool to depict the distribution and influence of individual feature contributions on the model's predictions, particularly when using Shapley values. It provides a detailed view of how each feature affects the output of a model across the entire dataset. 
By displaying each data point as a dot and arranging these dots to show the distribution of Shapley values for each feature, the beeswarm plot offers a comprehensive overview of feature importance. It allows for the identification of patterns and outliers, helping to understand the behavior of the model with respect to individual features. 

\end{itemize}

\subsection{Imputation techniques}
In this section, we briefly summarize the missing data handling methods that we will examine for the effects on Shapley values. The methods being investigated consist of a method that can directly learn from missing data, such as XGBoost, to a simple imputation method as Mean Imputation, as well as the widely used or recently developed imputation techniques, such as MICE, DIMV, missForest, and SOFT-IMPUTE. The details of the methods are as follows: 
\begin{itemize}
    \item \textbf{XGBoost} (Extreme Gradient Boosting) \cite{chen2016xgboost} is a powerful and efficient algorithm that belongs to the family of gradient boosting techniques. It builds an ensemble of decision trees, where each tree corrects errors made by the previous ones, enhancing predictive accuracy. XGBoost stands out for its speed and performance, employing advanced features like tree pruning, regularization, and parallel processing, which help in reducing overfitting and handling large-scale data. XGBoost can handle missing data directly both during training and inference.

    \item \textbf{Mean Imputation (MI)} is a basic technique consisting of replacing the missing values with the mean (sample average) of the available data for that particular variable. This method is straightforward and easy to implement. 
    
    \item \textbf{Multiple Imputation by Chained Equations (MICE)}~\cite{buuren2010mice} is a technique for imputing missing data in datasets through regression. It considers the variability of missing values by crafting various imputed datasets. This technique entails sequentially modeling each incomplete variable based on other data variables, producing numerous imputations for each lacking value. The iteration ceases when the divergence among different entries falls below a set threshold. Through this approach, MICE preserves inter-variable connections and yields a stronger and more precise approximation of the missing values.

    \item \textbf{Conditional Distribution-based Imputation of Missing Values with Regularization (DIMV)}~\cite{vu2023conditional}, is an innovative algorithm designed to handle missing data. This method works by determining the conditional distribution of a feature with missing entries, leveraging information from fully observed features. Even though DIMV is based on the assumption that the data follows a multivariate normal distribution, it is robust in the sense that it still works well when this assumption is not met. DIMV can also explain the contribution of each feature to the imputation of a feature with a missing value in a regression coefficient-like manner. Moreover, the technique is robust to multicollinearity due to $L_2$ norm regularization.

    \item \textbf{MissForest}~\cite{stekhoven2012missforest} is an imputation method that leverages the predictive power of random forests to estimate missing values in datasets. In particular, it operates by iteratively predicting the missing values for each variable using the observed values and previously imputed values. This process continues until the imputations converge, minimizing the error between successive iterations. Moreover, missForest can handle both continuous and categorical data and capture complex interactions and non-linear relationships between variables, leading to more accurate imputations that make it versatile and suitable for various types of datasets.

    \item \textbf{SOFT-IMPUTE}~\cite{mazumder2010spectral} is an imputation technique based on matrix completion and low-rank approximation, particularly useful for handling missing data in large and sparse datasets. This method operates by iteratively replacing the missing entries with values that minimize the reconstruction error of the data matrix, using singular value decomposition (SVD). Furthermore, SOFT-IMPUTE alternates between imputing missing values and performing a soft-thresholding operation on the singular values of the data matrix. This process promotes a low-rank structure in the resulting matrix, effectively capturing the underlying low-rank structure of the data while filling in the missing entries. 

    \item \textbf{Generative Adversarial Imputation Nets (GAIN)}~\cite{yoon2018gain} 
    is a deep learning method for handling missing data by leveraging the Generative Adversarial Networks (GAN) framework. In GAIN, the generator (G) is trained to impute missing values in a data vector by conditioning the observed values, producing a completed vector that closely resembles the true data distribution. The discriminator (D), on the other hand, is tasked with distinguishing between observed and imputed components within this completed vector. 
\end{itemize}
\section{Theoretical analysis}\label{sec-theo}
In this section, we consider a simple linear regression model to analyze the effects of missing data (a missing-completely-at-random mechanism) in training and test datasets on global feature importance, measured by mean absolute Shapley values.

Assume we have an original training dataset, denoted as $\mathcal{D} = (\mathbf{x}, \mathbf{y})$, consisting of $n$ samples, where $\mathbf{x}$ is a univariate input and $\mathbf{y}$ is the target variable. The corresponding dataset with missing values from $\mathcal{D}$ is denoted as $\mathcal{D}^* = (\mathbf{x^*}, \mathbf{y})$. Here, only $\mathbf{x}^*$ has missing values and we denote $\text{Obs}(\mathbf{x}^*)$ as the index set of observed samples in $\mathbf{x}^*$, i.e., $\text{Obs}(\mathbf{x}^*) = \{i\in \{1, 2, \dots ,n\}: x_i \text{ is observed}\}$. Moreover, we denote $\mathcal{D}'=(\mathbf{x}',\mathbf{y})$ as imputed data by using mean imputation. It means that
\begin{equation}
        x'_i = \begin{cases}
    x_i, & \text{ if } i \in \text{Obs}(\mathbf{x}^*)\\
    \mathbb{E}[\mathbf{x}^*], & \text{ if } i \notin \text{Obs}(\mathbf{x}^*)
    \end{cases}, \text{ for } i=1,2,\dots,n.
\end{equation}
Now, suppose that the corresponding linear regression models on $\mathcal{D}$ and $\mathcal{D'}$ are
\begin{align*}
    &y= \widehat{f}(x)  = \beta_0 +\beta_1 x,\\
    &y' = \widehat{f^\prime}(x') = \beta'_0 + \beta'_1x',
\end{align*}
where
 \begin{align*}
     &\beta_1 = \frac{\text{Cov}(\mathbf{x},\mathbf{y})}{\text{Var}(\mathbf{x})}, \beta_0 = \mathbb{E}[\textbf{y}] - \beta_1\mathbb{E}[\textbf{x}], \beta'_1 = \frac{\text{Cov}(\mathbf{x'},\mathbf{y})}{\text{Var}(\mathbf{x'})}, \text{ and } \beta'_0 = \mathbb{E}[\textbf{y}] - \beta'_1\mathbb{E}[\textbf{x}'].
\end{align*}

We consider a test dataset $\mathbf{z}$, consisting of $m$ samples. It should be recalled that the test data $\mathbf{z}$ could contain missing values (denoted by $\mathbf{z}^*$) and the imputed version $\mathbf{z}'$ is presented as
\begin{equation}\label{z_imputed}
    z'_j = \begin{cases}
    z_j, & \text{ if } j \in \text{Obs}(\mathbf{z}^*)  \\
    \mathbb{E}[\mathbf{x}^*], & \text{ if } j \notin \text{Obs}(\mathbf{z}^*)
    \end{cases}, \text{ for } j=1,2,\dots,m.
\end{equation}
Then, we want to illustrate a global feature importance of $\mathbf{x}$ and $\mathbf{x}'$ on $\mathbf{z}$ and $\mathbf{z}'$ respectively by using the mean absolute Shapley values, i.e.,
\begin{align}
    |\phi(\mathbf{z})| = \frac{1}{m}\left(\sum_{j=1}^{m}|\phi(z_j)|\right),\\
   |\phi'(\mathbf{z}')| = \frac{1}{m}\left(\sum_{j=1}^{m}|\phi'(z'_j)|\right),
\end{align}
where $|\phi(\mathbf{z})|$ and $|\phi'(\mathbf{z'})|$ represent the global feature importance on $\mathbf{z}$ and $\mathbf{z'}$, respectively; $\phi(z_j)$ and $\phi'(z'_j)$ denote the Shapley values for the predictions of samples $z_j$ and $z'_j$.
The following theorem shows that the global feature importance on $\mathbf{z'}$ can be simplified.
\begin{theorem}
The global feature importance of $\mathbf{x}'$ on $\mathbf{z'}$ can be simplified to
\begin{equation}
     |\phi'(\mathbf{z}')| = \frac{1}{m}\left(\sum_{j\in\text{Obs}(\mathbf{z}^*)}|\phi'(z'_j)|\right).
\end{equation}
\end{theorem}

\begin{proof}
In~\cite{lundberg2017unified}, the \textit{Corollary 1}, the Shapley value for the prediction on sample $z'_j$ for $j=1,2,\dots,m$ with the corresponding model on $\mathcal{D}'$ are defined as

\begin{equation}
    \phi'(z'_j) =(z'_j - \mathbb{E}[\mathbf{x'}]) \beta'_1.
\end{equation}
By using the mean imputation~\eqref{z_imputed}, if $z_j$ is missing then $z'_j = \mathbb{E}[\mathbf{x}^*]$. Thus,
\begin{equation}
    \phi'(z'_j) =(\mathbb{E}[\mathbf{x}^*] - \mathbb{E}[\mathbf{x'}]) \beta'_1,\text{ for } j \notin \text{Obs}(\mathbf{z}^*).
\end{equation}
Moreover, it is worth noting that $\mathbf{x}'$ is the version of mean imputation of $\mathbf{x}^*$, therefore $\mathbb{E}[\mathbf{x}^*]=\mathbb{E}[\mathbf{x}']$. This implies
\begin{equation}
    \phi'(z'_j) = 0, \text{ for } j \notin \text{Obs}(\mathbf{z}^*).
\end{equation}
The global feature importance of $\mathbf{x}'$ on $\mathbf{z}'$ is
\begin{align}
   |\phi'(\mathbf{z}')| &= \frac{1}{m}\left(\sum_{j=1}^{m}|\phi'(z'_j)|\right)\nonumber \\
   &=\frac{1}{m}\left(\sum_{j\in \text{Obs}(\mathbf{z}^*)}|\phi'(z'_j)| + \sum_{j \notin \text{Obs}(\mathbf{z}^*)}|\phi'(z'_j)|\right)\nonumber\\
    &=\frac{1}{m}\left(\sum_{j\in \text{Obs}(\mathbf{z}^*)}|\phi'(z'_j)|\right).
\end{align}
\end{proof}
Based on the theorem, we have the following analysis. First, when the missing rate of $\mathbf{z}^*$ increases, the number of samples that have non-zero contributions in $ |\phi'(\mathbf{z}')|$ is decreased. We can find that $ |\phi'(\mathbf{z}')|$ would be small and less than $|\phi(\mathbf{z})|$. However, one important thing is that we cannot conclude that $|\phi'(\mathbf{z}')|$ is always less than $ |\phi(\mathbf{z})|$. The reason is that $|\phi'(z'_j)|$ could be greater than $|\phi(z_j)|$ for some $j\in \text{Obs}(\mathbf{z}^*)$.
To state it clearly, expanding the formulas $\phi(z_j)$ and $\phi'(z'_j)$ gives
\begin{align}
    &|\phi(z_j)| = \left|(z_j - \mathbb{E}[\mathbf{x}])\beta_1 \right|= \left|z_j - \mathbb{E}[\mathbf{x}]\right|.\frac{|\text{Cov}(\mathbf{x},\mathbf{y})|}{\text{Var}(\mathbf{x})}, \\
    &|\phi'(z'_j)| = \left|(z'_j - \mathbb{E}[\mathbf{x'}] \beta'_1 \right|= \left|z'_j - \mathbb{E}[\mathbf{x'}]\right|.\frac{|\text{Cov}(\mathbf{x'},\mathbf{y})|}{\text{Var}(\mathbf{x'})}.
\end{align}
Note that if $j\in \text{Obs}(\mathbf{z}^*)$, then $z'_j = z_j$. Given that the missing mechanism is missing-completely-at-random (MCAR), we have $\mathbb{E}[\mathbf{x}] = \mathbb{E}[\mathbf{x}']$, so the difference between $|\phi(z_j)|$ and $|\phi'(z'_j)|$ comes from the remaining terms $\frac{|\text{Cov}(\mathbf{x},\mathbf{y})|}{\text{Var}(\mathbf{x})}$ and $\frac{|\text{Cov}(\mathbf{x'},\mathbf{y})|}{\text{Var}(\mathbf{x'})}$. Additionally, these terms depend on the missing rate in training data $\mathbf{x}$ and $\mathbf{x}'$. In particular, because $\mathbf{x'}$ is the imputed version of $\mathbf{x}$ with mean imputation then $\text{Var}(\mathbf{x'}) \le \text{Var}(\mathbf{x})$. It means when the missing rate in training data increases, the term $\frac{|\text{Cov}(\mathbf{x'},\mathbf{y})|}{\text{Var}(\mathbf{x'})}$ could be greater than $\frac{|\text{Cov}(\mathbf{x},\mathbf{y})|}{\text{Var}(\mathbf{x})}$. To illustrate this, consider that, when the $i$-th sample becomes missing, the impact on the covariance is
\begin{equation}
    \text{Cov}(\mathbf{x'},\mathbf{y}) - \text{Cov}(\mathbf{x},\mathbf{y}) = \frac{-1}{N} (y_i - \mathbb{E}[\mathbf{y}]) \cdot (x_i - \mathbb{E}[\bf x']),
\end{equation}
which can be positive or negative and can contribute to an increase or decrease in the absolute value of the covariance. The derivation of the expression above is provided in Appendix \ref{appendix-derivation-cov-diff}. \\
Thus, $|\phi'(z'_j)|$ could be greater than $|\phi(z_j)|$ for some $j\in \text{Obs}(\mathbf{z}^*)$. In summary, there is a different trend in global feature importance changes, this trend depends on the missing rates in training and test datasets. A high missing rate in the test dataset can decrease global feature importance, while a high missing rate in the training dataset can increase it.

Second, the distribution of the Shapley values on $\mathbf{z}$ and $\mathbf{z'}$ should have a zero mean. The evidence for that is 
\begin{align}
    \mathbb{E}[\phi(\mathbf{z})] &=\mathbb{E}[(\textbf{z}- \mathbb{E}[\mathbf{x}])\beta_1] =  (\mathbb{E}[\mathbf{z}] - \mathbb{E}[\mathbf{x}])\beta_1, \\
    \mathbb{E}[\phi'(\mathbf{z'})] &= \mathbb{E}[(\textbf{z}'- \mathbb{E}[\mathbf{x}])\beta'_1] =(\mathbb{E}[\mathbf{z}'] - \mathbb{E}[\mathbf{x}'])\beta'_1.
\end{align}
In practice, $\mathbf{x}$ and $\mathbf{z}$ come from one dataset, so it is possible to have $\mathbb{E}[\mathbf{z}] \approx \mathbb{E}[\mathbf{x}]$ and $\mathbb{E}[\mathbf{z}'] \approx \mathbb{E}[\mathbf{x}']$. Therefore, $\mathbb{E}[\phi(\mathbf{z})] \approx  \mathbb{E}[\phi'(\mathbf{z'})] \approx 0$. This is why we see the symmetry in beeswarm plots of the original data and mean imputation.

\clearpage
\section{Experiments}\label{sec-experiments}
\subsection{Experiment settings}\label{sec-experiment settings}
In this study, we conduct experiments to study the effects of missing values on Shapley values on regression and classification using four data sets from the Machine Learning Database Repository at the University of California, Irvine\footnote{http://archive.ics.uci.edu/ml}. Detailed descriptions of the datasets are provided in Table~\ref{tab:data}. 
Specifically, three scenarios were considered in our experiments: (1) Linear Regression on Original data (denoted as LRO and considered the ground truth), (2) XGBoost without imputation, and (3) linear regression on imputed data. Here, in the third scenario, we tested six imputation methods: Mean Imputation (MI), Multiple Imputation by Chained Equations (MICE)~\cite{buuren2010mice}, conditional Distribution-based Imputation of Missing Values with 
Regularization (DIMV)~\cite{vu2023conditional}, missForest~\cite{stekhoven2012missforest}, SOFT-IMPUTE~\cite{mazumder2010spectral}, and Generative Adversarial Imputation Nets (GAIN)~\cite{yoon2018gain}. These methods are implemented with default settings using the \textit{fancyimpute}~\footnote{\url{https://github.com/iskandr/fancyimpute}}, \textit{scikit-learn}~\cite{scikit-learn}, and \textit{DIMV Imputation} packages~\footnote{\url{https://github.com/maianhpuco/DIMVImputation}}. For regression tasks, the prediction model used is XGBoost regressor, while for classification, it is XGBoost Classifier. The code for the experiment is made available on \href{https://github.com/simulamet-host/SHAP/tree/main}{Github}.

We use global feature importance and beeswarm plots to analyze the effects of missing rate and handling missing strategy on Shapley values. Furthermore, the Mean Square Error (MSE) can also be utilized for that evaluation. For regression tasks, the MSE is calculated between the true values and the predicted values in the test set labels. For classification tasks, the MSE is computed between the imputed values and the original values (before simulating missing data) in the test set inputs. Similarly, MSESHAP is defined as the MSE between the Shapley values derived from a model fitted on the imputed data and the Shapley values derived from a model fitted on the original data. In the resulting tables and figures, MI denotes Mean Imputation, and SOFT denotes SOFT-IMPUTE. Values highlighted in boldface indicate the best performance. 


\begin{table}[h]
\centering
\begin{tabular}{cccc}
\hline
Tasks & Datasets & \#Features & \#Samples\\
\hline
\multirow{2}{*}{Regression} & California & 9 & 20,640 \\
& Diabetes & 8 & 768 \\
\hline
\multirow{2}{*}{Classification} &MNIST & (28,28) & 60,000\\
& Glass & 9 & 214\\
\hline
\end{tabular}
\caption{Descriptions of datasets used in the experiments}
\label{tab:data}
\end{table}

For the missing data, we simulated missing-completely-at-random (MCAR) data from the original dataset by randomly assigning missing values with missing rates ranging from 0.2 to 0.8. The missing rate \(r\) is defined as the ratio of the number of missing values to the total number of values in the dataset.

The experiments were run on a Windows-based machine equipped with an AMD Ryzen 7 3700X 8-Core Processor running at 3.59 GHz, 16 GB of RAM, and a 64-bit operating system. Each experiment was repeated ten times, and the results were averaged. The code for the experiments is available at \url{https://github.com/simulamet-host/SHAP}

\subsection{Global feature importance plot analysis}
In this section, we analyze the effects of missing rate and handling missing strategy on global feature importance. We focus on examining the results of the California dataset because the key insights in this case can also be seen in the Diabetes dataset. We present the Global Feature Importance plots across various missing rates for the California dataset in Figures~\ref{fig:california-r2-bar},~\ref{fig:california-r4-bar},~\ref{fig:california-r6-bar} and~\ref{fig:california-r8-bar}, while the plots for the Diabetes dataset are included in \ref{appendix-importance-plots}. In addition, it is important to note that we compare eight processing workflows: LRO (Linear Regression on Original data, considered the ground truth), XGBoost, MI, MICE, DIMV, missForest, SOFT-IMPUTE, and GAIN. In the case of XGBoost, no previous imputation is necessary as that technique can deal with missing data by default. For the last six imputation methods, we simplify the notation by using the name of the imputation method to represent the combination of that method followed by linear regression.

To begin with, regardless of the missing rates, the global feature importance consistently highlights three main predictors of data: Latitude, Longitude, and Median Income (MedInc). However, their orders seem to change compared to the plot on the original data (LRO). For instance, at a missing rate of 0.2, LRO highlights Latitude the most, then Longitude and MedInc, while all the remaining methods highlight MedInc the most.

Interestingly, across all missing rates, when comparing the global feature importance of three features (MedInc, Latitude, and Longitude) of LRO to the remaining methods, then XGBoost produces the most different plots. Specifically, XGBoost exhibits high global feature importance for those three features, with a range from $0.28$ to $0.69$. Meanwhile, for the six imputation methods, those values show a distribution similar to the results on the original data but at a lower range, from $0.05$ to $0.15$.

Additionally, as the missing rate increases from 0.2 to 0.8, there is a decrease in the global feature importance of the three mentioned features for the six comparison imputation methods, while XGBoost shows an upward trend. For example, for the MedInc feature, its global feature importance with the six imputation methods declines from around $0.15$ to $0.13$, whereas with XGBoost, it increases from $0.33$ to $0.69$.

\begin{figure}[!t]
    \centering
    \begin{subfigure}{0.45\textwidth}
        \centering
        \includegraphics[width=\textwidth]{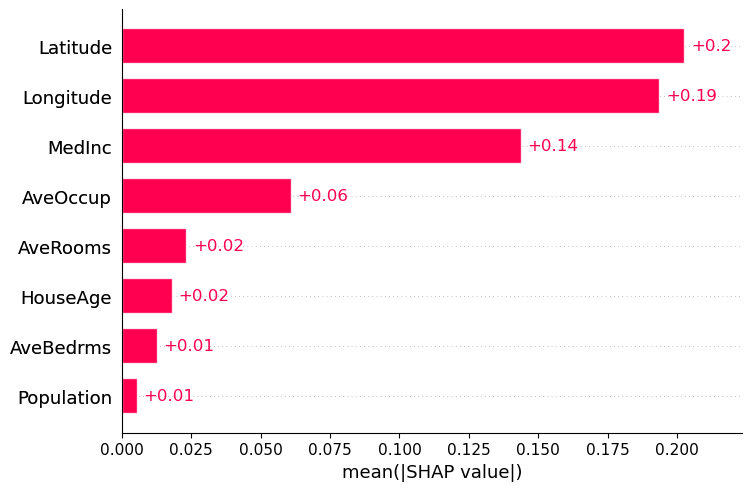}
        \caption{Original (LRO)}
        \label{fig:california-ori02}
    \end{subfigure}
    \begin{subfigure}{0.45\textwidth}
        \centering
        \includegraphics[width=\textwidth]{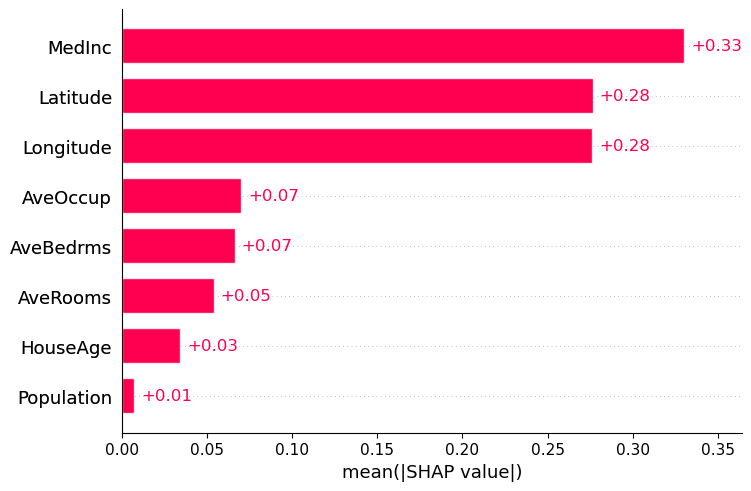}
        \caption{XGBoost without imputation}
        \label{fig:california-xm02}
    \end{subfigure}
    
    \begin{subfigure}{0.45\textwidth}
        \centering
        \includegraphics[width=\textwidth]{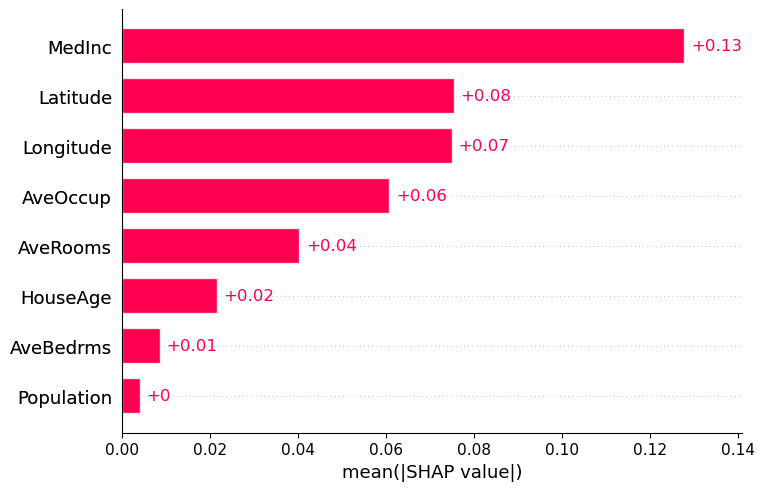}
        \caption{Mean Imputation}
        \label{fig:california-mi02}
    \end{subfigure}    
    \begin{subfigure}{0.45\textwidth}
        \centering
        \includegraphics[width=\textwidth]{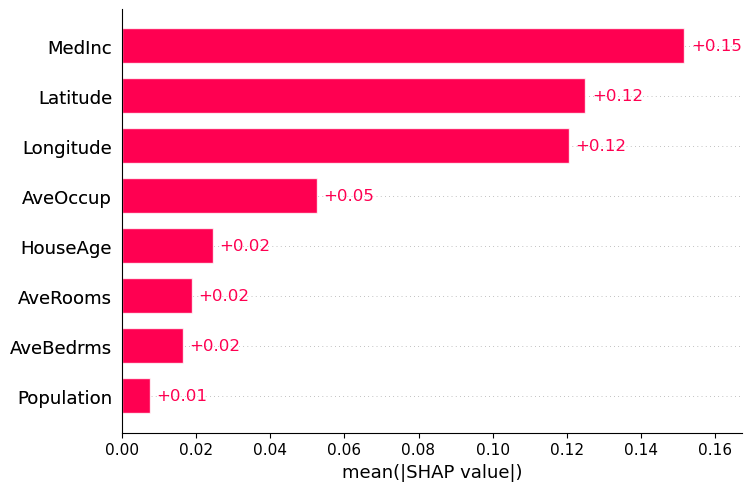}
        \caption{MICE}
        \label{fig:california-mice02}
    \end{subfigure}
    
    \begin{subfigure}{0.45\textwidth}
        \centering
        \includegraphics[width=\textwidth]{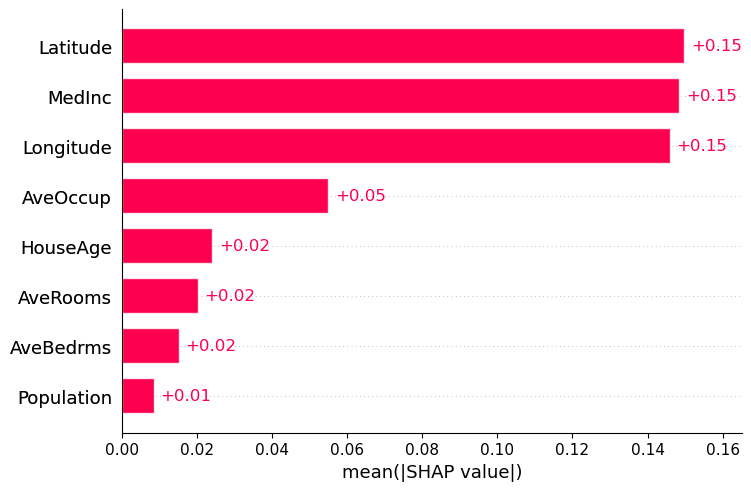}
        \caption{DIMV}
        \label{fig:california-dimv02}
    \end{subfigure}
    \begin{subfigure}{0.45\textwidth}
        \centering
        \includegraphics[width=\textwidth]{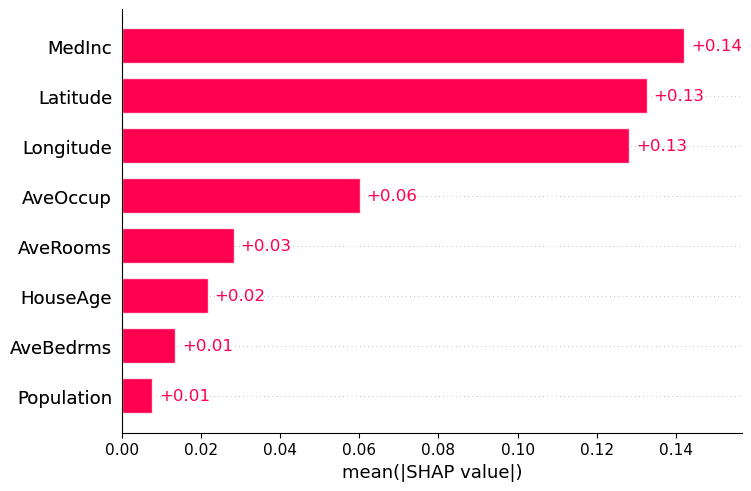}
        \caption{missForest}
        \label{fig:california-mf02}
    \end{subfigure}
    
    \begin{subfigure}{0.45\textwidth}
        \centering
        \includegraphics[width=\textwidth]{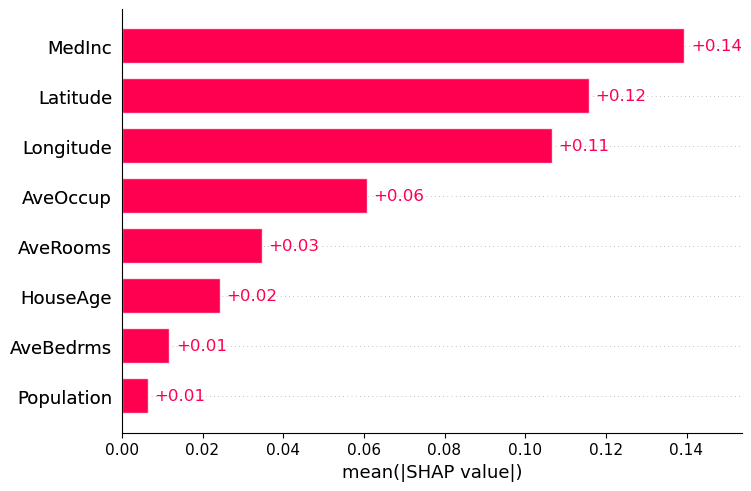}
        \caption{SOFT-IMPUTE}
        \label{fig:california-soft02}
    \end{subfigure}
    \begin{subfigure}{0.45\textwidth}
        \centering
        \includegraphics[width=\textwidth]{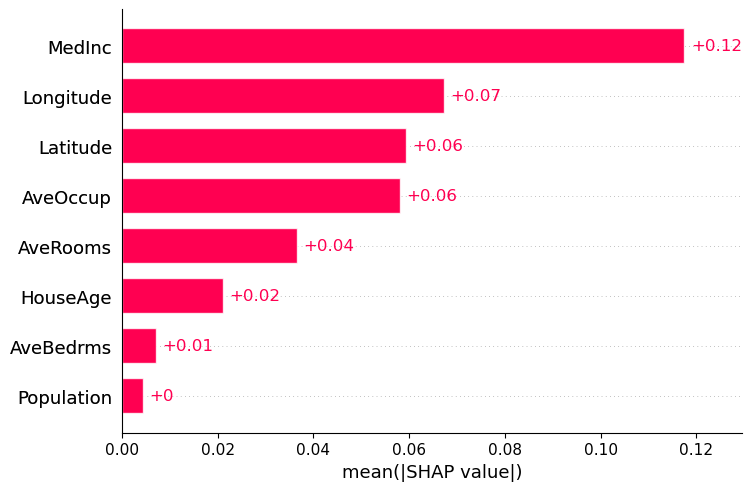}
        \caption{GAIN}
        \label{fig:california-gain02}
    \end{subfigure}
    
    \caption{Global feature importance plot on the California dataset with the missing rate $r=0.2$}
    \label{fig:california-r2-bar}
\end{figure}

\begin{figure}[]
    \centering
    \begin{subfigure}{0.45\textwidth}
        \centering
        \includegraphics[width=\textwidth]{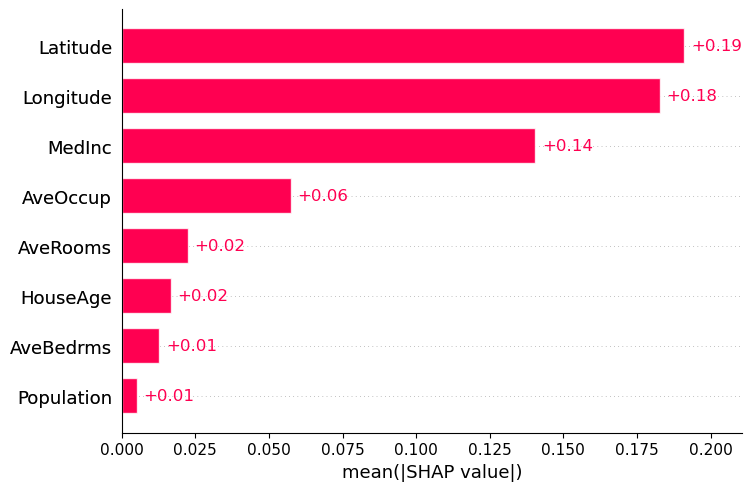}
        \caption{Original (LRO)}
        \label{fig:california-ori04}
    \end{subfigure}
    \begin{subfigure}{0.45\textwidth}
        \centering
        \includegraphics[width=\textwidth]{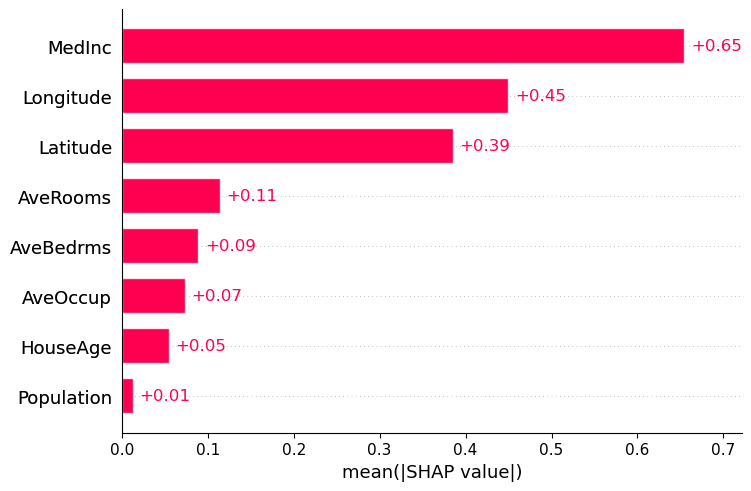}
        \caption{XGBoost without imputation}
        \label{fig:california-xm04}
    \end{subfigure}
    
    \begin{subfigure}{0.45\textwidth}
        \centering
        \includegraphics[width=\textwidth]{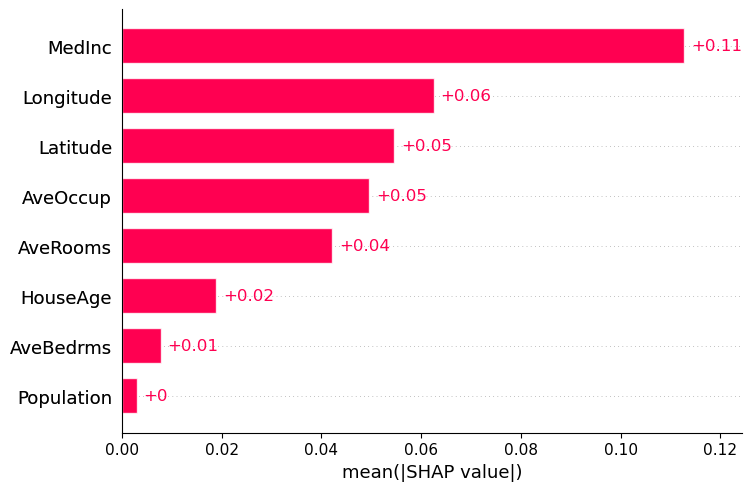}
        \caption{Mean Imputation}
        \label{fig:california-mi04}
    \end{subfigure}    
    \begin{subfigure}{0.45\textwidth}
        \centering
        \includegraphics[width=\textwidth]{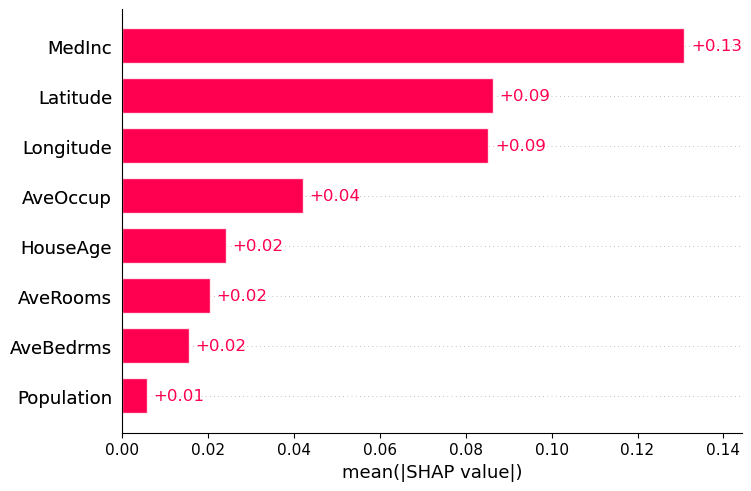}
        \caption{MICE}
        \label{fig:california-mice04}
    \end{subfigure}
    
    \begin{subfigure}{0.45\textwidth}
        \centering
        \includegraphics[width=\textwidth]{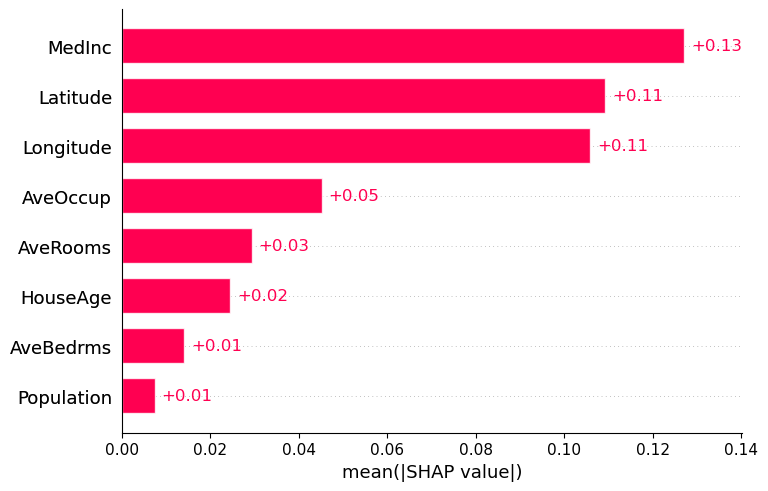}
        \caption{DIMV}
        \label{fig:california-dimv04}
    \end{subfigure}
    \begin{subfigure}{0.45\textwidth}
        \centering
        \includegraphics[width=\textwidth]{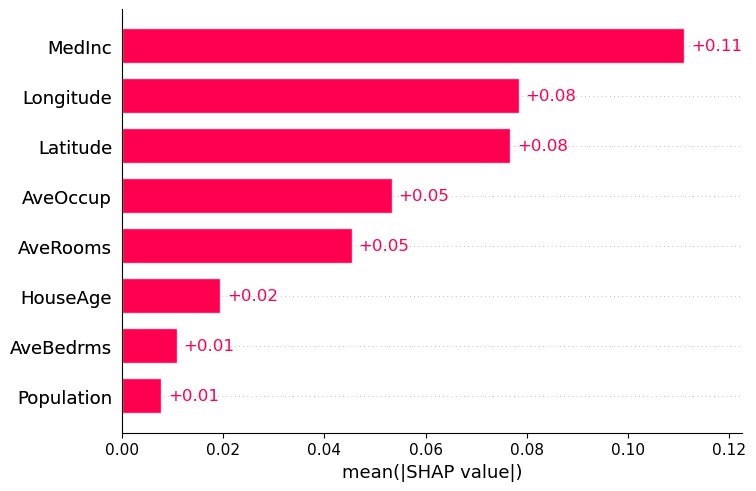}
        \caption{missForest}
        \label{fig:california-mf04}
    \end{subfigure}
    
    \begin{subfigure}{0.45\textwidth}
        \centering
        \includegraphics[width=\textwidth]{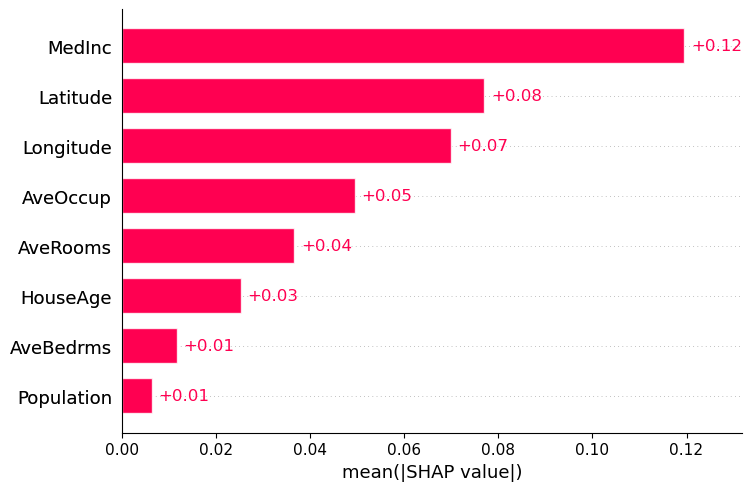}
        \caption{SOFT-IMPUTE}
        \label{fig:california-soft04}
    \end{subfigure}
    \begin{subfigure}{0.45\textwidth}
        \centering
        \includegraphics[width=\textwidth]{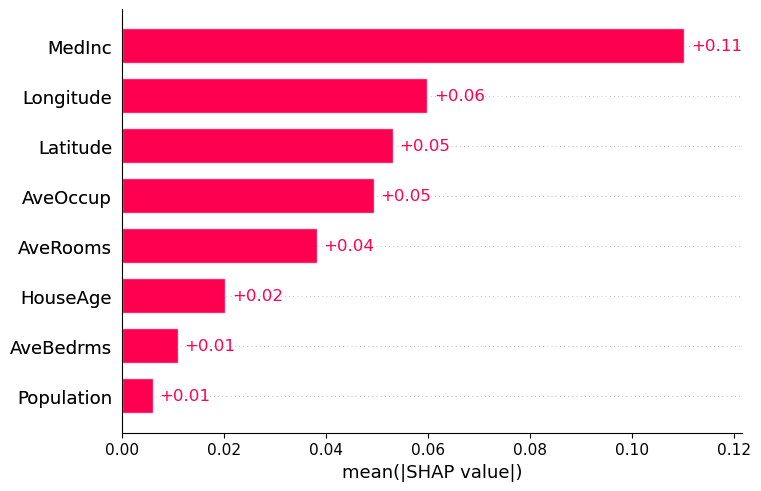}
        \caption{GAIN}
        \label{fig:california-gain04}
    \end{subfigure}
    
    \caption{Global feature importance plot on the California dataset with the missing rate $r=0.4$}
    \label{fig:california-r4-bar}
\end{figure}

\begin{figure}[]
    \centering
    \begin{subfigure}{0.45\textwidth}
        \centering
        \includegraphics[width=\textwidth]{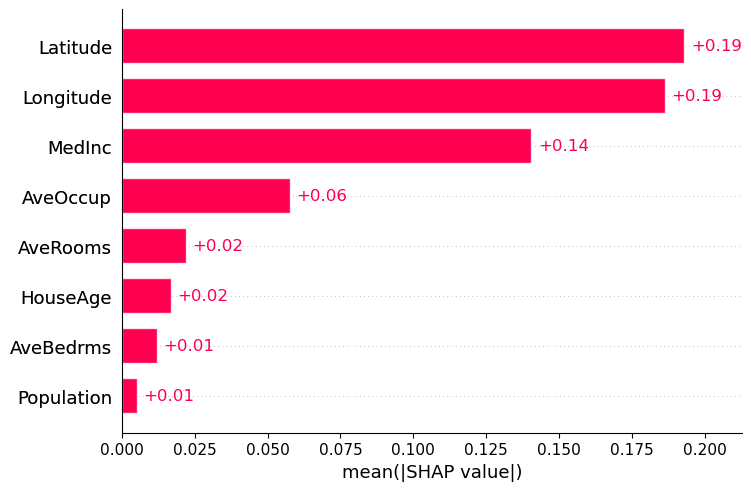}
        \caption{Original (LRO)}
        \label{fig:california-ori06}
    \end{subfigure}
    \begin{subfigure}{0.45\textwidth}
        \centering
        \includegraphics[width=\textwidth]{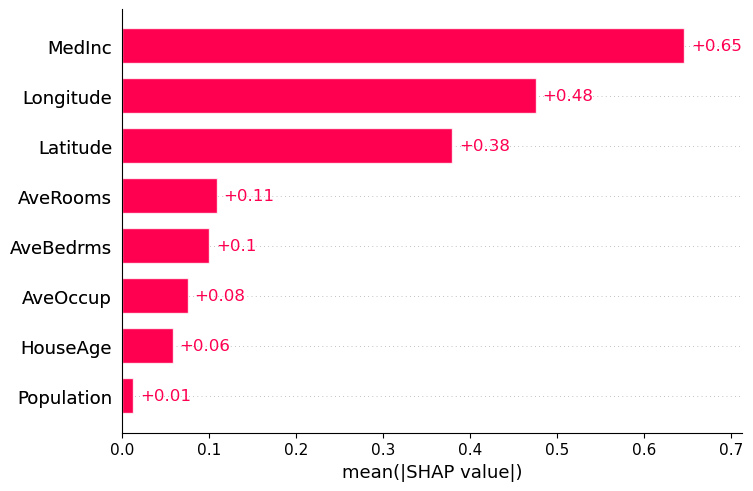}
        \caption{XGBoost without imputation}
        \label{fig:california-xm06}
    \end{subfigure}
    
    \begin{subfigure}{0.45\textwidth}
        \centering
        \includegraphics[width=\textwidth]{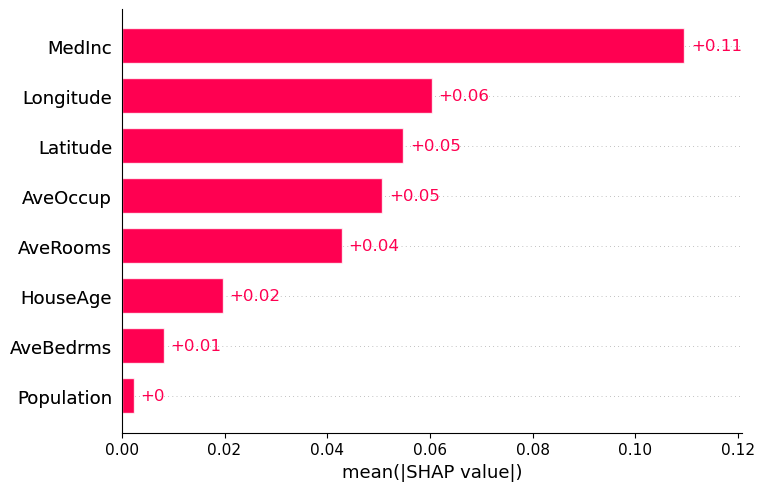}
        \caption{Mean Imputation}
        \label{fig:california-mi06}
    \end{subfigure}    
    \begin{subfigure}{0.45\textwidth}
        \centering
        \includegraphics[width=\textwidth]{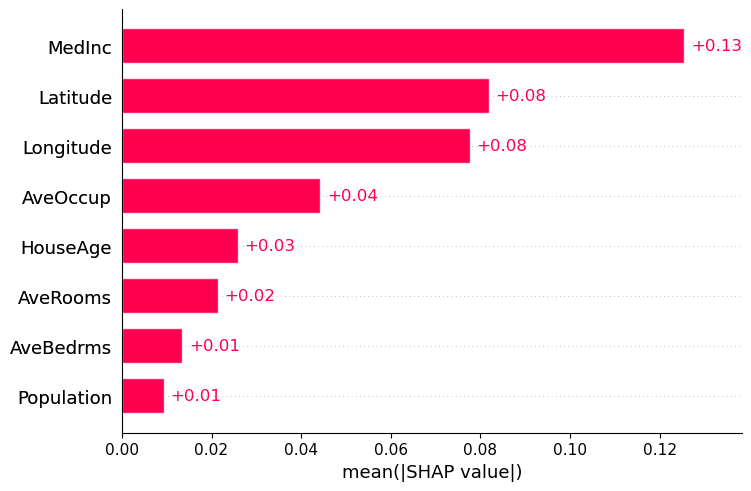}
        \caption{MICE}
        \label{fig:california-mice06}
    \end{subfigure}
    
    \begin{subfigure}{0.45\textwidth}
        \centering
        \includegraphics[width=\textwidth]{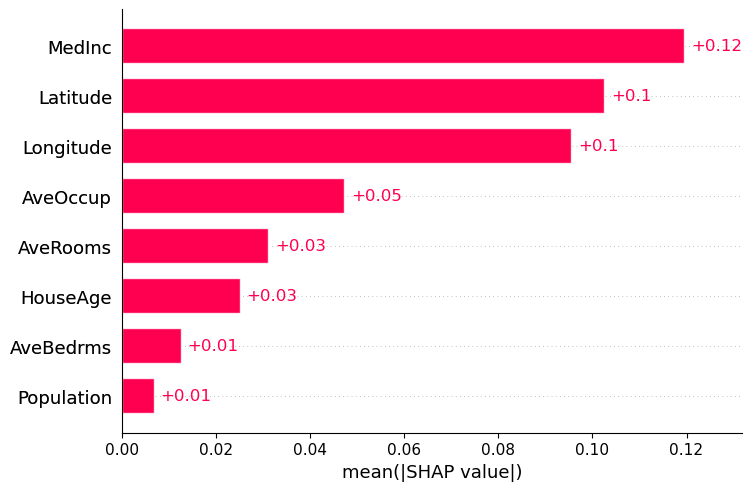}
        \caption{DIMV}
        \label{fig:california-dimv06}
    \end{subfigure}
    \begin{subfigure}{0.45\textwidth}
        \centering
        \includegraphics[width=\textwidth]{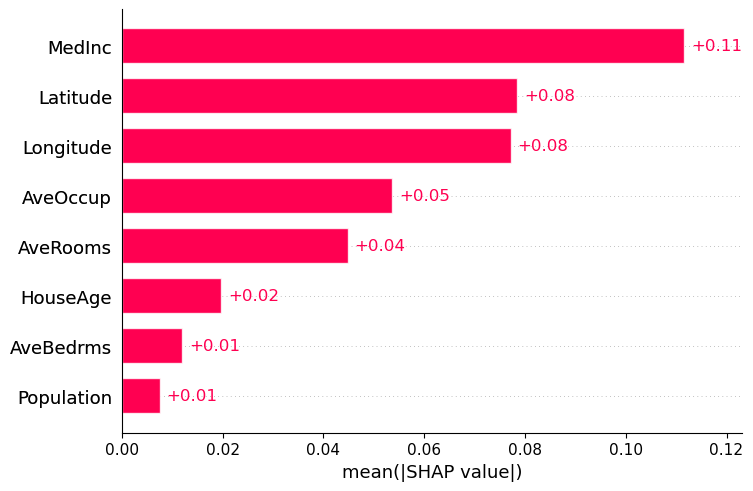}
        \caption{missForest}
        \label{fig:california-mf06}
    \end{subfigure}
    
    \begin{subfigure}{0.45\textwidth}
        \centering
        \includegraphics[width=\textwidth]{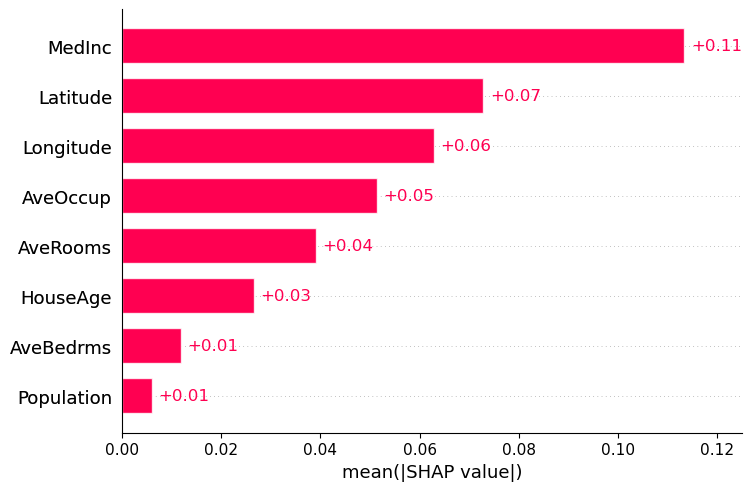}
        \caption{SOFT-IMPUTE}
        \label{fig:california-soft06}
    \end{subfigure}
    \begin{subfigure}{0.45\textwidth}
        \centering
        \includegraphics[width=\textwidth]{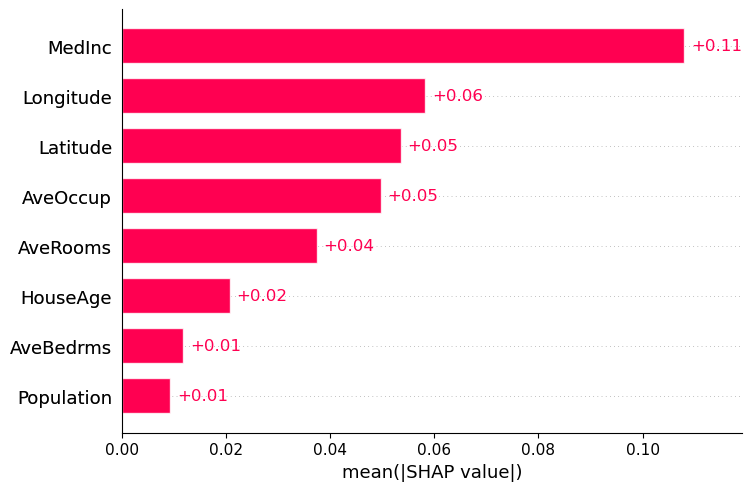}
        \caption{GAIN}
        \label{fig:california-gain06}
    \end{subfigure}
    
    \caption{Global feature importance plot on the California dataset with the missing rate $r=0.6$}
    \label{fig:california-r6-bar}
\end{figure}

\begin{figure}[]
    \centering
    \begin{subfigure}{0.45\textwidth}
        \centering
        \includegraphics[width=\textwidth]{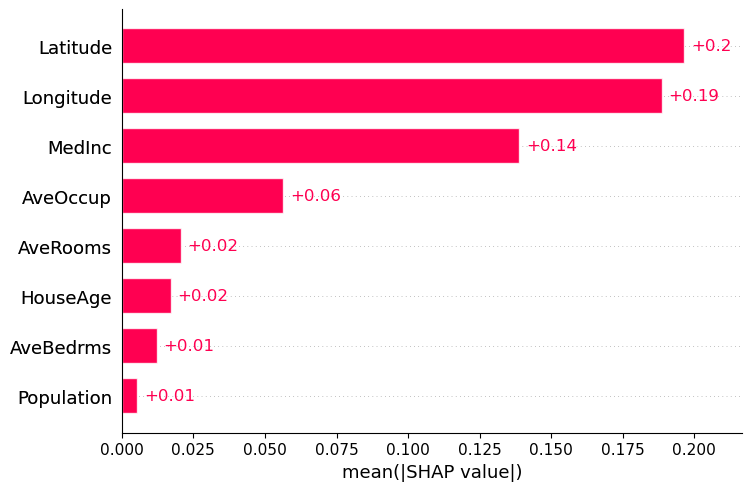}
        \caption{Original (LRO)}
        \label{fig:california-ori08}
    \end{subfigure}
    \begin{subfigure}{0.45\textwidth}
        \centering
        \includegraphics[width=\textwidth]{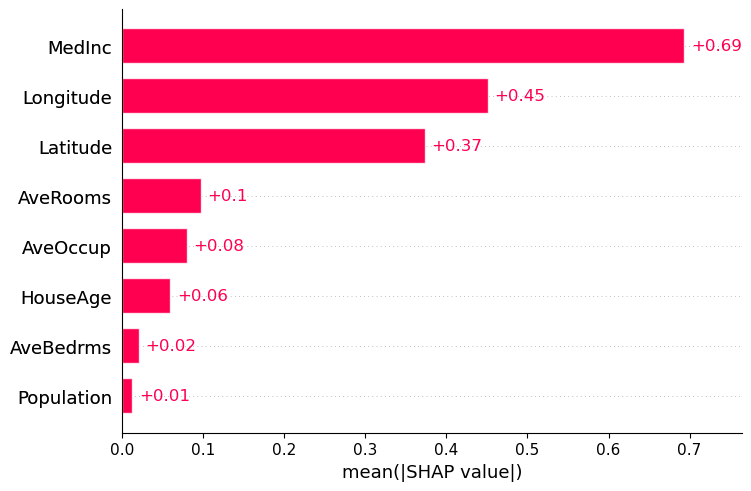}
        \caption{XGBoost without imputation}
        \label{fig:california-xm08}
    \end{subfigure}
    
    \begin{subfigure}{0.45\textwidth}
        \centering
        \includegraphics[width=\textwidth]{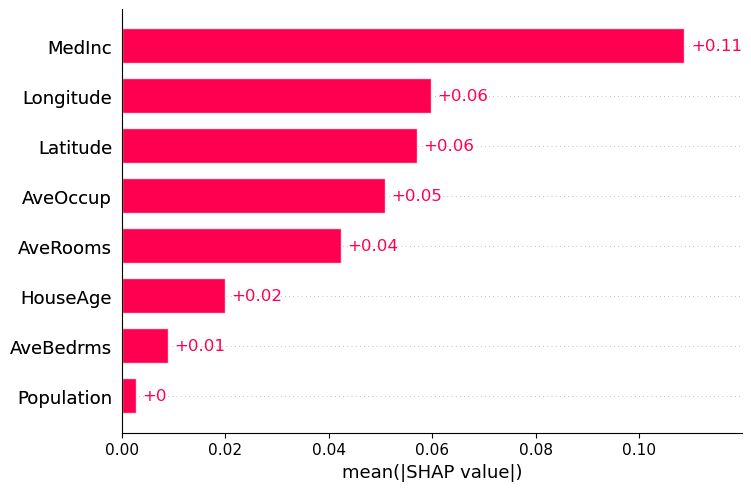}
        \caption{Mean Imputation}
        \label{fig:california-mi08}
    \end{subfigure}
    \begin{subfigure}{0.45\textwidth}
        \centering
        \includegraphics[width=\textwidth]{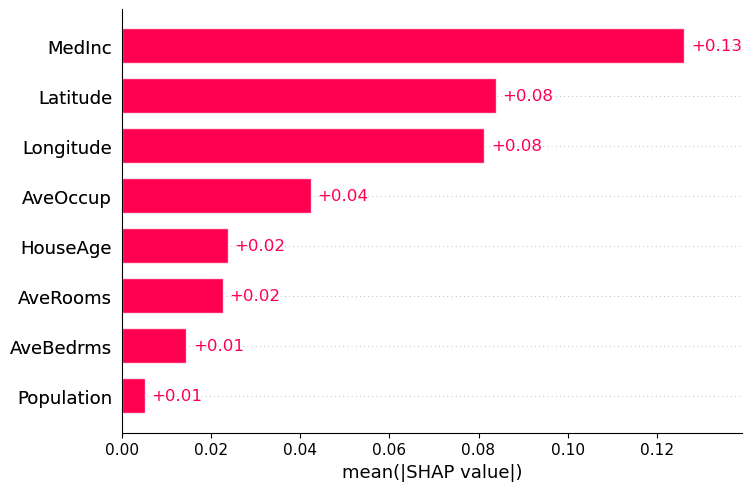}
        \caption{MICE}
        \label{fig:california-mice08}
    \end{subfigure}
    
    \begin{subfigure}{0.45\textwidth}
        \centering
        \includegraphics[width=\textwidth]{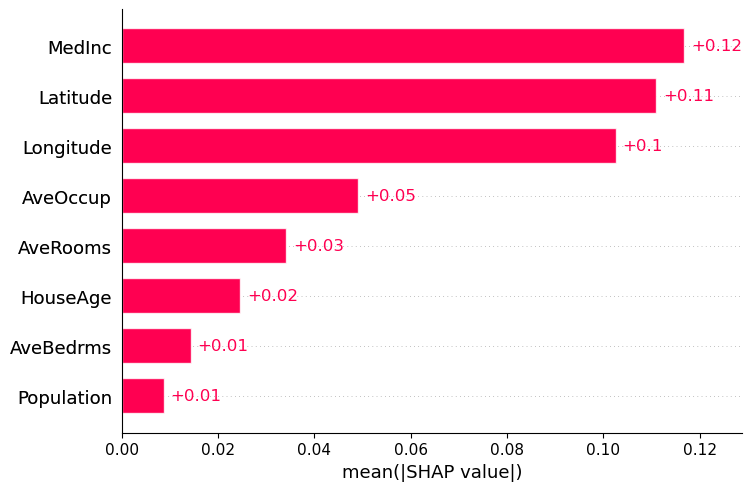}
        \caption{DIMV}
        \label{fig:california-dimv08}
    \end{subfigure}
    \begin{subfigure}{0.45\textwidth}
        \centering
        \includegraphics[width=\textwidth]{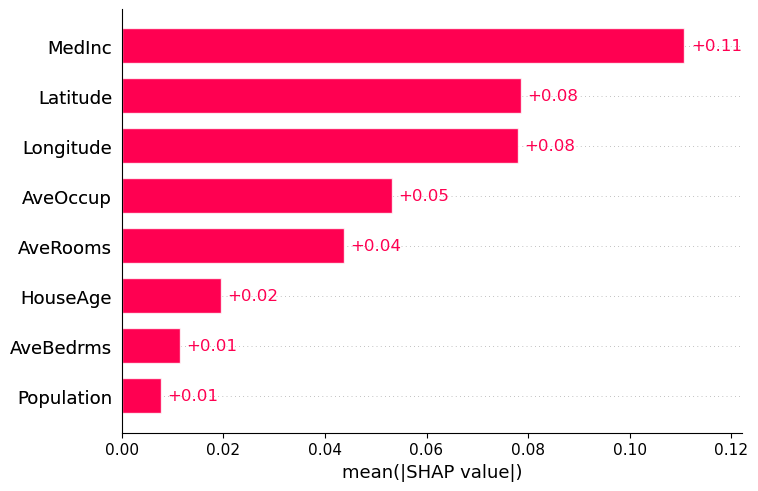}
        \caption{missForest}
        \label{fig:california-mf08}
    \end{subfigure}
    
    \begin{subfigure}{0.45\textwidth}
        \centering
        \includegraphics[width=\textwidth]{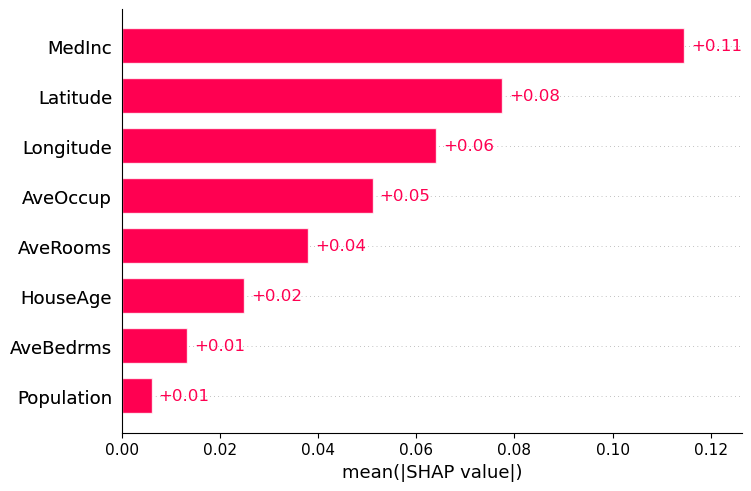}
        \caption{SOFT-IMPUTE}
        \label{fig:california-soft08}
    \end{subfigure}
    \begin{subfigure}{0.45\textwidth}
        \centering
        \includegraphics[width=\textwidth]{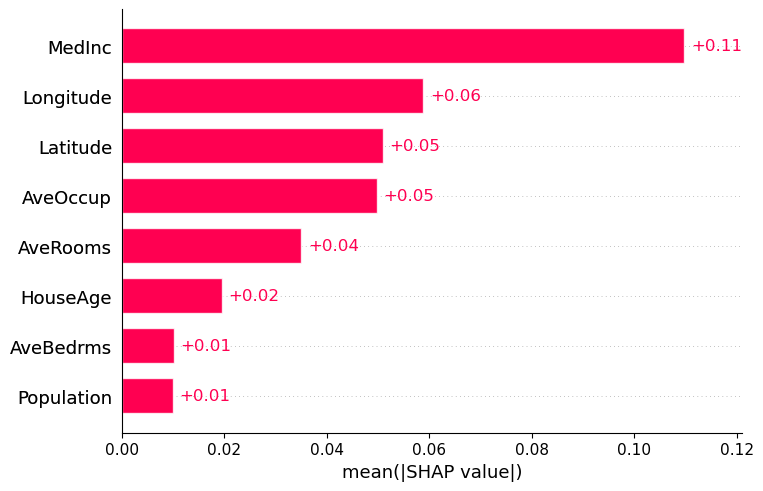}
        \caption{GAIN}
        \label{fig:california-gain08}
    \end{subfigure}
    
    \caption{Global feature importance plot on the California dataset with the missing rate $r=0.8$}
    \label{fig:california-r8-bar}
\end{figure}

\subsection{Beeswarm plot analysis}
While the global plot highlights overall feature importance, it doesn’t reveal how individual feature values impact predictions. The beeswarm plot fills this gap by showing the distribution and direction of Shapley values for each feature across all instances, giving insight into how specific feature values drive predictions up or down. Therefore, in this section, we deeply analyze the effects of feature value, missing rate, and handling missing strategy on the Shapley values by using the beeswarm plot. Note that the gray points in the beeswarm plots represent missing values.

\subsubsection{Beeswarm plot for the California dataset}
The beeswarm plot for the California datasets at missing rates $0.2, 0.4, 0.6,$ and $0.8$ are presented in Figures \ref{fig:beeswarm_plots02}, \ref{fig:beeswarm_plots04}, \ref{fig:beeswarm_plots06}, and \ref{fig:beeswarm_plots08}, respectively. In these figures, we focus on three key features (Latitude, Longitude, and MedInc) because the others show a relatively small impact.

Firstly, we examine the effect of feature values. Notably, feature values consistently influence the Shapley values across all comparison methods and missing rates. For example in Figure~\ref{fig:beeswarm_plots02}, at a missing rate r = 0.2, for the MedInc feature, its low value has a negative impact whereas its high value has a positive impact on model output regardless of all comparison methods.

Secondly, considering the results at each missing rate, we observed similar distributions in the Shapley values for the three key features on both LRO and six imputation methods. However, XGBoost shows a different distribution that is more skewed from zero. For instance, at a missing rate of 0.2 in Figure~\ref{fig:beeswarm_plots02}, the Shapley values in LRO for Latitude and Longitude range from $(-0.75)$ to $0.75$. Likewise, that range for six imputation methods is from around $(-0.5)$ to $0.5$. Meanwhile, XGBoost displays a different distribution from $(-1.2)$ to $0.6$ for Latitude and Longitude.

Finally, we focus on XGBoost without imputation. XGBoost seems to be stable across all missing rates. Especially, missing values always strongly affect the model output. For instance, at a missing rate of 0.4 in Figure~\ref{fig:beeswarm_plots04}, the missing values in the MedInc feature exhibit a strong positive impact, with Shapley values reaching as high as approximately $1.6$ or the missing values in the Longitude and Latitude features show a strong negative effect, with Shapley values dropping to around $(-1.2)$. In other words, for XGBoost without imputation, missing values tend to skew the distribution of Shapley values, either strongly negative or strongly positive.

\begin{figure}[h!]
    \centering
    \begin{subfigure}{0.49\textwidth}
        \centering
        \includegraphics[width=\textwidth]{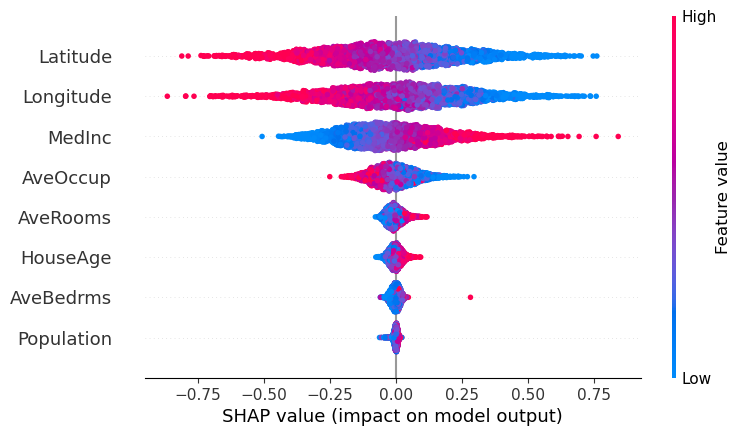}
        \caption{Original (LRO)}
    \end{subfigure}
    \begin{subfigure}{0.49\textwidth}
        \centering
        \includegraphics[width=\textwidth]{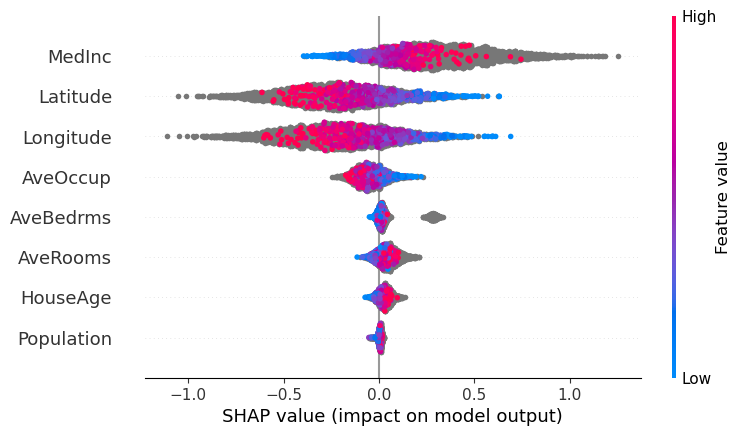}
        \caption{XGBoost without imputation}
    \end{subfigure}
    
    \begin{subfigure}{0.49\textwidth}
        \centering
        \includegraphics[width=\textwidth]{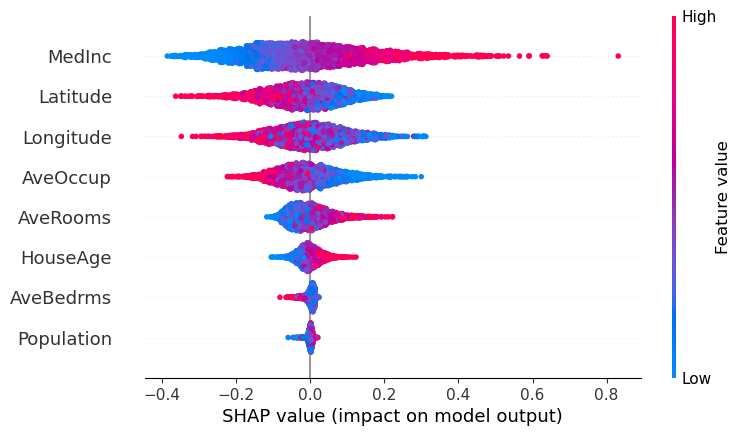}
        \caption{Mean imputation}
    \end{subfigure}    
    \begin{subfigure}{0.49\textwidth}
        \centering
        \includegraphics[width=\textwidth]{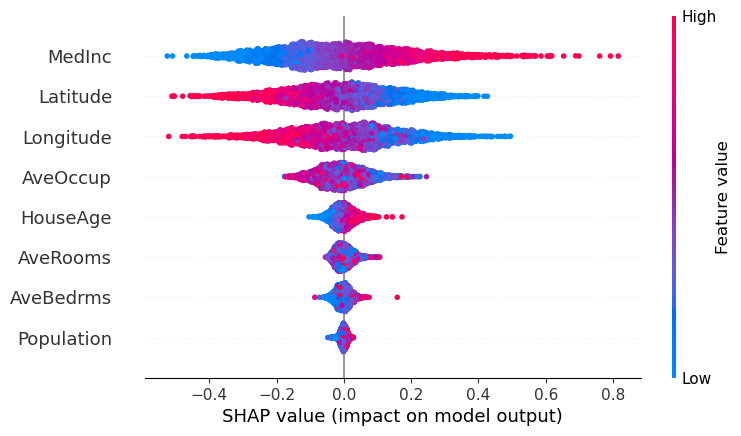}
        \caption{MICE}
    \end{subfigure}
    
    \begin{subfigure}{0.49\textwidth}
        \centering
        \includegraphics[width=\textwidth]{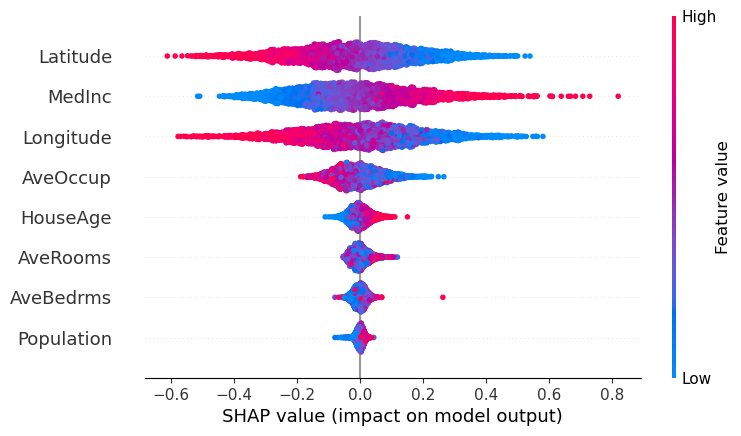}
        \caption{DIMV}
    \end{subfigure}    
    \begin{subfigure}{0.49\textwidth}
        \centering
        \includegraphics[width=\textwidth]{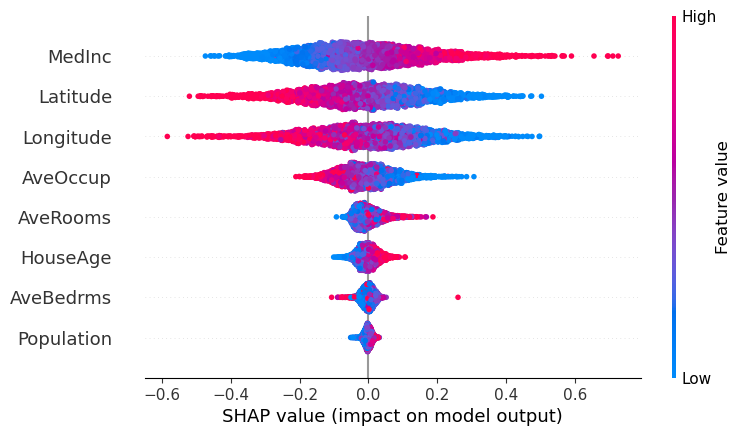}
        \caption{MissForest}
    \end{subfigure}
    
    \begin{subfigure}{0.49\textwidth}
        \centering
        \includegraphics[width=\textwidth]{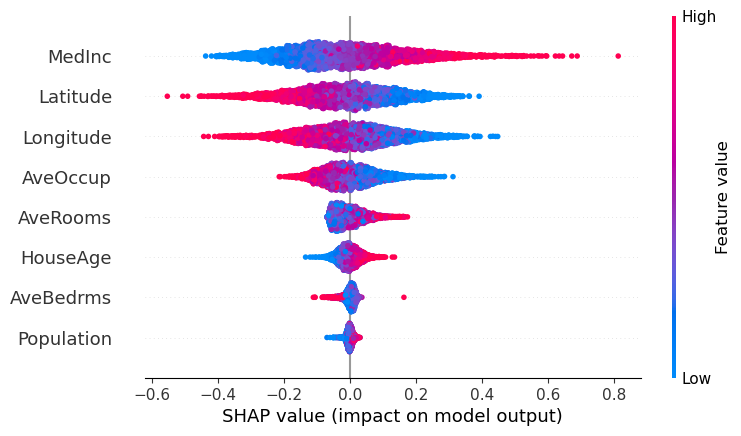}
        \caption{SOFT-IMPUTE}
    \end{subfigure}
    \begin{subfigure}{0.49\textwidth}
        \centering
        \includegraphics[width=\textwidth]{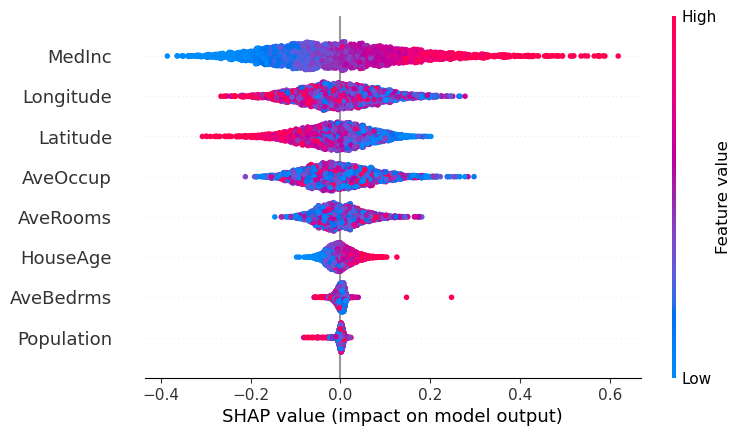}
        \caption{GAIN}
    \end{subfigure}
    
    \caption{Beeswarm plots for the California dataset at missing rate $r=0.2$ }
    \label{fig:beeswarm_plots02}
\end{figure}

\begin{figure}[h!]
    \centering
    \begin{subfigure}{0.49\textwidth}
        \centering
        \includegraphics[width=\textwidth]{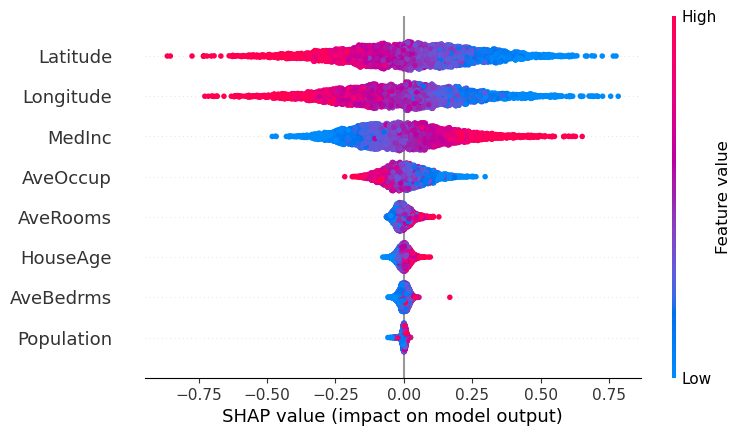}
        \caption{Original (LRO)}
    \end{subfigure}    
    \begin{subfigure}{0.49\textwidth}
        \centering
        \includegraphics[width=\textwidth]{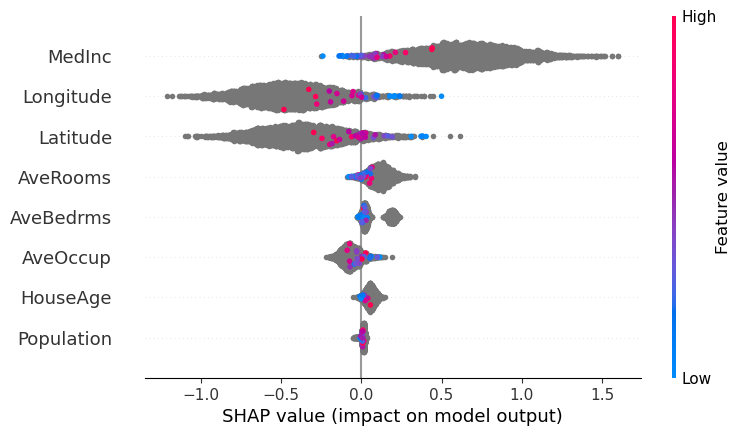}
        \caption{XGBoost without imputation}
    \end{subfigure}
    
    \begin{subfigure}{0.49\textwidth}
        \centering
        \includegraphics[width=\textwidth]{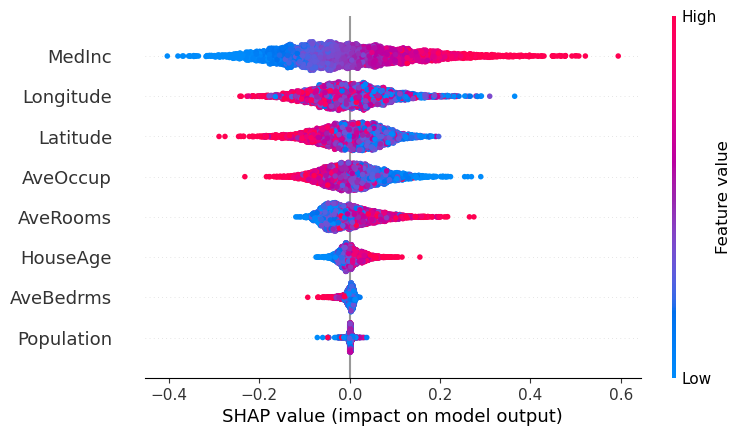}
        \caption{Mean imputation}
    \end{subfigure}    
    \begin{subfigure}{0.49\textwidth}
        \centering
        \includegraphics[width=\textwidth]{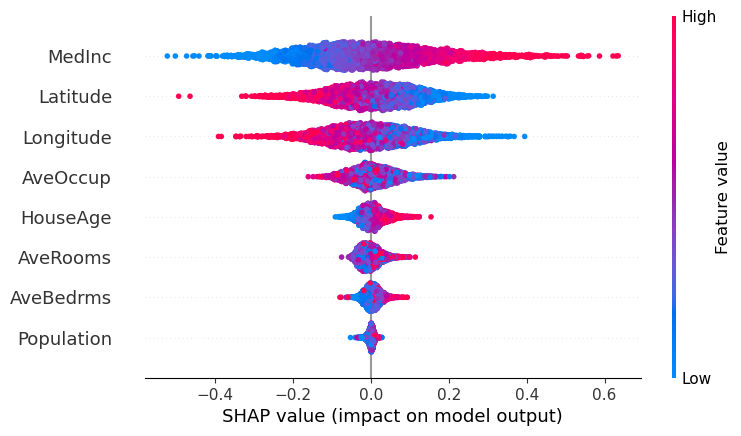}
        \caption{MICE}
    \end{subfigure}
    
    \begin{subfigure}{0.49\textwidth}
        \centering
        \includegraphics[width=\textwidth]{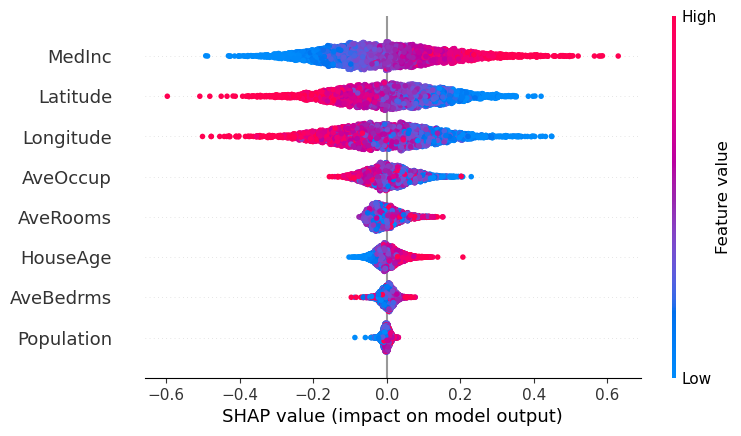}
        \caption{DIMV}
    \end{subfigure}    
    \begin{subfigure}{0.49\textwidth}
        \centering
        \includegraphics[width=\textwidth]{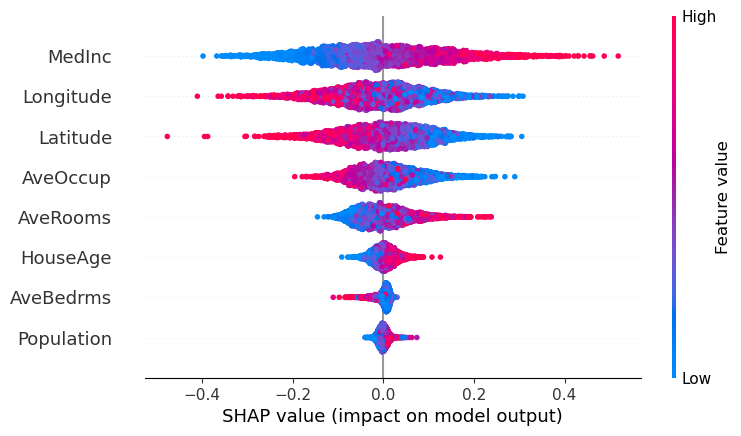}
        \caption{MissForest}
    \end{subfigure}
    
    \begin{subfigure}{0.49\textwidth}
        \centering
        \includegraphics[width=\textwidth]{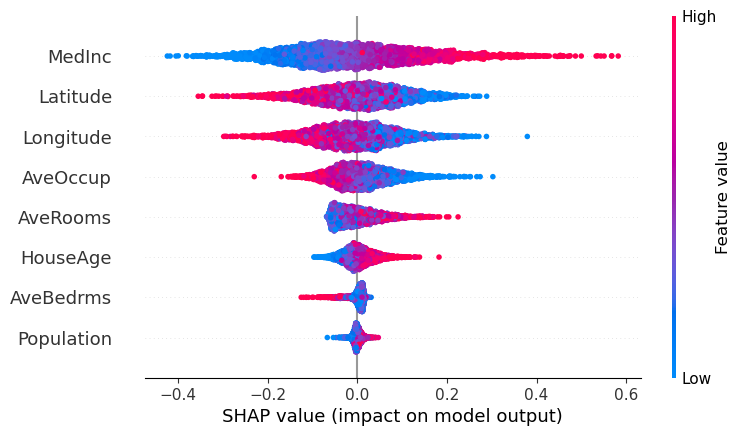}
        \caption{SOFT-IMPUTE}
    \end{subfigure}
    \begin{subfigure}{0.49\textwidth}
        \centering
        \includegraphics[width=\textwidth]{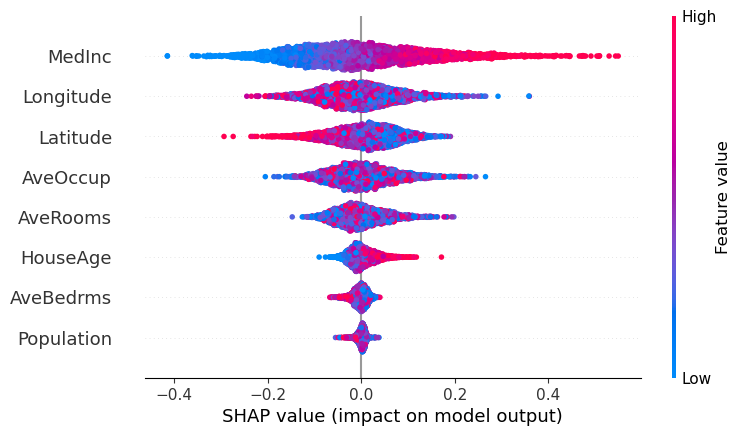}
        \caption{GAIN}
    \end{subfigure}
    
    \caption{Beeswarm plots for the California dataset at missing rate $r=0.4$ }
    \label{fig:beeswarm_plots04}
\end{figure}

\begin{figure}[h!]
    \centering
    \begin{subfigure}{0.49\textwidth}
        \centering
        \includegraphics[width=\textwidth]{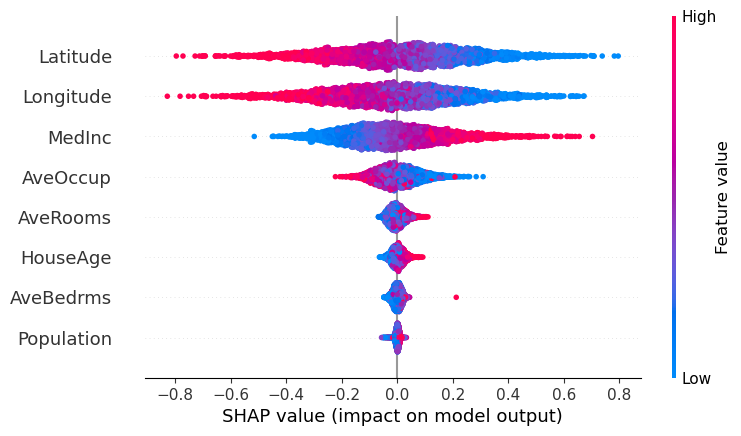}
        \caption{Original (LRO)}
    \end{subfigure}
    \begin{subfigure}{0.49\textwidth}
        \centering
        \includegraphics[width=\textwidth]{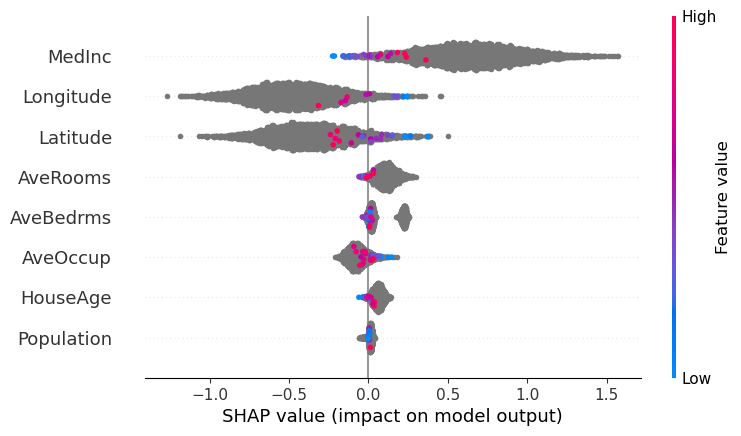}
        \caption{XGBoost without imputation}
    \end{subfigure}
    
    \begin{subfigure}{0.49\textwidth}
        \centering
        \includegraphics[width=\textwidth]{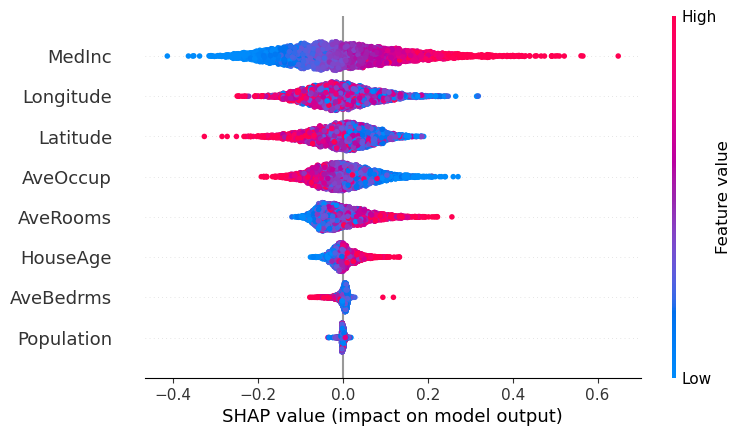}
        \caption{Mean imputation}
    \end{subfigure}    
    \begin{subfigure}{0.49\textwidth}
        \centering
        \includegraphics[width=\textwidth]{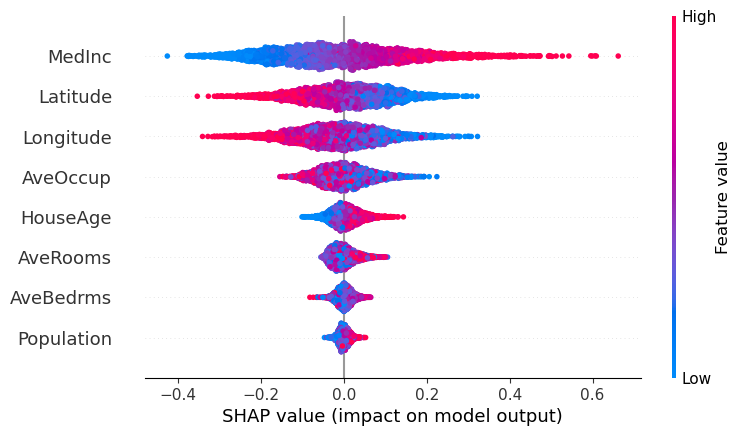}
        \caption{MICE}
    \end{subfigure}
    
    \begin{subfigure}{0.49\textwidth}
        \centering
        \includegraphics[width=\textwidth]{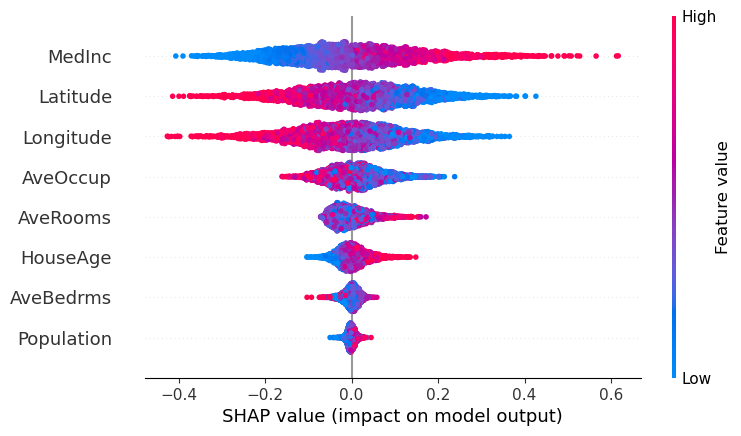}
        \caption{DIMV}
    \end{subfigure}    
    \begin{subfigure}{0.49\textwidth}
        \centering
        \includegraphics[width=\textwidth]{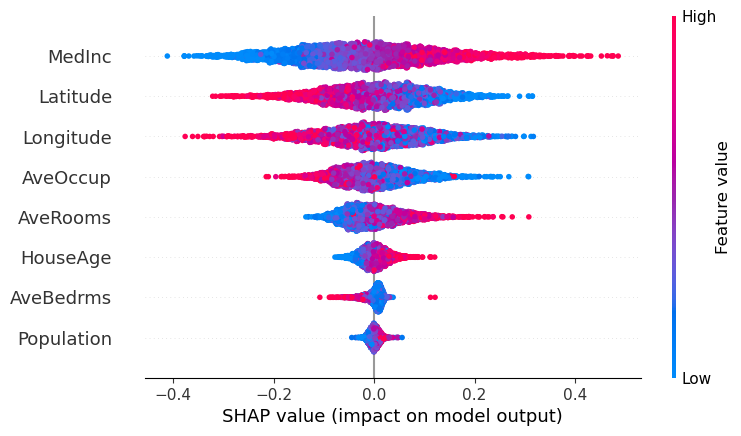}
        \caption{MissForest}
    \end{subfigure}
    
    \begin{subfigure}{0.49\textwidth}
        \centering
        \includegraphics[width=\textwidth]{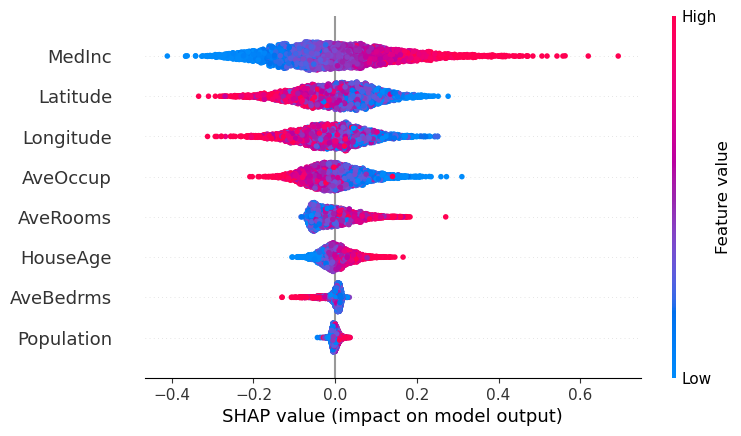}
        \caption{SOFT-IMPUTE}
    \end{subfigure}
    \begin{subfigure}{0.49\textwidth}
        \centering
        \includegraphics[width=\textwidth]{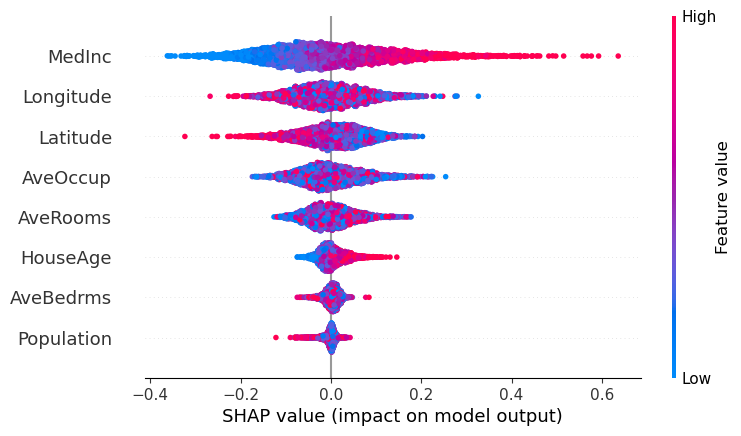}
        \caption{GAIN}
    \end{subfigure}
    
    \caption{Beeswarm plots for the California dataset at missing rate $r=0.6$ }
    \label{fig:beeswarm_plots06}
\end{figure}

\begin{figure}[h!]
    \centering
    \begin{subfigure}{0.49\textwidth}
        \centering
        \includegraphics[width=\textwidth]{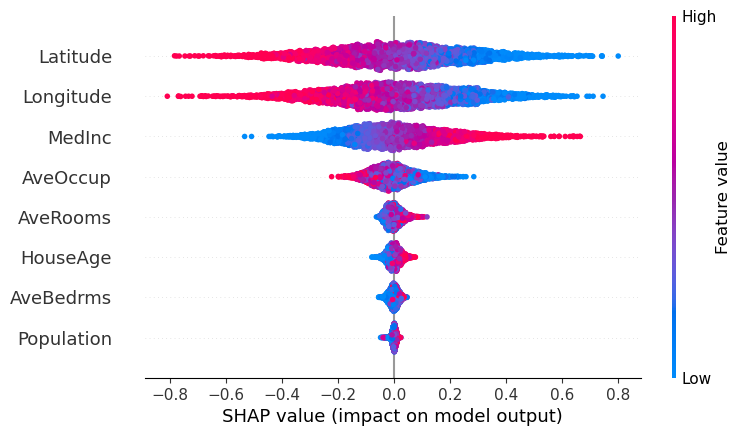}
        \caption{Original (LRO)}
    \end{subfigure}    
    \begin{subfigure}{0.49\textwidth}
        \centering
        \includegraphics[width=\textwidth]{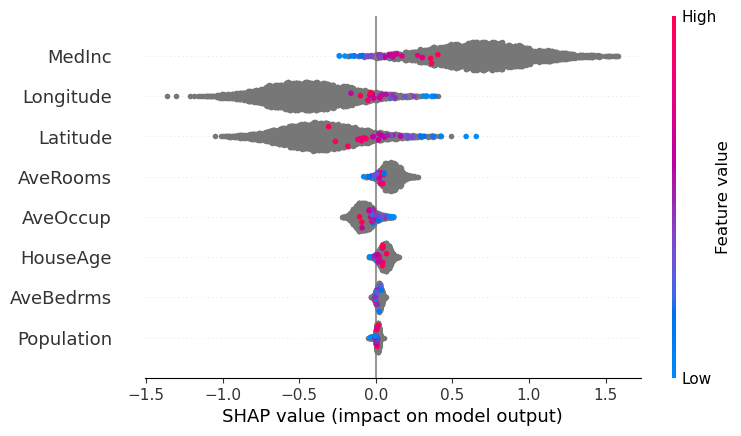}
        \caption{XGBoost without imputation}
    \end{subfigure}
    
    \begin{subfigure}{0.49\textwidth}
        \centering
        \includegraphics[width=\textwidth]{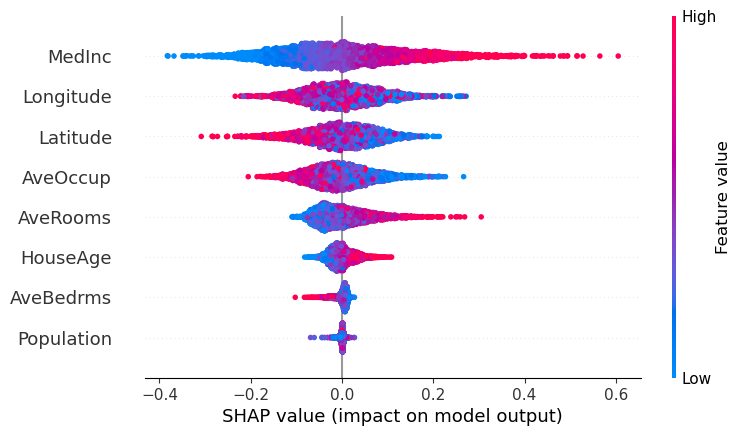}
        \caption{Mean imputation}
    \end{subfigure}    
    \begin{subfigure}{0.49\textwidth}
        \centering
        \includegraphics[width=\textwidth]{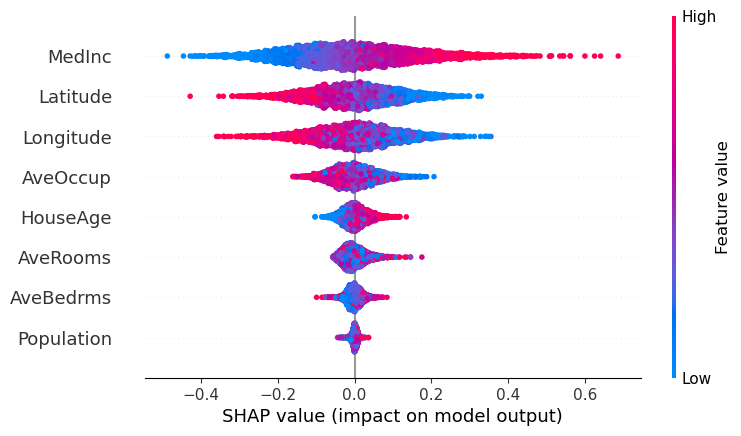}
        \caption{MICE}
    \end{subfigure}
    
    \begin{subfigure}{0.49\textwidth}
        \centering
        \includegraphics[width=\textwidth]{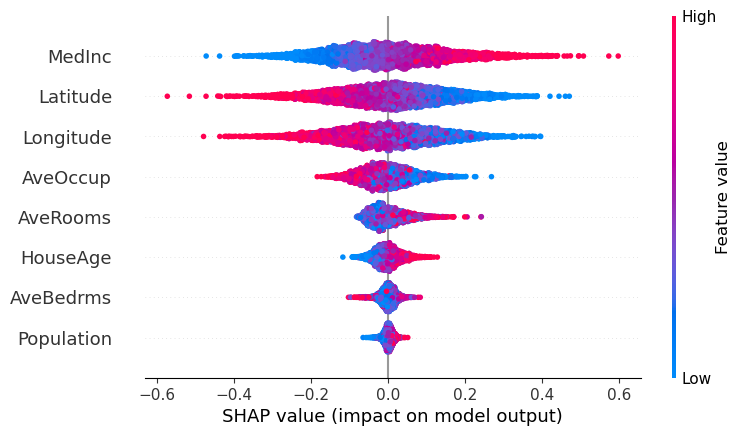}
        \caption{DIMV}
    \end{subfigure}    
    \begin{subfigure}{0.49\textwidth}
        \centering
        \includegraphics[width=\textwidth]{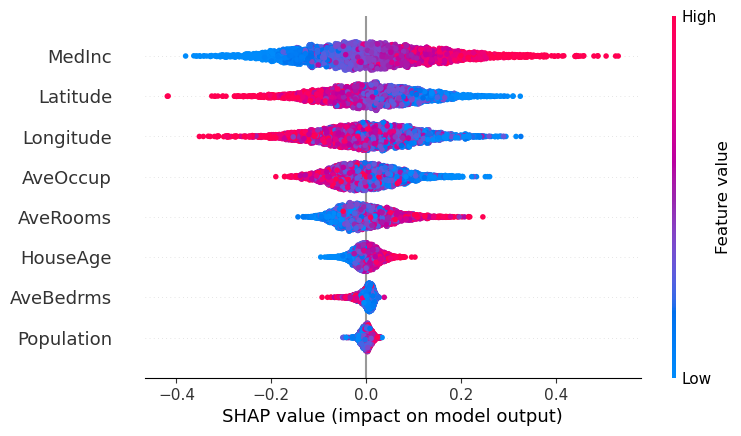}
        \caption{MissForest}
    \end{subfigure}
    
    \begin{subfigure}{0.49\textwidth}
        \centering
        \includegraphics[width=\textwidth]{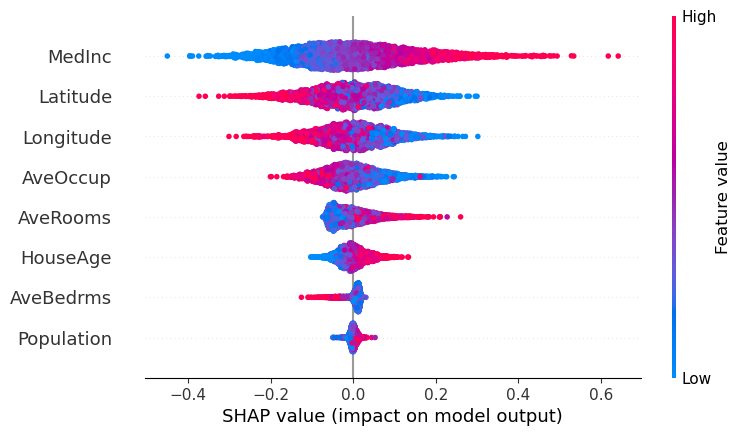}
        \caption{SOFT-IMPUTE}
    \end{subfigure}
    \begin{subfigure}{0.49\textwidth}
        \centering
        \includegraphics[width=\textwidth]{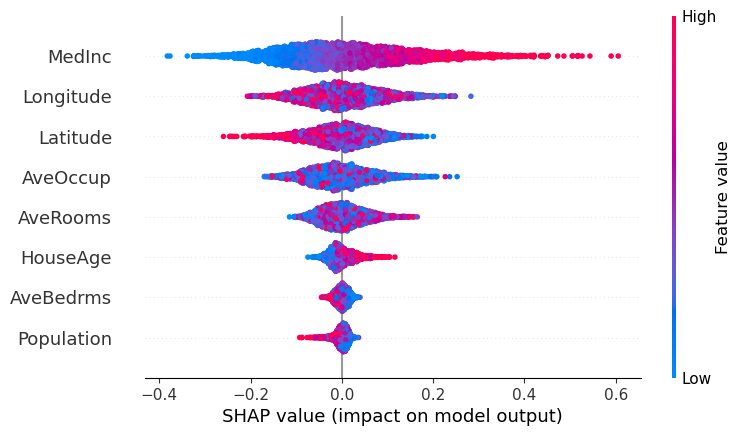}
        \caption{GAIN}
    \end{subfigure}
    
    \caption{Beeswarm plots for the California dataset at missing rate $r=0.8$ }
    \label{fig:beeswarm_plots08}
\end{figure}

\subsubsection{Beeswarm plot for the Diabetes dataset}
The beeswarm plot for the Diabetes dataset at missing rates $0.2, 0.4, 0.6,$ and $0.8$ are presented in Figures \ref{fig:diabetes-beeswarm_plots02}, \ref{fig:diabetes-beeswarm_plots04}, \ref{fig:diabetes-beeswarm_plots06}, and \ref{fig:diabetes-beeswarm_plots08}, respectively.
A similar pattern in analyzing the impact of feature values on Shapley values can be seen, where feature values consistently influence Shapley values across all comparison methods and missing rates.

Regarding the distribution of Shapley values, across all missing rates, LRO and imputation methods exhibit a rather balanced distribution around the zero value. In contrast, the XGBoost model shows an increasingly skewed distribution as the missing rate rises. For example, focusing on the ``bmi" feature in Figure~\ref{fig:diabetes-beeswarm_plots02}, most of the Shapley values for LRO and imputation methods range from $(-0.25)$ to $0.25$, while for XGBoost, they range from $(-0.15)$ to $0.45$.

Considering XGBoost without imputation, as the missing rates increase, the Shapley values of missing values become increasingly skewed to one side. For instance, in Figure~\ref{fig:diabetes-beeswarm_plots08}, at a missing rate of 0.8, the Shapley values for the features ``bmi", ``bp", ``s5", and ``s6" are consistently positive, while those for ``s2" and ``s1" are predominantly negative.

\begin{figure}[h!]
    \centering
    \begin{subfigure}{0.42\textwidth}
        \centering
        \includegraphics[width=\textwidth]{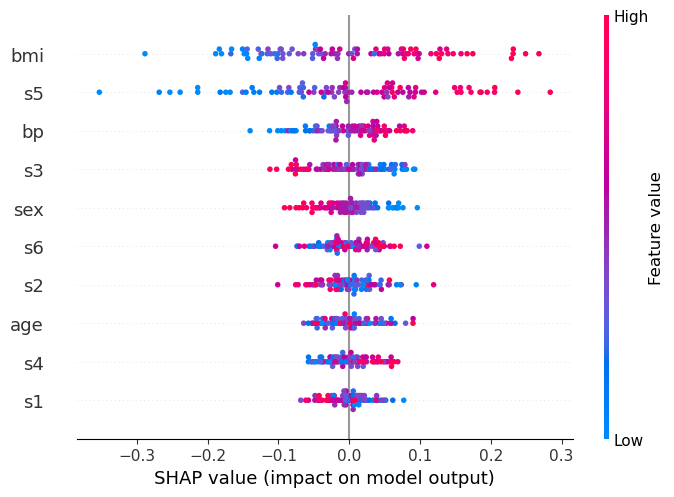}
        \caption{Original (LRO)}
    \end{subfigure}    
    \begin{subfigure}{0.42\textwidth}
        \centering
        \includegraphics[width=\textwidth]{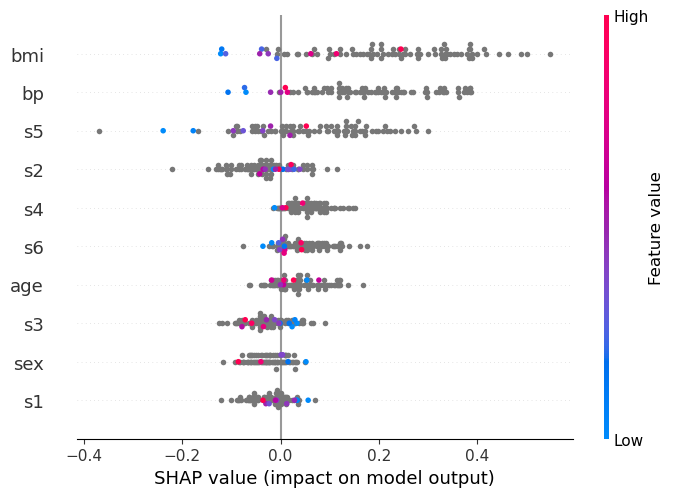}
        \caption{XGBoost without imputation}
    \end{subfigure}
    
    \begin{subfigure}{0.42\textwidth}
        \centering
        \includegraphics[width=\textwidth]{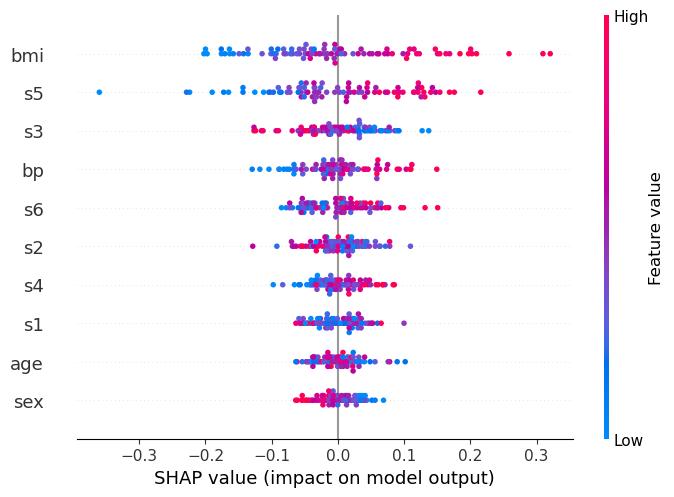}
        \caption{Mean imputation}
    \end{subfigure}    
    \begin{subfigure}{0.42\textwidth}
        \centering
        \includegraphics[width=\textwidth]{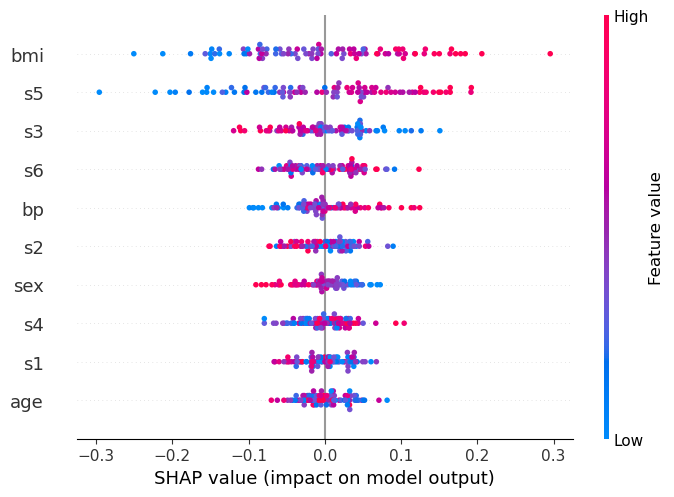}
        \caption{MICE}
    \end{subfigure}
    
    \begin{subfigure}{0.42\textwidth}
        \centering
        \includegraphics[width=\textwidth]{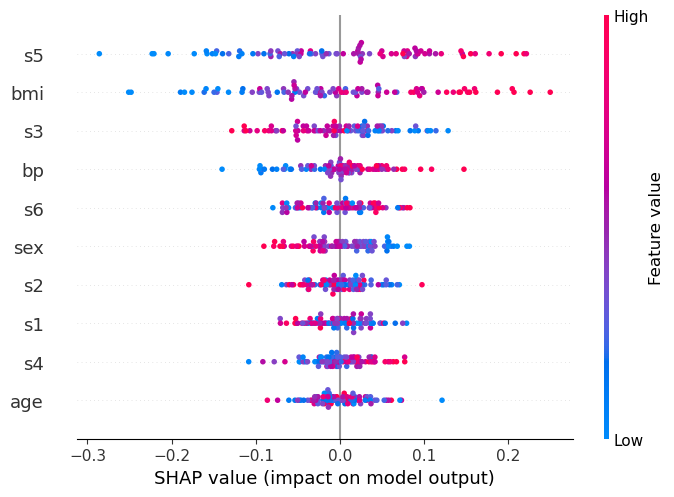}
        \caption{DIMV}
    \end{subfigure}    
    \begin{subfigure}{0.42\textwidth}
        \centering
        \includegraphics[width=\textwidth]{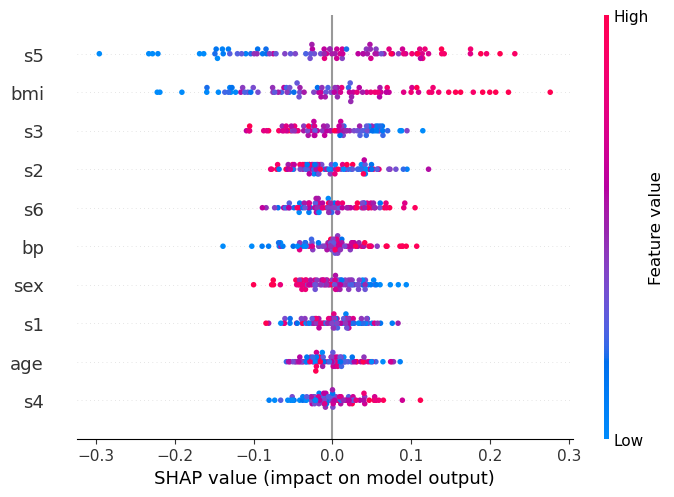}
        \caption{MissForest}
    \end{subfigure}
    
    \begin{subfigure}{0.42\textwidth}
        \centering
        \includegraphics[width=\textwidth]{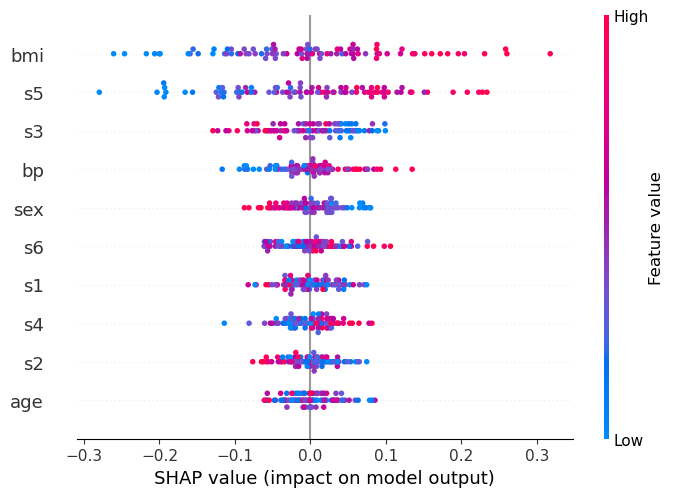}
        \caption{SOFT-IMPUTE}
    \end{subfigure}
    \begin{subfigure}{0.42\textwidth}
        \centering
        \includegraphics[width=\textwidth]{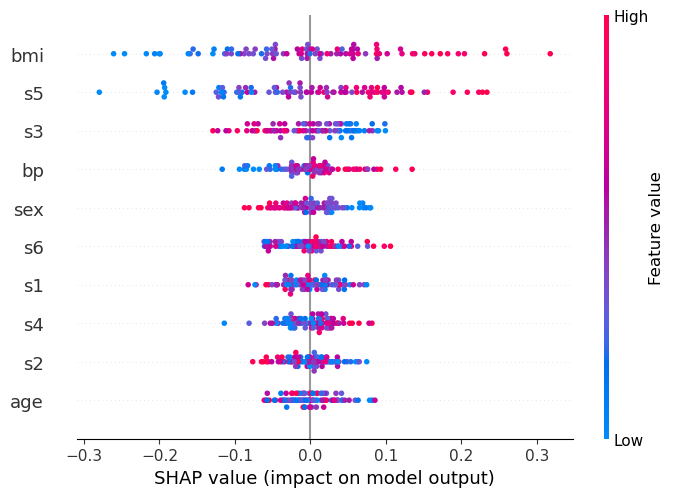}
        \caption{GAIN}
    \end{subfigure}
    
    \caption{Beeswarm plots for the Diabetes dataset at missing rate $r=0.2$ }
    \label{fig:diabetes-beeswarm_plots02}
\end{figure}

\begin{figure}[h!]
    \centering
    \begin{subfigure}{0.42\textwidth}
        \centering
        \includegraphics[width=\textwidth]{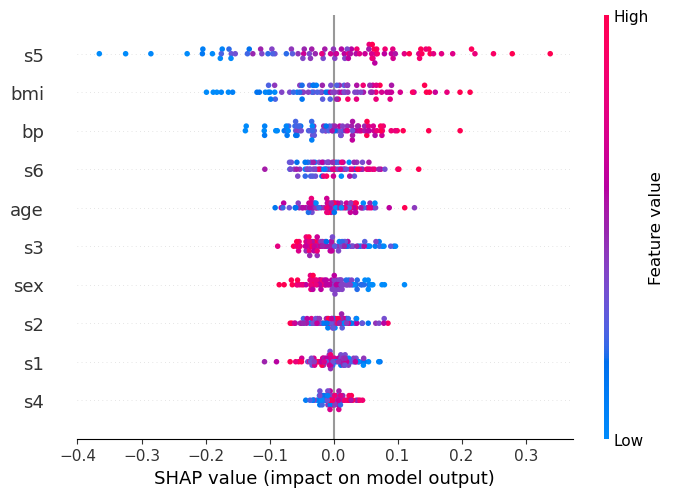}
        \caption{Original (LRO)}
    \end{subfigure}    
    \begin{subfigure}{0.42\textwidth}
        \centering
        \includegraphics[width=\textwidth]{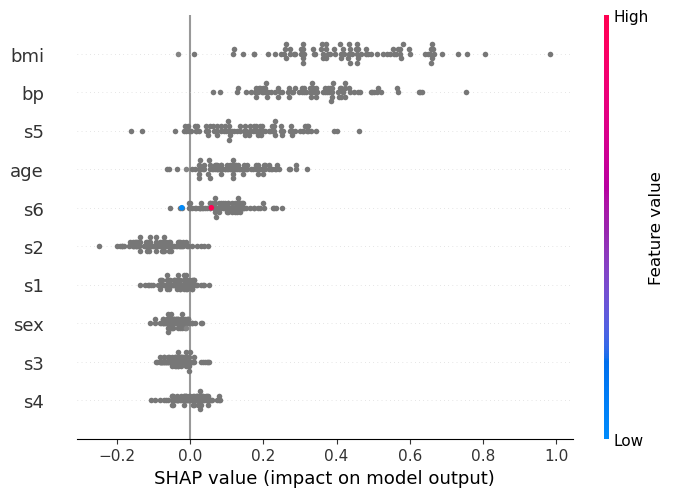}
        \caption{XGBoost without imputation}
    \end{subfigure}
    
    \begin{subfigure}{0.42\textwidth}
        \centering
        \includegraphics[width=\textwidth]{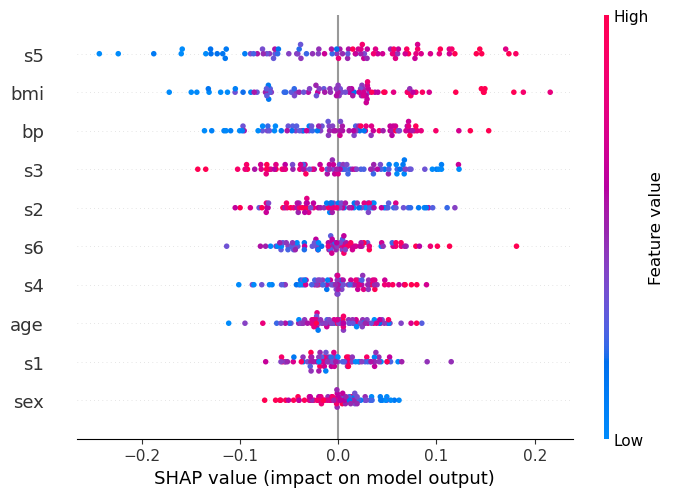}
        \caption{Mean imputation}
    \end{subfigure}    
    \begin{subfigure}{0.42\textwidth}
        \centering
        \includegraphics[width=\textwidth]{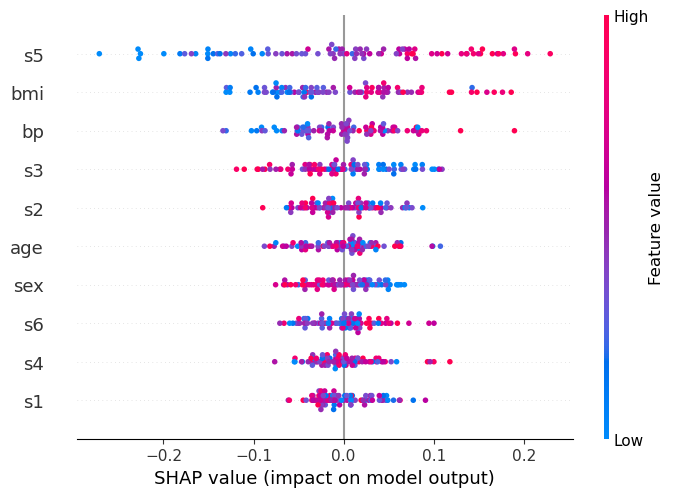}
        \caption{MICE}
    \end{subfigure}
    
    \begin{subfigure}{0.42\textwidth}
        \centering
        \includegraphics[width=\textwidth]{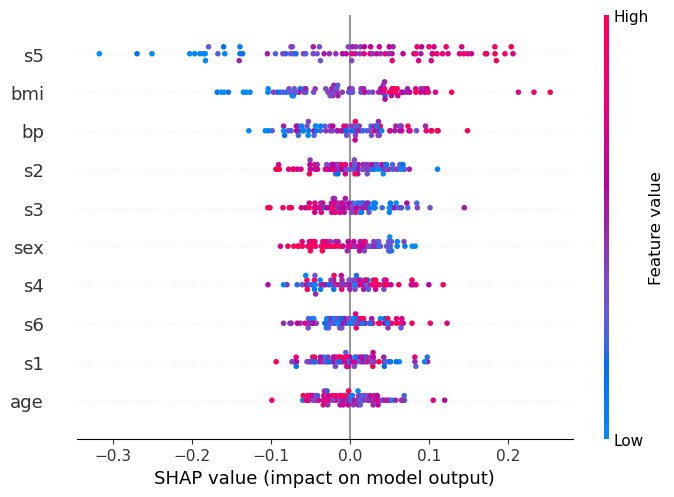}
        \caption{DIMV}
    \end{subfigure}    
    \begin{subfigure}{0.42\textwidth}
        \centering
        \includegraphics[width=\textwidth]{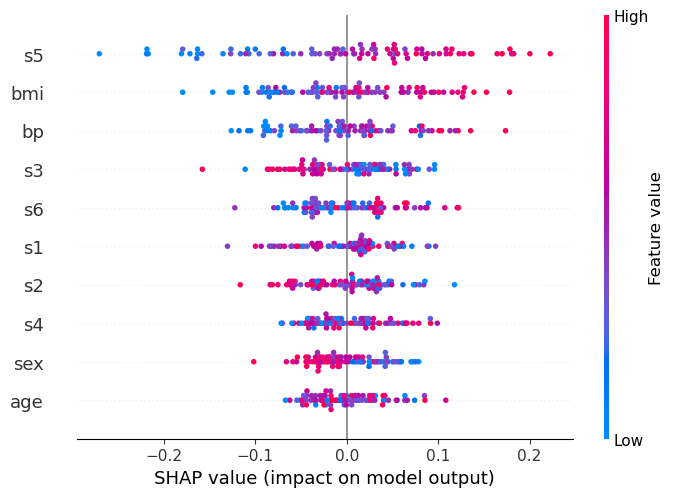}
        \caption{MissForest}
    \end{subfigure}
    
    \begin{subfigure}{0.42\textwidth}
        \centering
        \includegraphics[width=\textwidth]{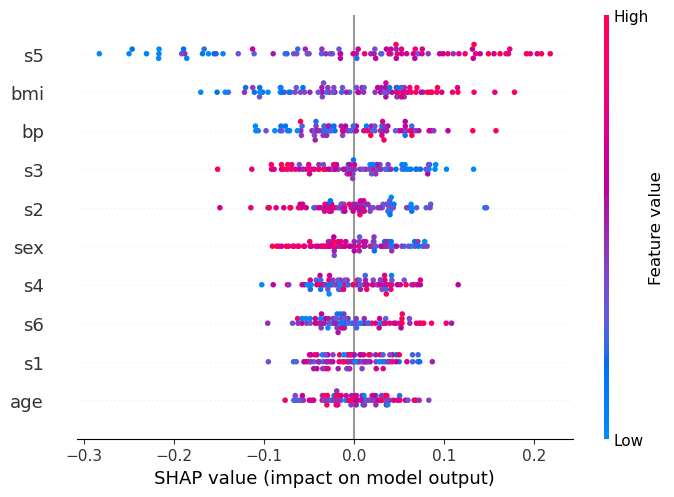}
        \caption{SOFT-IMPUTE}
    \end{subfigure}
    \begin{subfigure}{0.42\textwidth}
        \centering
        \includegraphics[width=\textwidth]{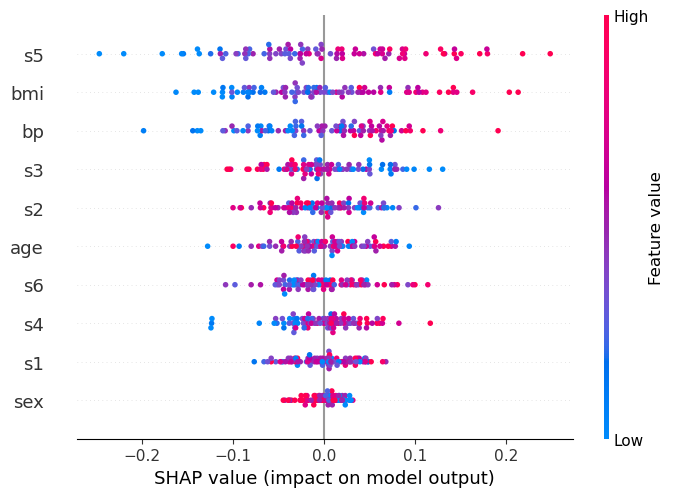}
        \caption{GAIN}
    \end{subfigure}
    
    \caption{Beeswarm plots for the Diabetes dataset at missing rate $r=0.4$ }
    \label{fig:diabetes-beeswarm_plots04}
\end{figure}

\begin{figure}[h!]
    \centering
    \begin{subfigure}{0.42\textwidth}
        \centering
        \includegraphics[width=\textwidth]{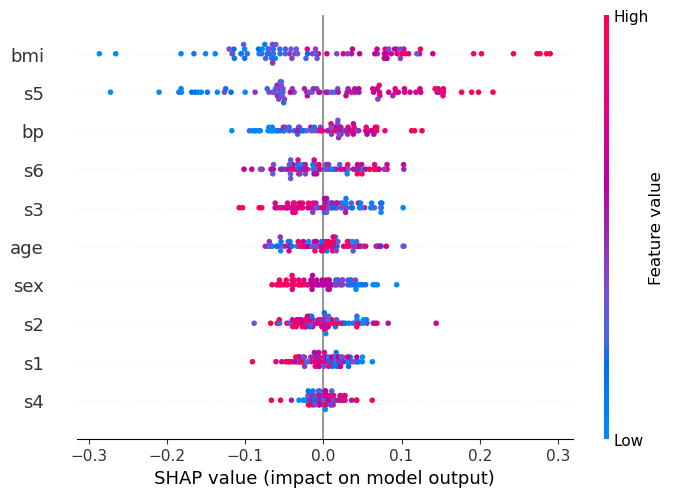}
        \caption{Original (LRO)}
    \end{subfigure}    
    \begin{subfigure}{0.42\textwidth}
        \centering
        \includegraphics[width=\textwidth]{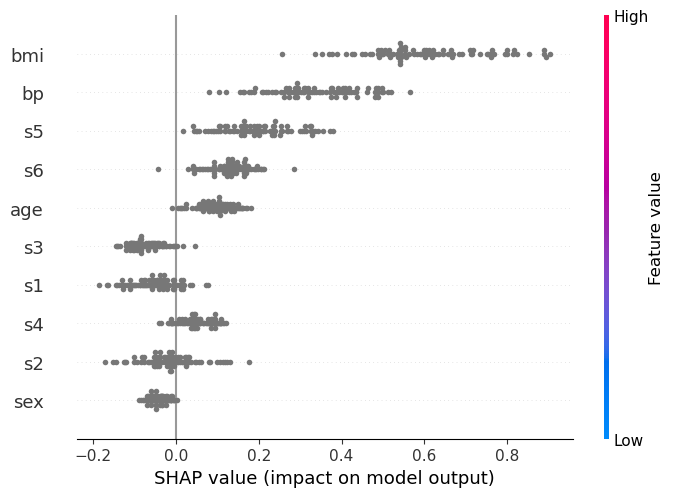}
        \caption{XGBoost without imputation}
    \end{subfigure}
    
    \begin{subfigure}{0.42\textwidth}
        \centering
        \includegraphics[width=\textwidth]{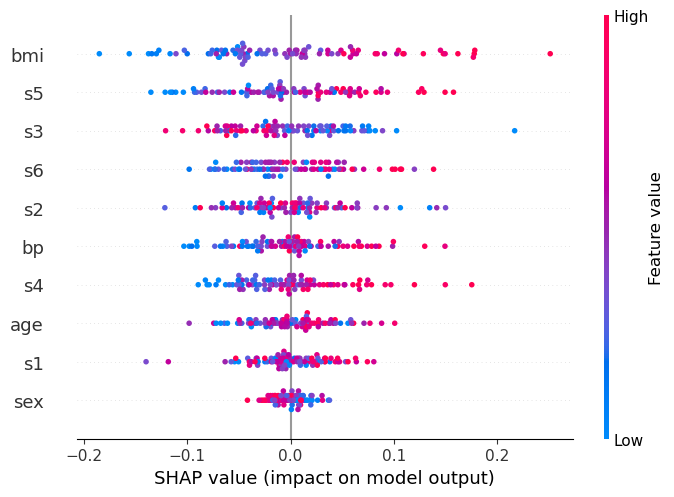}
        \caption{Mean imputation}
    \end{subfigure}    
    \begin{subfigure}{0.42\textwidth}
        \centering
        \includegraphics[width=\textwidth]{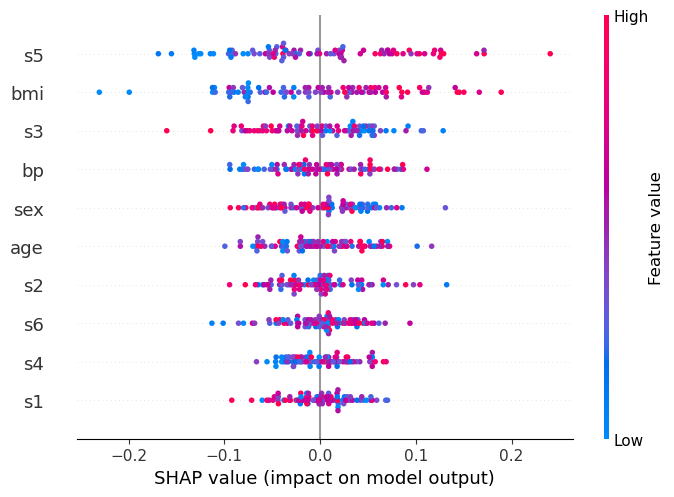}
        \caption{MICE}
    \end{subfigure}
    
    \begin{subfigure}{0.42\textwidth}
        \centering
        \includegraphics[width=\textwidth]{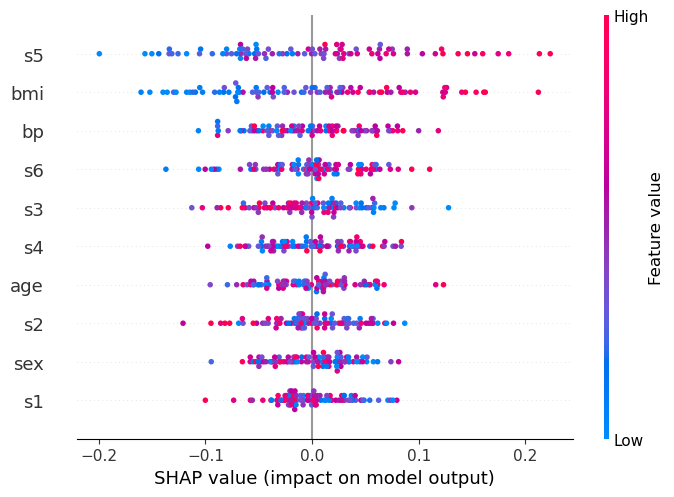}
        \caption{DIMV}
    \end{subfigure}    
    \begin{subfigure}{0.42\textwidth}
        \centering
        \includegraphics[width=\textwidth]{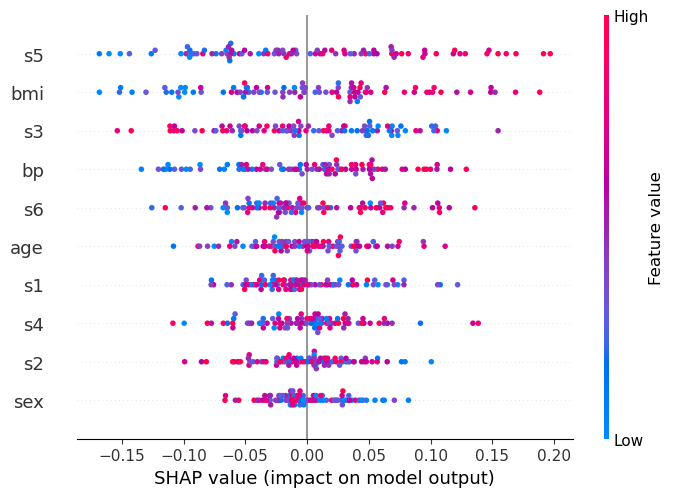}
        \caption{MissForest}
    \end{subfigure}
    
    \begin{subfigure}{0.42\textwidth}
        \centering
        \includegraphics[width=\textwidth]{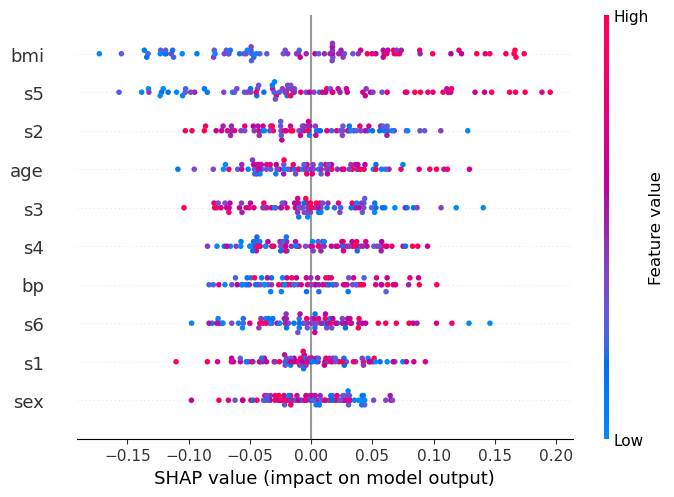}
        \caption{SOFT-IMPUTE}
    \end{subfigure}
    \begin{subfigure}{0.42\textwidth}
        \centering
        \includegraphics[width=\textwidth]{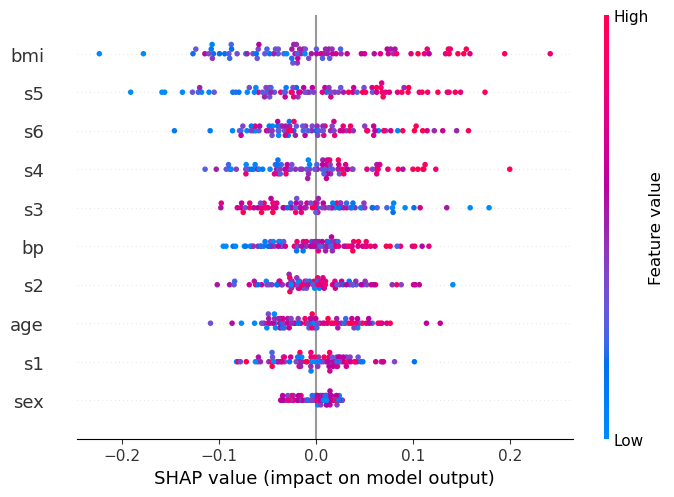}
        \caption{GAIN}
    \end{subfigure}
    
    \caption{Beeswarm plots for the Diabetes dataset at missing rate $r=0.6$ }
    \label{fig:diabetes-beeswarm_plots06}
\end{figure}

\begin{figure}[h!]
    \centering
    \begin{subfigure}{0.42\textwidth}
        \centering
        \includegraphics[width=\textwidth]{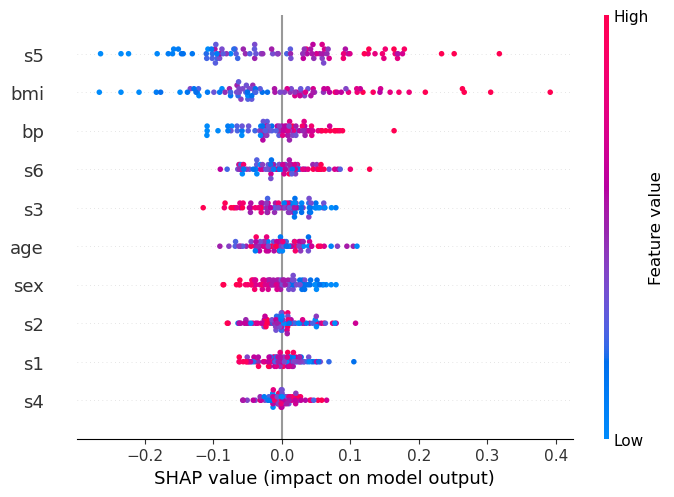}
        \caption{Original (LRO)}
    \end{subfigure}    
    \begin{subfigure}{0.42\textwidth}
        \centering
        \includegraphics[width=\textwidth]{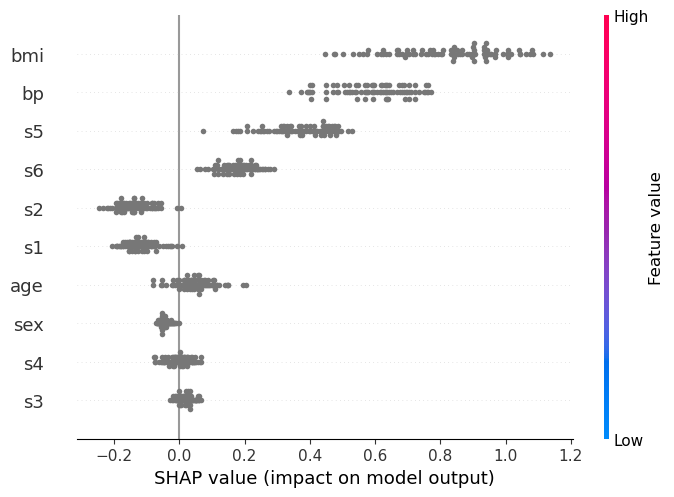}
        \caption{XGBoost without imputation}
    \end{subfigure}
    
    \begin{subfigure}{0.42\textwidth}
        \centering
        \includegraphics[width=\textwidth]{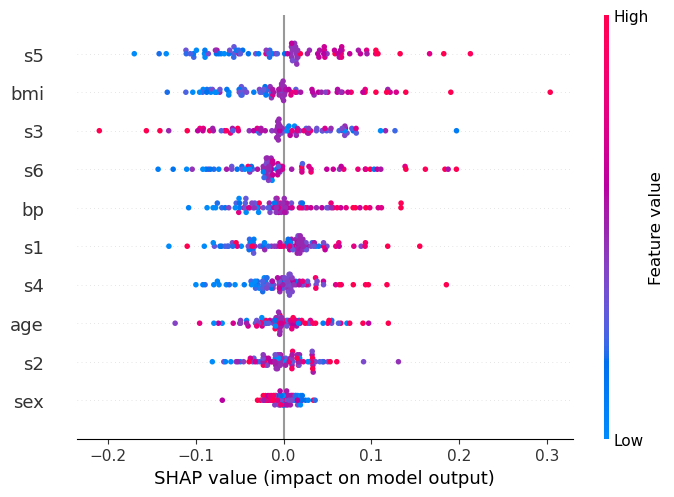}
        \caption{Mean imputation}
    \end{subfigure}    
    \begin{subfigure}{0.42\textwidth}
        \centering
        \includegraphics[width=\textwidth]{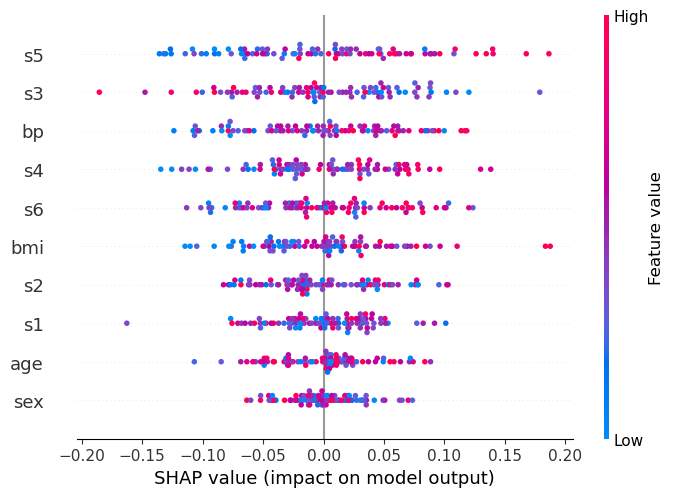}
        \caption{MICE}
    \end{subfigure}
    
    \begin{subfigure}{0.42\textwidth}
        \centering
        \includegraphics[width=\textwidth]{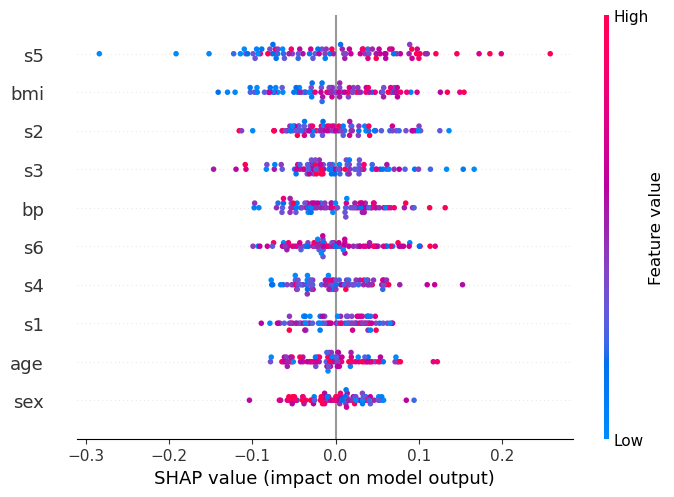}
        \caption{DIMV}
    \end{subfigure}    
    \begin{subfigure}{0.42\textwidth}
        \centering
        \includegraphics[width=\textwidth]{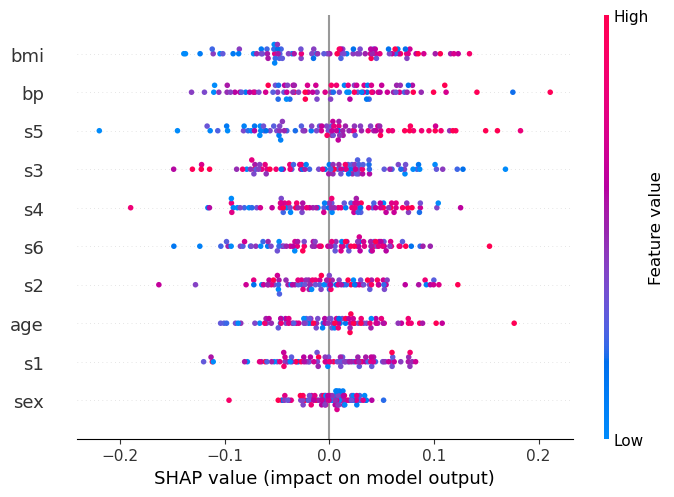}
        \caption{MissForest}
    \end{subfigure}
    
    \begin{subfigure}{0.42\textwidth}
        \centering
        \includegraphics[width=\textwidth]{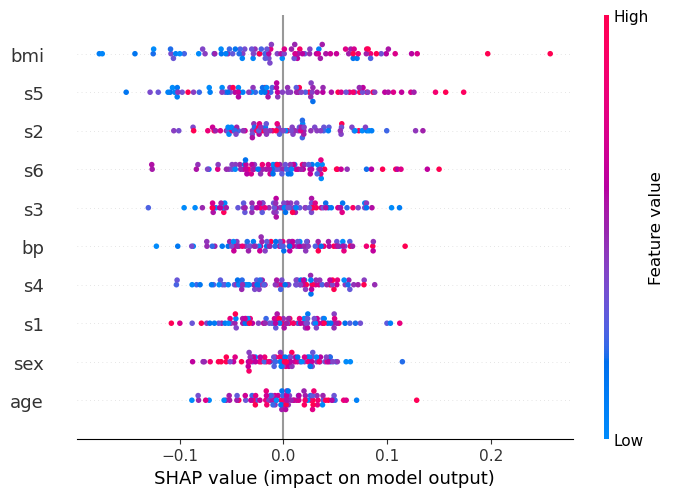}
        \caption{SOFT-IMPUTE}
    \end{subfigure}
    \begin{subfigure}{0.42\textwidth}
        \centering
        \includegraphics[width=\textwidth]{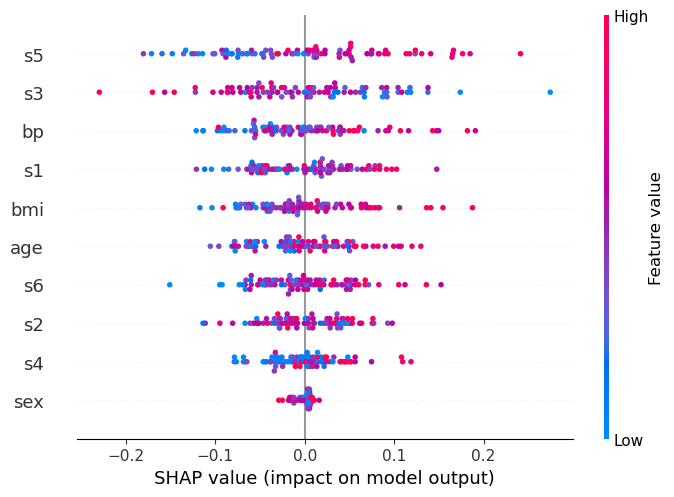}
        \caption{GAIN}
    \end{subfigure}

    \caption{Beeswarm plots for the Diabetes dataset at missing rate $r=0.8$ }
    \label{fig:diabetes-beeswarm_plots08}
\end{figure}

\subsection{MSE analysis}
\subsubsection{MSE analysis for regression tasks}
The performances of XGBoost without imputation and various imputation methods are evaluated for regression on the California and Diabetes datasets at different missing rates, as shown in Table~\ref{tab:MSE_LinReg}. Moreover,  recall that the MSE represents the mean squared error between the true and predicted values in the test set labels. Also, \MSESHAP is the MSE between the Shapley values of the model fitted based on imputed data versus the Shapley values of the model fitted based on the original data. 

\begin{table}[!htp]
\centering
\resizebox{\textwidth}{!}{
\begin{tabular}{c c c c c c c c c c}
\hline
Dataset & $r$ & Criteria & XGBoost & MI & MICE & DIMV & missForest & SOFT & GAIN \\
\hline
\multirow{8}{*}{California} & \multirow{2}{*}{0.2} & MSE & 0.223 & 0.254 & 0.239 & 0.239 & \textbf{0.174} & 0.263 & 0.317 \\
& & \MSESHAP & 0.354 & 0.175 & 0.077 & \textbf{0.054} & 0.061 & 0.096 & 0.217 \\ \cline{2-10}
& \multirow{2}{*}{0.4} & MSE & \textbf{0.238} & 0.250 & 0.271 & 0.268 & 0.239 & 0.262 & 0.300 \\
& & \MSESHAP & 0.345 & 0.129 & 0.082 & \textbf{0.067} & 0.093 & 0.100 & 0.142 \\ \cline{2-10}
& \multirow{2}{*}{0.6} & MSE & \textbf{0.240} & 0.253 & 0.270 & 0.270 & 0.243 & 0.263 & 0.308 \\
& & \MSESHAP & 0.362 & 0.134 & 0.090 & \textbf{0.076} & 0.097 & 0.110 & 0.149 \\ \cline{2-10}
& \multirow{2}{*}{0.8} & MSE & \textbf{0.231} & 0.243 & 0.263 & 0.260 & 0.233 & 0.254 & 0.293 \\
& & \MSESHAP & 0.346 & 0.136 & 0.090 & \textbf{0.073} & 0.098 & 0.106 & 0.151 \\ \hline

\multirow{8}{*}{Diabetes} & \multirow{2}{*}{0.2}& MSE &0.193 & 0.188 & 0.161 & \textbf{0.150} & 0.170 & 0.155 & 0.216\\
& &\MSESHAP & 0.063 & 0.019 & \textbf{0.016} & \textbf{0.016} & 0.017 & \textbf{0.016} & 0.021 \\ \cline{2-10}
&\multirow{2}{*}{0.4} & MSE &0.309 & 0.298 & 0.332 & \textbf{0.275} & 0.320 & 0.304 & 0.358 \\
&&\MSESHAP & 0.125 & 0.029 & 0.028 & \textbf{0.027} & 0.030 & 0.028 & 0.033\\ \cline{2-10}
&\multirow{2}{*}{0.6}&MSE &0.359 & 0.366 & 0.386 & 0.371 & 0.445 & \textbf{0.349} & 0.427 \\
&&\MSESHAP & 0.142 & 0.038 & 0.039 & \textbf{0.035} & 0.042 & 0.037 & 0.044\\ \cline{2-10}
&\multirow{2}{*}{0.8} & MSE & 0.549 & \textbf{0.535} & 0.563 & 0.574 & 0.660 & 0.555 & 0.603 \\
&& \MSESHAP & 0.224 & 0.049 & 0.051 & 0.050 & 0.058 & \textbf{0.048} & 0.055 \\
\hline

\end{tabular}
}
\caption{MSE on the test set labels and MSE of Shapley values for linear regression tasks on the California and Diabetes datasets.}
\label{tab:MSE_LinReg}
\end{table}

Taking the results on the California dataset from Table~\ref{tab:MSE_LinReg}, at $r=0.2$, missForest achieves the lowest MSE of 0.174, which is significantly lower than the second-best method, XGBoost, with an MSE of 0.223. However, when the missing rate exceeds 0.4, XGBoost consistently achieves the lowest MSE, followed closely by missForest. Interestingly, although missForest and XGBoost alternate in obtaining the lowest MSE, this does not correspond to the lowest \MSESHAP. Instead, DIMV demonstrates the best performance, indicating its robustness in preserving the explainability of the imputed data. For example, the \MSESHAP values for DIMV are 0.054, 0.067, 0.076, and 0.073, outperforming those of XGBoost, which are 0.354, 0.345, 0.362, and 0.346, as the missing rate increases from 0.2 to 0.8, respectively.

For the Diabetes dataset, overall, both MSE and \MSESHAP consistently increase as the missing rate increases. 
With a low missing rate (from 0.4 to lower), DIMV shows the best performance at both MSE and \MSESHAP. For example, at $r=0.2$, DIMV achieves the best MSE of 0.150, followed by SOFT with an MSE of 0.155. Also, DIMV, SOFT, and MICE demonstrate the best \MSESHAP, all of 0.016. As the missing rate increases, there is no consistency in the method that delivers the best results; instead, different methods alternately show their superiority. For instance, at $r=0.6$, SOFT outperforms the others in MSE with a value of 0.349, while DIMV achieves the best \MSESHAP at 0.035. At $r=0.8$, MI achieves the lowest MSE (0.535), while SOFT records the lowest \MSESHAP (0.048).

We, therefore, observe that different imputation methods excel in terms of prediction performance and explainability, underscoring the importance of selecting appropriate imputation methods based on specific criteria in the context of missing data analysis. The reason for the prominent performance of DIMV in terms of \MSESHAP could be that DIMV might better preserve the relationships between features and their contributions to the model's predictions compared to the other approaches. In fact, DIMV is based on the Gaussian conditional formula with \(L_2\) norm regularization. Meanwhile, missForest is a tree-based nonlinear approach, and SOFT-IMPUTE is a matrix decomposition approach to imputation. Also, the simplicity of DIMV and its regularization terms could reduce biases introduced by other imputation methods, leading to more accurate Shapley values.

\subsubsection{MSE analysis for classification tasks}
\begin{table}[!htp]
\centering
\resizebox{\textwidth}{!}{
\begin{tabular}{c c c c c c c c c c}
\hline
Dataset & $r$ & Criteria & XGBoost & MI & MICE & DIMV & missForest & SOFT & GAIN \\
\hline

\multirow{8}{*}{MNIST} & \multirow{2}{*}{0.2}&MSE & - &  0.004& 0.001& 0.001& \textbf{0.000}& 0.002& 0.004\\
&&\MSESHAP & 0.038& 0.021& 0.010& 0.009& \textbf{0.005}& 0.011& 0.019\\ \cline{2-10}
&\multirow{2}{*}{0.4}&MSE & - & 0.009& 0.004& 0.003& \textbf{0.002}& 0.005& 0.008\\
&&\MSESHAP & 0.074& 0.034& 0.020& 0.018& \textbf{0.011}& 0.023& 0.034\\ \cline{2-10}
&\multirow{2}{*}{0.6} & MSE & - & 0.013& 0.010& 0.006& \textbf{0.004}& 0.009& 0.014\\
&&\MSESHAP & 0.109& 0.046& 0.038& 0.030& \textbf{0.024}& 0.037& 0.045 \\ \cline{2-10}
&\multirow{2}{*}{0.8}&MSE & - & 0.017& 0.017& 0.012& \textbf{0.011}& 0.027& 0.020\\
&&\MSESHAP & 0.142& 0.053& 0.053& 0.044& \textbf{0.041}& 0.053& 0.052\\\hline

\multirow{6}{*}{Glass} & \multirow{2}{*}{0.2}&MSE & -  & 0.192 & \textbf{0.094} & 0.115 & \textbf{0.094} & 0.095 & 0.265 \\
&&\MSESHAP & 0.041 & 0.029 & \textbf{0.015} & 0.017 & 0.020 & 0.018 & 0.032  \\ \cline{2-10}
&\multirow{2}{*}{0.4}&MSE & - & 0.782 & 1.617 & 0.624 & 0.612 & \textbf{0.578} & 1.319 \\
&&\MSESHAP & 0.186 & 0.080 & 0.063 & 0.057 & 0.071 & \textbf{0.062} & 0.0965 \\ \cline{2-10}
&\multirow{2}{*}{0.6} & MSE & - & 0.706 & 1.258 & 0.675 & 0.707 & \textbf{0.623} & 0.924 \\
&&\MSESHAP & 0.149 & 0.064 & 0.059 & \textbf{0.050} & 0.060 & \textbf{0.050} & 0.070 \\\hline
\end{tabular}
}
\caption{MSE on the test set inputs and MSE of Shapley values for classification tasks on the MNIST and Glass datasets.}
\label{tab:MSE_Classification}
\end{table}

For the classification tasks on the MNIST and Glass datasets, note that the MSE is calculated based on the test set inputs, and the results for the missing rate of 0.8 are unavailable because some rows have all features missing. Focusing on the MNIST dataset, missForest consistently delivers the best MSE and \MSESHAP across all missing rates. Specifically, the MSE values are 0 - 0.002 - 0.004 - 0.011, and the \MSESHAP values are 0.005 - 0.011 - 0.024 - 0.041 as the missing rate increases from 0.2 to 0.8. On the Glass dataset, at the missing rate of 0.2, MICE outperforms the others with an MSE of 0.094 and a \MSESHAP of 0.015. When the missing rate increases from 0.4 to 0.6, SOFT takes the lead with an MSE of 0.578 and 0.623, and \MSESHAP values of 0.062 and 0.050, respectively. One important observation is that the XGBoost classifier gives the highest \MSESHAP across all missing rates, thereby illustrating its lower ability to preserve the explainability of the imputed data compared to other methods.

\section{Discussion}\label{sec-discuss}

In summary, from the present results and the deep analysis in the previous section, we can observe several interesting key insights:
\begin{itemize}
\item The global feature importance plots and the beeswarm plots clearly show that different imputation methods lead to varying Shapley value distributions. This shows that the choice of the imputation method can significantly affect the explainability of the model.

\item Across different missing rates, XGBoost, without imputation, seems to cause the most significant changes in the Shapley values. While XGBoost can train and predict directly on data with missing values, the MSE between Shapley values of XGBoost and the Original is the highest in most of the experiments.

\item As the missing rate increases from 0.2 to 0.8, the differences between the imputation methods become more pronounced. This indicates that the choice of the imputation method becomes increasingly critical as the number of missing data increases.

\item The results for the California and Diabetes datasets show some differences, suggesting that the impact of imputation methods may be data set dependent. This highlights the importance of considering the characteristics of the data set when choosing an imputation method. Such a research gap should be further investigated.

\item The MSE plots show that methods with lower imputation MSE do not necessarily preserve Shapley values better (as measured by Shapley MSE). This suggests a potential trade-off between the accuracy of the imputation and the maintenance of the original importance structure of the feature. Interestingly, the results by \cite{schroeder2024interpretable} show that the methods that achieved the best MSE did not necessarily yield the most interpretable rule-based ML models, in the sense of model size and rule recovery.  

\item Mean imputation (MI) tends to significantly alter the importance of features, especially at higher missing rates.

\item MICE and DIMV often show similar patterns, possibly due to the fact that MICE is based on regression and DIMV is based on a conditional Gaussian formula. MissForest and SOFT-IMPUTE sometimes preserve feature rankings better than simpler methods, but this is not consistent across all scenarios. GAIN does not show its superiority in all comparison datasets.
\end{itemize}
The variability in results between methods and missing rates underscores the need to evaluate imputation effects when using Shapley values for model interpretation for datasets with missing values. The following discussion is structured around our result and the specific pitfalls that may arise due to an incomplete understanding of the relationship between missing data, imputation methods, and Shapley values. We highlight how different approaches can lead to vastly different interpretations, how dataset characteristics and missing rates affect results, and the importance of considering both imputation accuracy and explainability preservation:

\begin{itemize}
    \item 
\textbf{Pitfall 1: Assume the neutrality of the imputation method.}
Our study reveals that different imputation methods can significantly alter Shapley values and, consequently, the explainability of the model. For instance, mean imputation and GAIN tend to distort feature importance, especially at higher missing rates, while methods like MICE and DIMV often show similar patterns. This underscores the importance of carefully considering the imputation method when using Shapley values for model explanation, as the choice can lead to vastly different interpretations of feature importance.

\item \textbf{Pitfall 2: Overlooking data set dependency. }
We observed that the effects of imputation methods on Shapley values vary between data sets. For example, the California and Diabetes datasets showed different patterns of importance of features in different imputation methods. This highlights that dataset characteristics play a crucial role in determining the best imputation approach, and warns against applying a one-size-fits-all solution across different datasets. Future work should address the question of what dataset characteristics are the most influential for the imputation.

\item \textbf{Pitfall 3: Ignoring the impact of missing rate. }
Our results demonstrate that the impact of the choice of the imputation method becomes more pronounced as the missing rate increases. At lower missing rates, differences between methods are less stark, but as missing data increases, the choice of the imputation method becomes increasingly critical. This emphasizes the need for more careful consideration of imputation techniques when dealing with data sets with substantial missing data. In addition, a method may perform better than another one at low missing rates but perform worse than another one at high missing rates, such as DIMV in the Diabetes dataset or MICE in the Glass dataset.

\item \textbf{Pitfall 4: Focusing solely on imputation accuracy. } 
Our findings suggest a potential trade-off between imputation precision and the preservation of Shapley values. Methods that provide more accurate imputations do not necessarily better preserve the original Shapley values. This highlights a potential tension between optimizing for prediction accuracy and maintaining explainability and suggests that practitioners should consider both aspects when selecting an imputation method.
\end{itemize}

These pitfalls underscore the complex interaction between missing data handling, imputation methods, and model explainability using Shapley values. Our findings highlight that there is no universal solution for handling missing data while preserving the explainability of the model. Instead, the choice of method should be context-dependent, considering factors such as data set characteristics, missing data rates, and the specific requirements of the analysis. Moreover, our results emphasize the need for a holistic approach that balances imputation accuracy with the preservation of feature-importance structures. As machine learning models continue to be applied in critical domains, understanding and addressing these pitfalls becomes important to ensuring reliable and interpretable results. As an additional suggestion for future work, it would be useful to develop explanation methods that are aware of whether some features in the training set have been imputed at training time and how often. Features that have been imputed more often are expected to be less reliable and, therefore, have relatively low importance.

\section{Conclusion}\label{sec-conclusion}

In this paper, we explore the impact of various imputation methods on the calculation of Shapley values for model explainability. Our findings indicate that the choice of the imputation strategy can significantly influence the accuracy and reliability of Shapley values, thereby affecting the insights drawn from the machine learning models. 

Our comparative analysis revealed that the chosen imputation method should align with the specific characteristics of the data set and the objectives of the analysis. Practitioners should carefully consider the trade-offs between computational efficiency and the potential for bias introduction when selecting an imputation method. Moreover, our results revealed that imputation techniques that achieve the best accuracy do not necessarily produce the most accurate feature importance values. The study underscores the necessity of evaluating the effects of imputation as a critical step in the preprocessing pipeline, especially when Shapley values are used for model interpretation.

Future work should focus on extending this analysis to a broader range of datasets and machine learning models to validate our findings further. In addition, the development of new imputation methods is tailored for specific types of data, model structures, and explainability, especially in relation to Shapley value interpretations. Also, more research needs to be done on the direction of handling missing data directly, as this helps to avoid noises and bias introduced to the model by imputation. By addressing these challenges, we can improve the reliability of model interpretations and support more informed decision-making in the application of machine learning. 
\clearpage
\bibliographystyle{elsarticle-num-names} 
\bibliography{ref}
\newpage
\appendix
\section{Global feature importance plots on the Diabetes dataset}\label{appendix-importance-plots}
\begin{figure}[!htp]
    \centering
    \begin{subfigure}{0.32\textwidth}
        \centering
        \includegraphics[width=\textwidth]{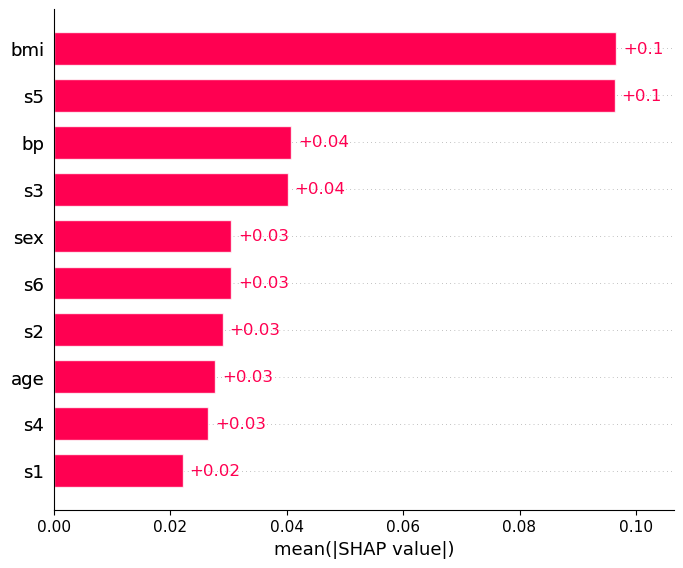}
        \caption{Original (LRO)}
        \label{fig:diabetes-ori02}
    \end{subfigure}
    \begin{subfigure}{0.32\textwidth}
        \centering
        \includegraphics[width=\textwidth]{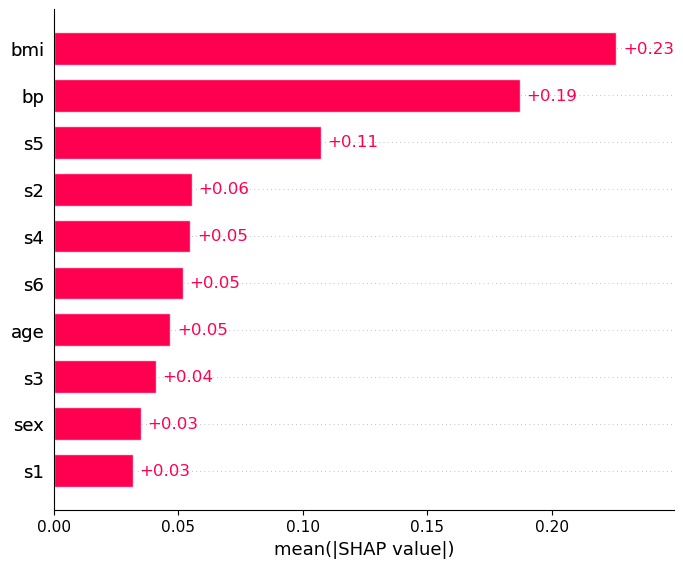}
        \caption{XGBoost without imputation}
        \label{fig:diabetes-xm02}
    \end{subfigure}
    \begin{subfigure}{0.32\textwidth}
        \centering
        \includegraphics[width=\textwidth]{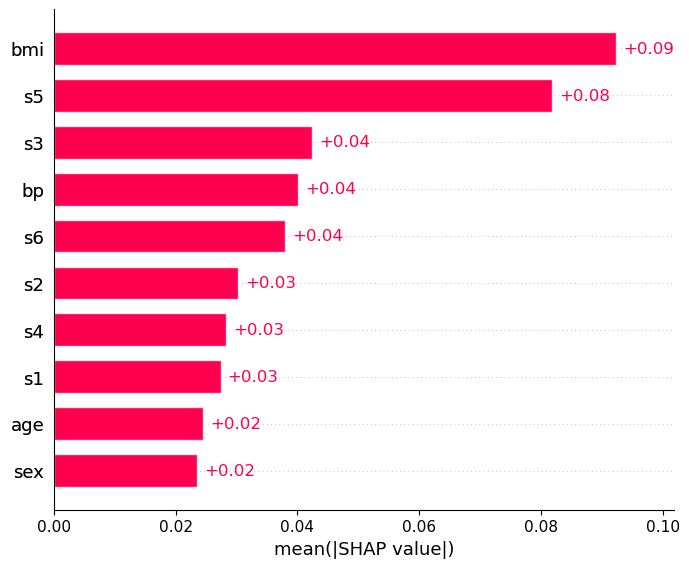}
        \caption{Mean Imputation}
        \label{fig:diabetes-mi02}
    \end{subfigure}
    
    \begin{subfigure}{0.32\textwidth}
        \centering
        \includegraphics[width=\textwidth]{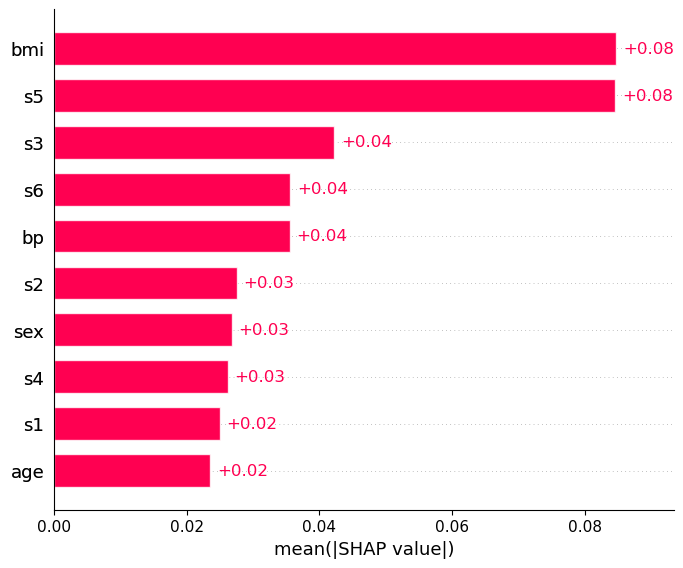}
        \caption{MICE}
        \label{fig:diabetes-mice02}
    \end{subfigure}
    \begin{subfigure}{0.32\textwidth}
        \centering
        \includegraphics[width=\textwidth]{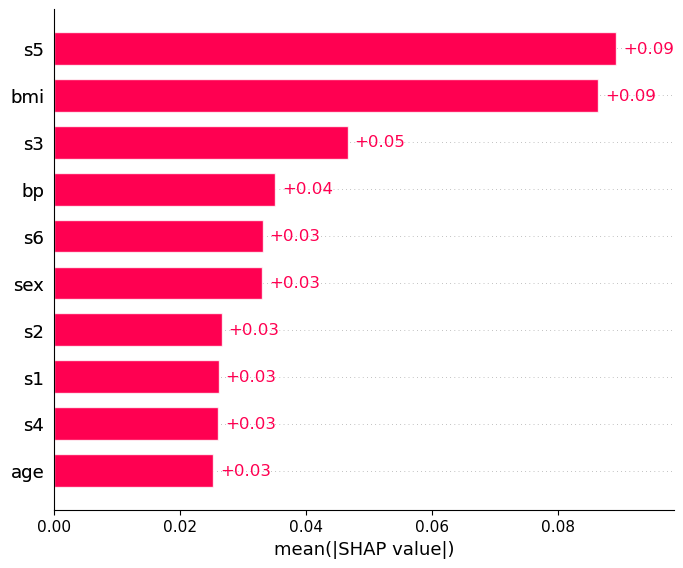}
        \caption{DIMV}
        \label{fig:diabetes-dimv02}
    \end{subfigure}
    \begin{subfigure}{0.32\textwidth}
        \centering
        \includegraphics[width=\textwidth]{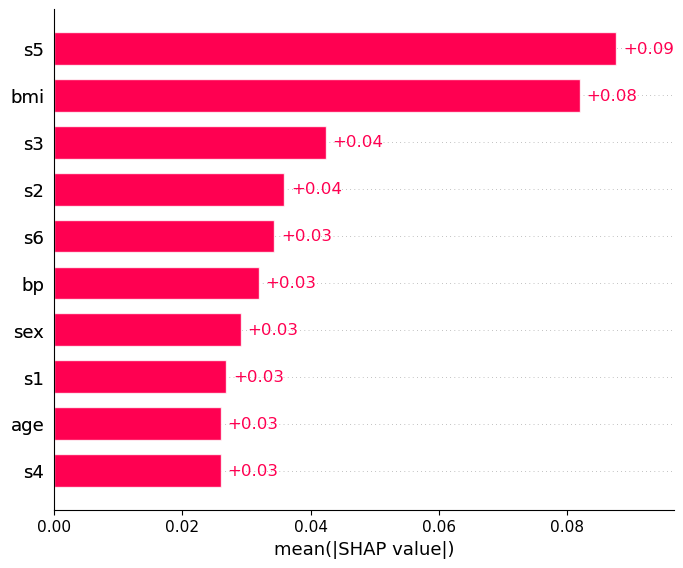}
        \caption{missForest}
        \label{fig:diabetes-mf02}
    \end{subfigure}
    \begin{subfigure}{0.32\textwidth}
        \centering
        \includegraphics[width=\textwidth]{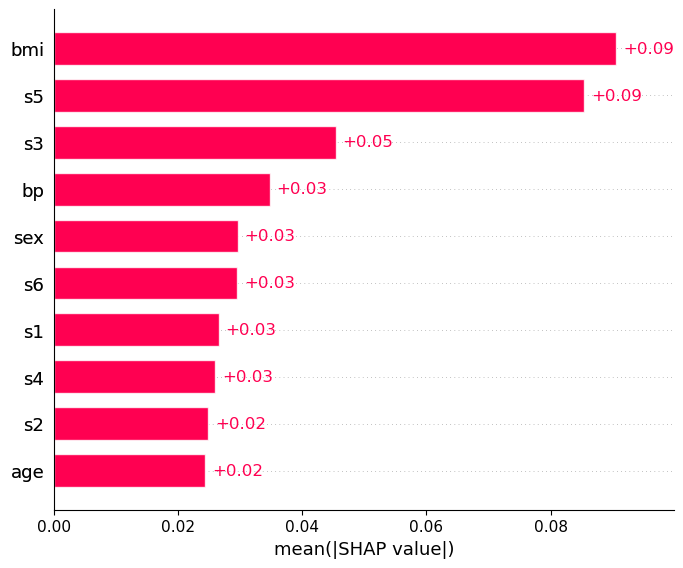}
        \caption{SOFT-IMPUTE}
        \label{fig:diabetes-soft02}
    \end{subfigure}
    \begin{subfigure}{0.32\textwidth}
        \centering
        \includegraphics[width=\textwidth]{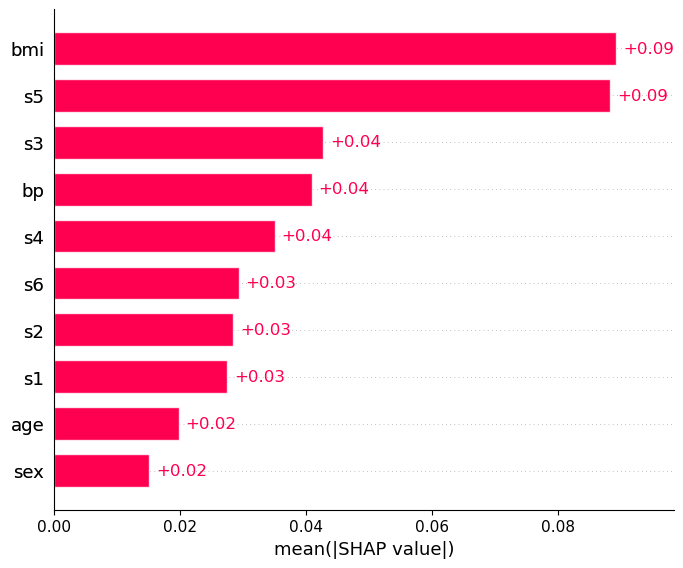}
        \caption{GAIN}
        \label{fig:diabetes-gain02}
    \end{subfigure}    

    \caption{Global feature importance plot on the Diabetes dataset with the missing rate $r=0.2$}
    \label{fig:diabetes-r2-bar}
\end{figure}

\begin{figure}[!htp]
    \centering
    \begin{subfigure}{0.32\textwidth}
        \centering
        \includegraphics[width=\textwidth]{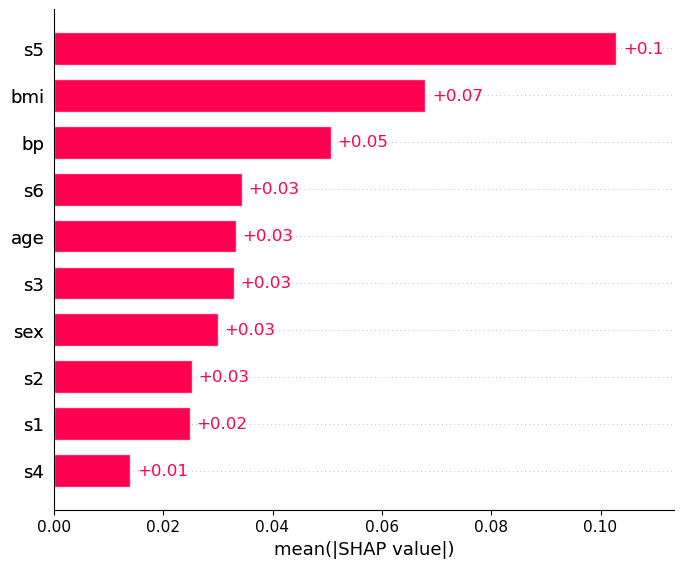}
        \caption{Original (LRO)}
        \label{fig:diabetes-ori04}
    \end{subfigure}
    \begin{subfigure}{0.32\textwidth}
        \centering
        \includegraphics[width=\textwidth]{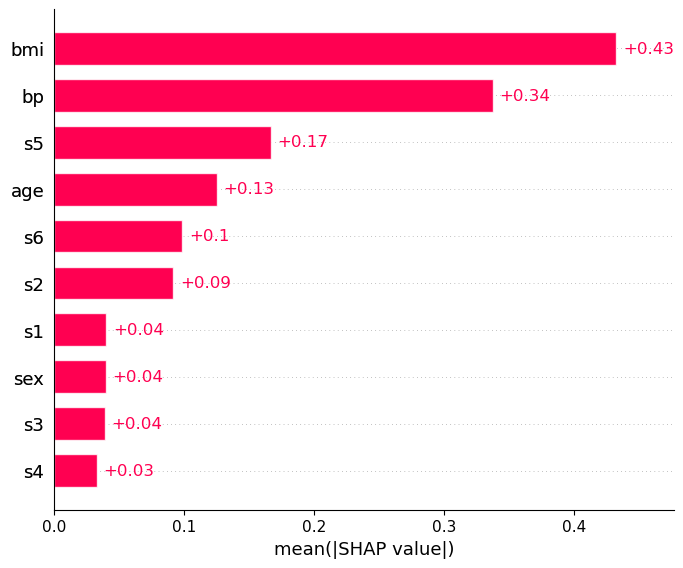}
        \caption{XGBoost without imputation}
        \label{fig:diabetes-xm04}
    \end{subfigure}
    \begin{subfigure}{0.32\textwidth}
        \centering
        \includegraphics[width=\textwidth]{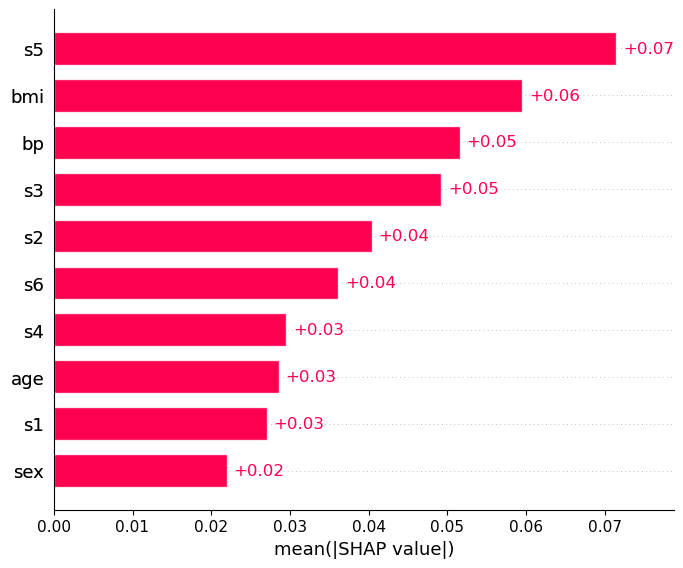}
        \caption{Mean Imputation}
        \label{fig:diabetes-mi04}
    \end{subfigure}
    
    \begin{subfigure}{0.32\textwidth}
        \centering
        \includegraphics[width=\textwidth]{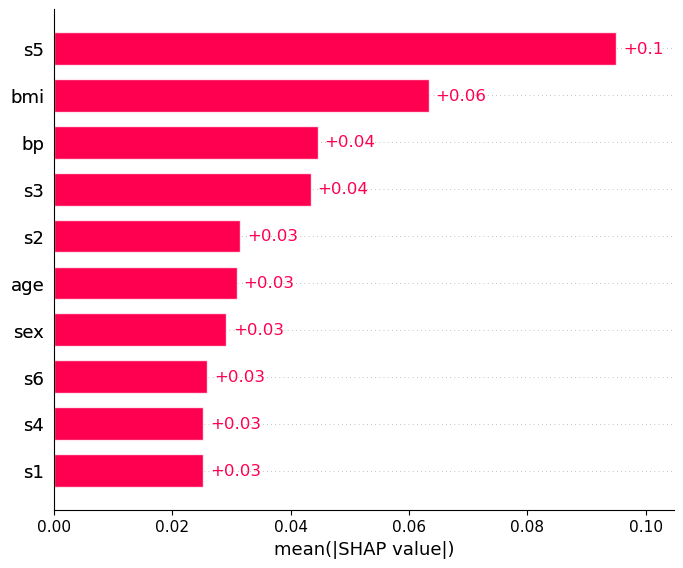}
        \caption{MICE}
        \label{fig:diabetes-mice04}
    \end{subfigure}
    \begin{subfigure}{0.32\textwidth}
        \centering
        \includegraphics[width=\textwidth]{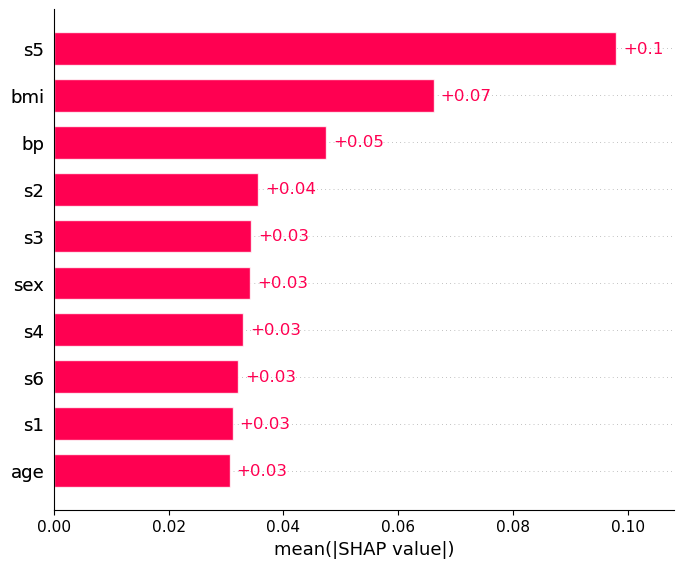}
        \caption{DIMV}
        \label{fig:diabetes-dimv04}
    \end{subfigure}
    \begin{subfigure}{0.32\textwidth}
        \centering
        \includegraphics[width=\textwidth]{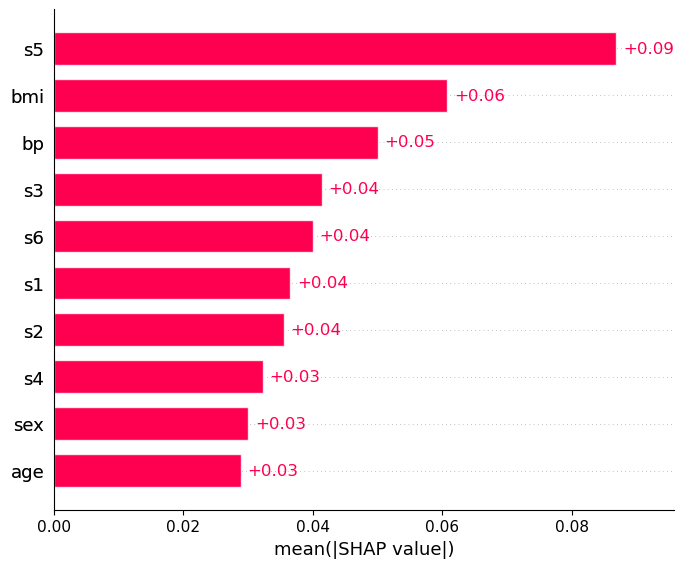}
        \caption{missForest}
        \label{fig:diabetes-mf04}
    \end{subfigure}
    \begin{subfigure}{0.32\textwidth}
        \centering
        \includegraphics[width=\textwidth]{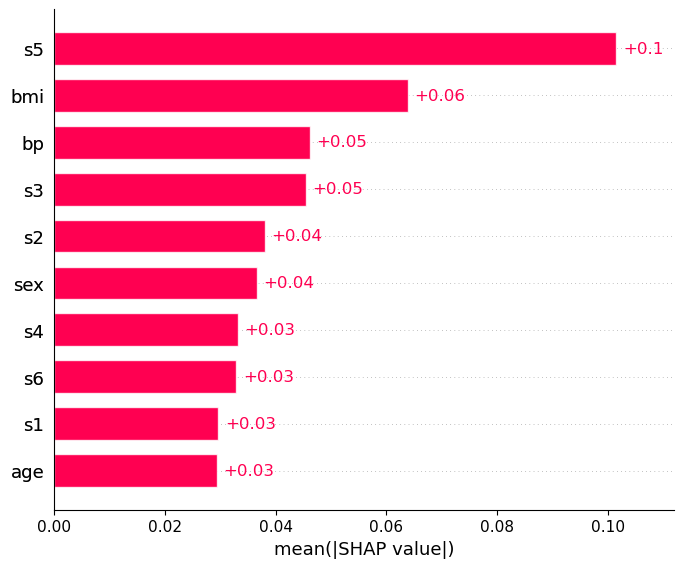}
        \caption{SOFT-IMPUTE}
        \label{fig:diabetes-soft04}
    \end{subfigure}
    \begin{subfigure}{0.32\textwidth}
        \centering
        \includegraphics[width=\textwidth]{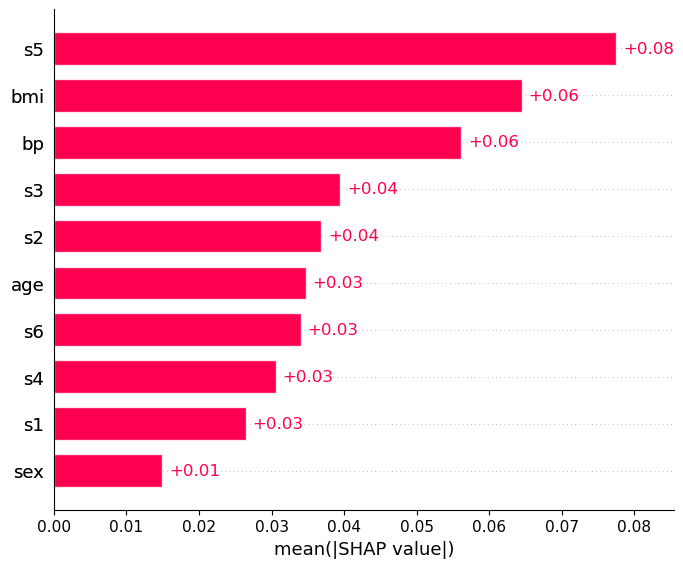}
        \caption{GAIN}
        \label{fig:diabetes-gain04}
    \end{subfigure}
    \caption{Global feature importance plot on the diabetes dataset with the missing rate $r=0.4$}
    \label{fig:diabetes-r4-bar}
\end{figure}

\begin{figure}[!htp]
    \centering
    \begin{subfigure}{0.32\textwidth}
        \centering
        \includegraphics[width=\textwidth]{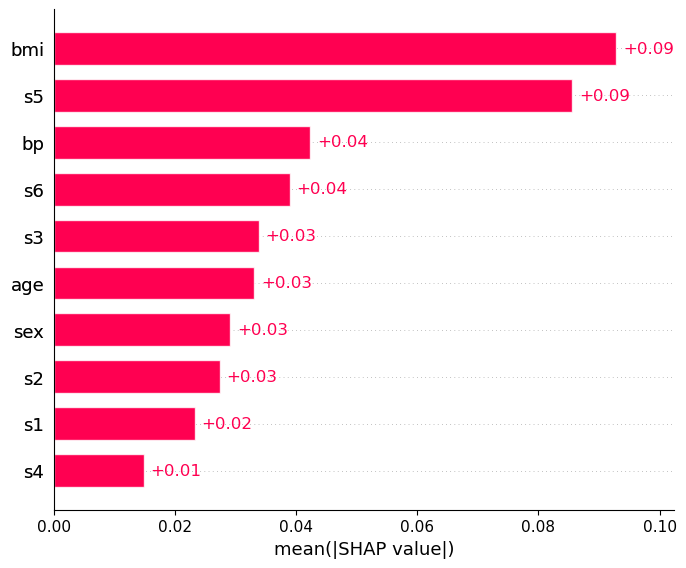}
        \caption{Original (LRO)}
        \label{fig:diabetes-ori06}
    \end{subfigure}
    \begin{subfigure}{0.32\textwidth}
        \centering
        \includegraphics[width=\textwidth]{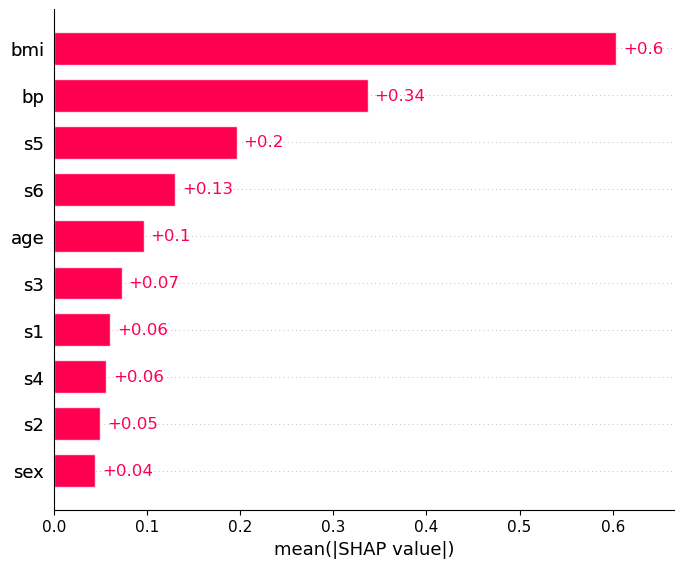}
        \caption{XGBoost without imputation}
        \label{fig:diabetes-xm06}
    \end{subfigure}
    \begin{subfigure}{0.32\textwidth}
        \centering
        \includegraphics[width=\textwidth]{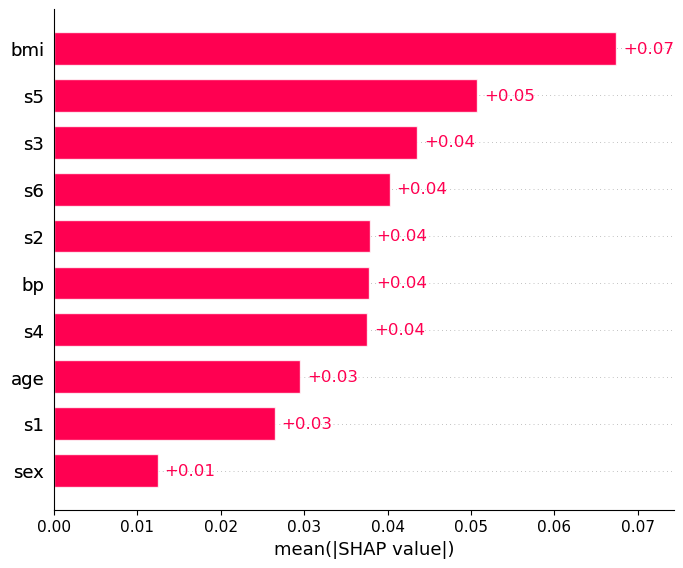}
        \caption{Mean Imputation}
        \label{fig:diabetes-mi06}
    \end{subfigure}
    
    \begin{subfigure}{0.32\textwidth}
        \centering
        \includegraphics[width=\textwidth]{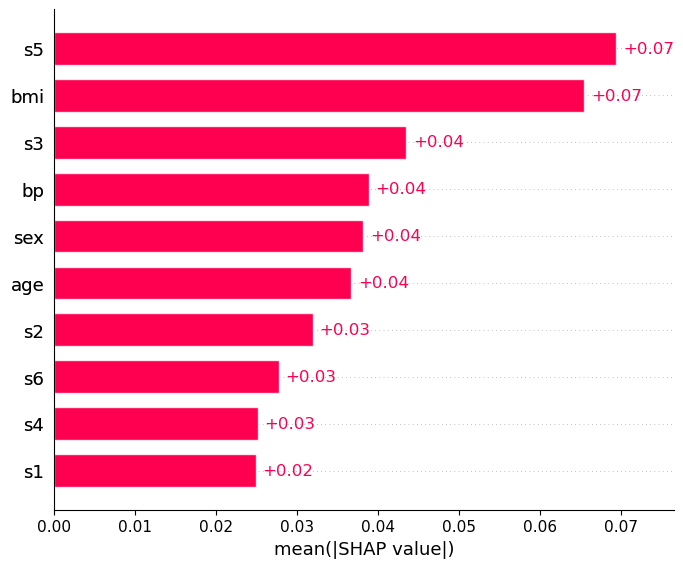}
        \caption{MICE}
        \label{fig:diabetes-mice06}
    \end{subfigure}
    \begin{subfigure}{0.32\textwidth}
        \centering
        \includegraphics[width=\textwidth]{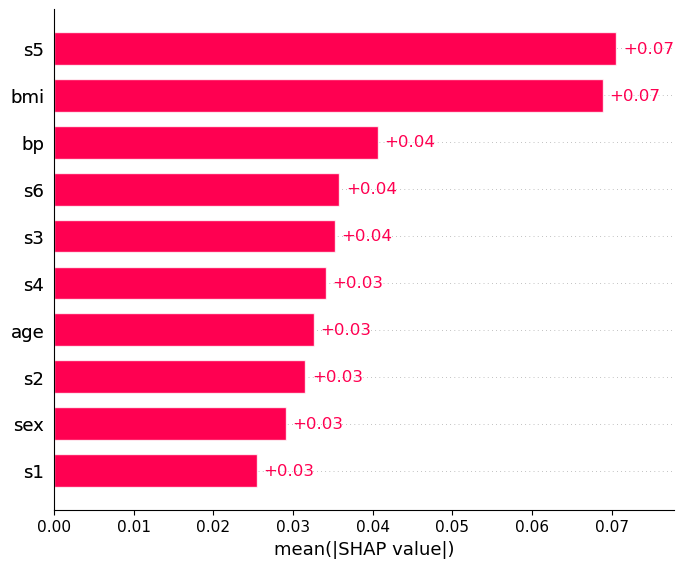}
        \caption{DIMV}
        \label{fig:diabetes-dimv06}
    \end{subfigure}
    \begin{subfigure}{0.32\textwidth}
        \centering
        \includegraphics[width=\textwidth]{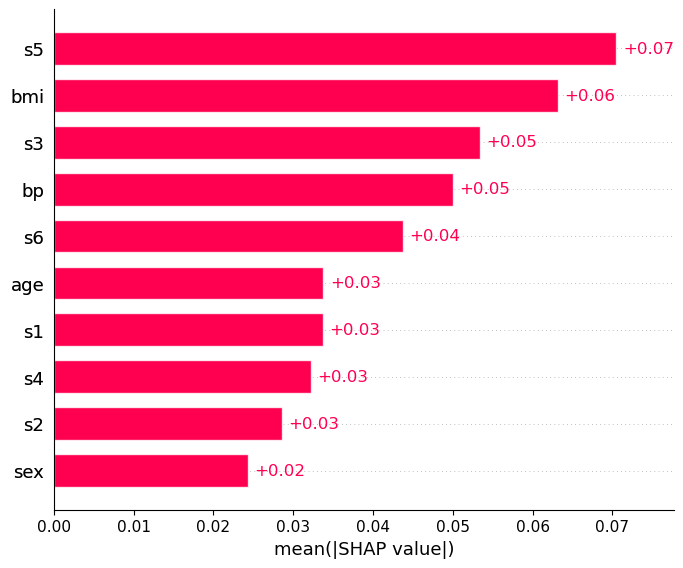}
        \caption{missForest}
        \label{fig:diabetes-mf06}
    \end{subfigure}
    \begin{subfigure}{0.32\textwidth}
        \centering
        \includegraphics[width=\textwidth]{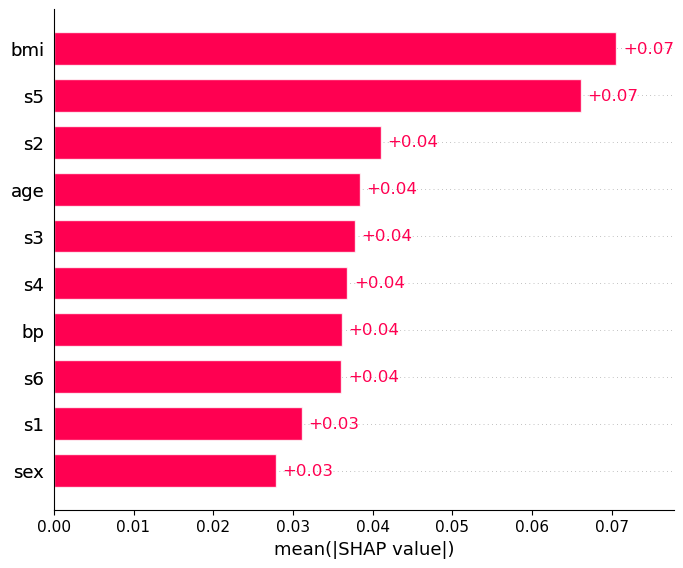}
        \caption{SOFT-IMPUTE}
        \label{fig:diabetes-soft06}
    \end{subfigure}
    \begin{subfigure}{0.32\textwidth}
        \centering
        \includegraphics[width=\textwidth]{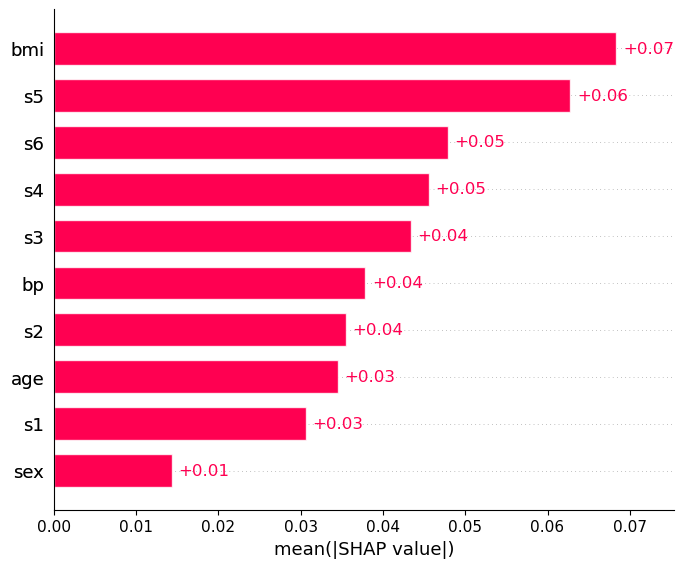}
        \caption{GAIN}
        \label{fig:diabetes-gain06}
    \end{subfigure}
    \caption{Global feature importance plot on the diabetes dataset with the missing rate $r=0.6$}
    \label{fig:diabetes-r6-bar}
\end{figure}

\begin{figure}[!htp]
    \centering
    \begin{subfigure}{0.32\textwidth}
        \centering
        \includegraphics[width=\textwidth]{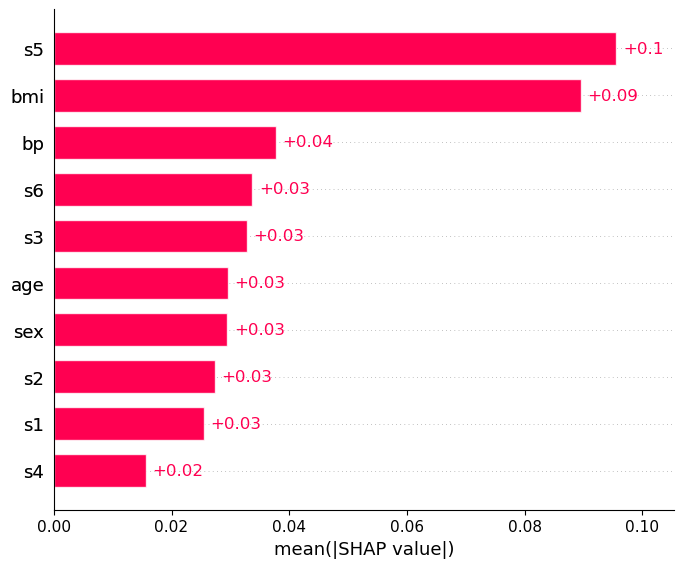}
        \caption{Original (LRO)}
        \label{fig:diabetes-ori08}
    \end{subfigure}
    \begin{subfigure}{0.32\textwidth}
        \centering
        \includegraphics[width=\textwidth]{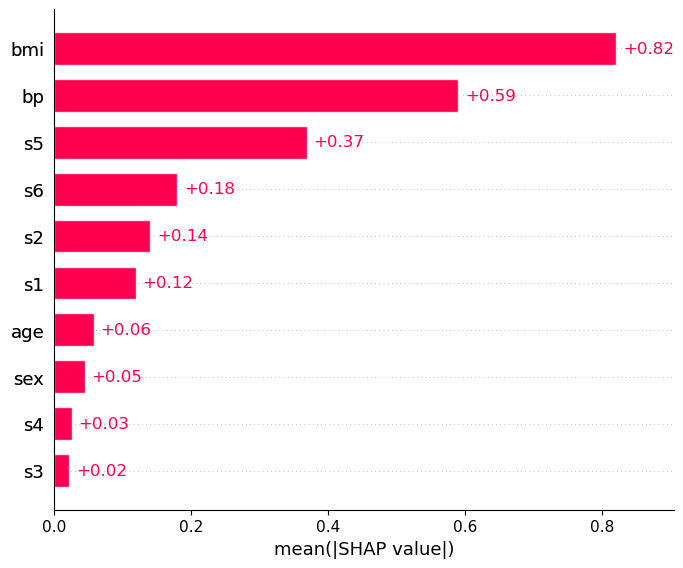}
        \caption{XGBoost without imputation}
        \label{fig:diabetes-xm08}
    \end{subfigure}
    \begin{subfigure}{0.32\textwidth}
        \centering
        \includegraphics[width=\textwidth]{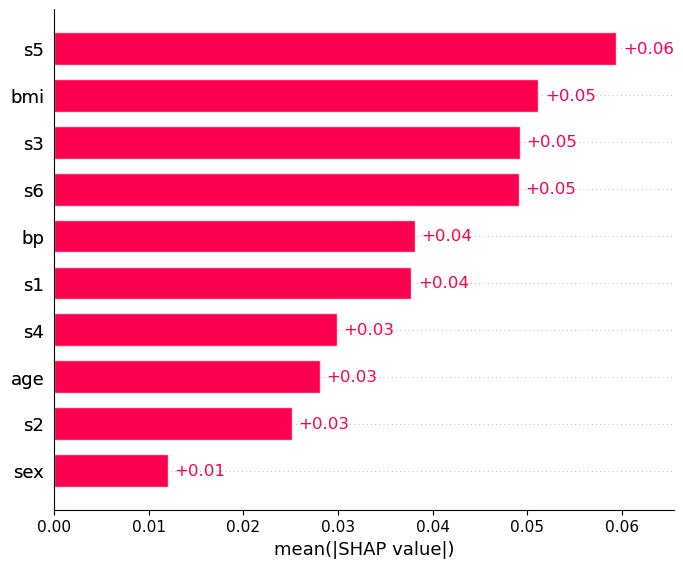}
        \caption{Mean Imputation}
        \label{fig:diabetes-mi08}
    \end{subfigure}
    
    \begin{subfigure}{0.32\textwidth}
        \centering
        \includegraphics[width=\textwidth]{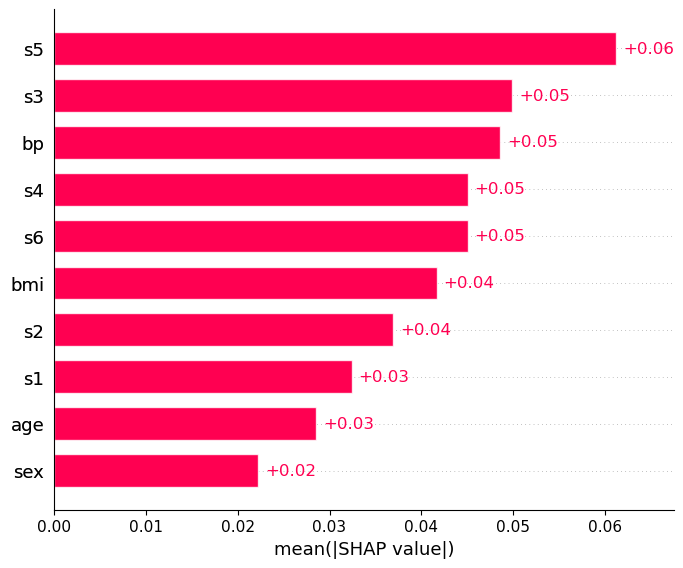}
        \caption{MICE}
        \label{fig:diabetes-mice08}
    \end{subfigure}
    \begin{subfigure}{0.32\textwidth}
        \centering
        \includegraphics[width=\textwidth]{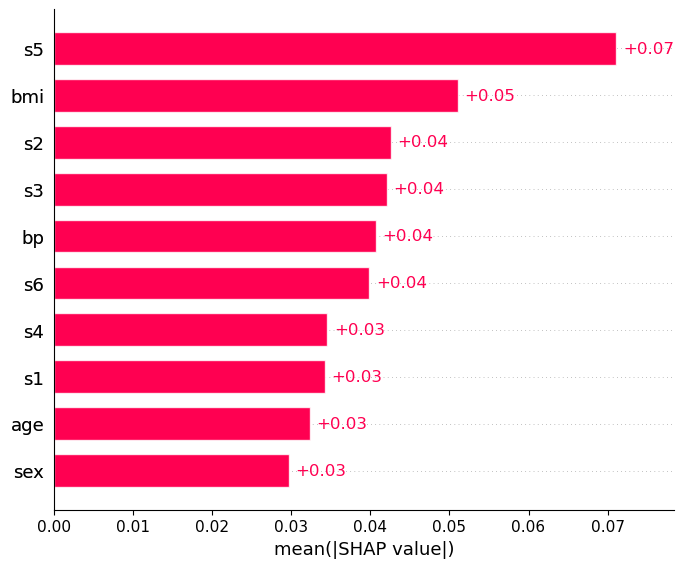}
        \caption{DIMV}
        \label{fig:diabetes-dimv08}
    \end{subfigure}
    \begin{subfigure}{0.32\textwidth}
        \centering
        \includegraphics[width=\textwidth]{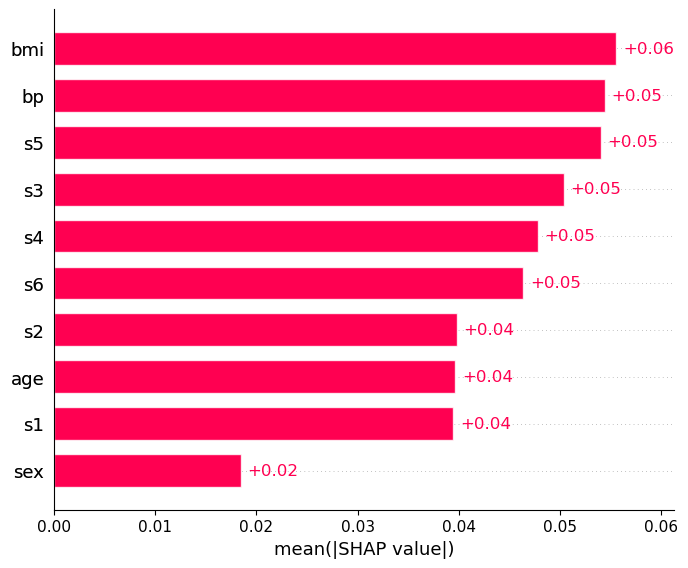}
        \caption{missForest}
        \label{fig:diabetes-mf08}
    \end{subfigure}
    \begin{subfigure}{0.32\textwidth}
        \centering
        \includegraphics[width=\textwidth]{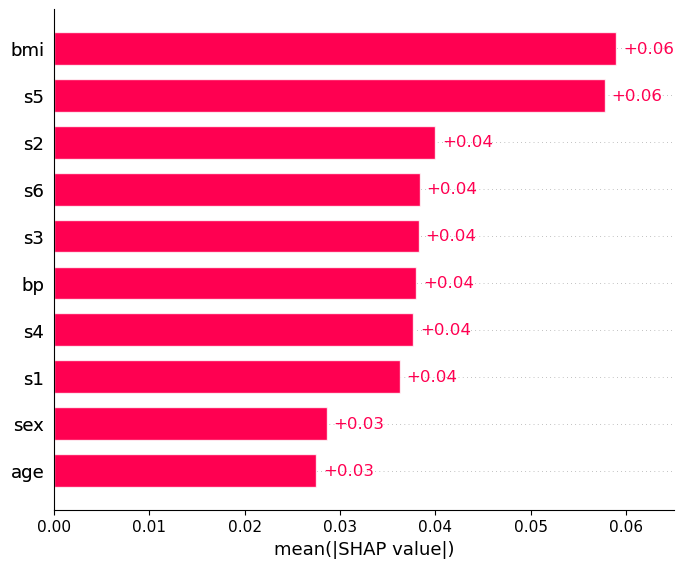}
        \caption{SOFT-IMPUTE}
        \label{fig:diabetes-soft08}
    \end{subfigure}
    \begin{subfigure}{0.32\textwidth}
        \centering
        \includegraphics[width=\textwidth]{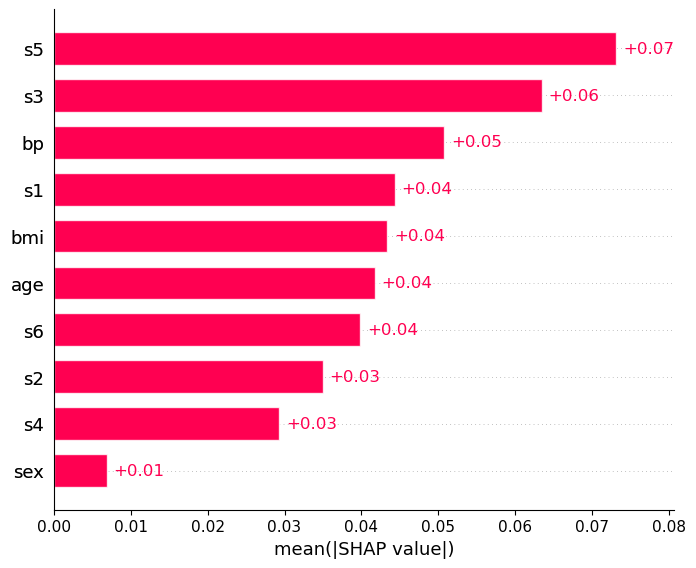}
        \caption{GAIN}
        \label{fig:diabetes-gain08}
    \end{subfigure}
    
    \caption{Global feature importance plot on the diabetes dataset with the missing rate $r=0.8$}
    \label{fig:diabetes-r8-bar}
\end{figure}

\section{Impact of a missing data point on the covariance}
\label{appendix-derivation-cov-diff}

To obtain the impact of one data point being missing in a dataset and imputed by the mean, let $i$ denote the index of the missing data point. With $\bar{\textbf{x}}$ and $\bar{\textbf{x}}^\prime$ denoting respectively the mean in the full and the imputed dataset, the corresponding covariances are respectively broken down as follows:
\newcommand{\cov}{\text{Cov}}
\begin{align}
    \cov(\textbf{x}, \textbf{y}) &= \frac{1}{N}\sum_{j = 1}^{N}(x_j - \bar{x})(y_j - \bar{y}) \nonumber \\
    &= \frac{1}{N}\left((x_i - \bar{x})(y_i - \bar{y}) + \underset{j\neq i}{\sum_{j = 1}^{N}}(x_j - \bar{x})(y_j - \bar{y})\right),
    \label{eq:cov}
\end{align}

\begin{align}
\cov(\textbf{x}^\prime, \textbf{y}) & =
    \frac{1}{N}\sum_{j = 1}^{N}(x_j - \bar{x}^\prime)(y_j - \bar{y}) 
    \nonumber \\
    &= \frac{1}{N}\left(
        (\bar{x}^\prime - \bar{x}^\prime)(y_i - \bar{y}) + \underset{j\neq i}{\sum_{j = 1}^{N}}(x_j - \bar{x}^\prime)(y_j - \bar{y})\right) 
    \nonumber \\ 
    & = \frac{1}{N}\left( \underset{j\neq i}{\sum_{j = 1}^{N}}(x_j - \bar{x}^\prime)(y_j - \bar{y})
    \right),
    \label{eq:cov_prime}
\end{align}
where the last equality holds because, in the imputed dataset, $x_i$ has been replaced with $\bar{\textbf{x}}^\prime$. 
The averages are defined as:
\begin{equation*}
    \bar{\textbf{x}} = \frac{1}{N} {\sum_{j = 1}^{N}} x_j = \frac{1}{N} \left( x_i +
\underset{j\neq i}{\sum_{j = 1}^{N}} x_j
\right) \text{ and } \bar{\textbf{x}}^\prime = \frac{1}{N-1}\left( \underset{j\neq i}{\sum_{j = 1}^{N}} x_j
\right).
\end{equation*}
Then one can derive the following equalities:
\begin{equation}
\bar{\textbf{x}} - \bar{\textbf{x}}^\prime = \frac{1}{N} x_i + \left( \frac{1}{N}-\frac{1}{N-1}\right)\underset{j\neq i}{\sum_{j = 1}^{N}} x_j = \frac{1}{N}  (x_i - \bar{\textbf{x}}^\prime),
\label{eq:x_bar_diff}
\end{equation}
\begin{equation}
\bar{\textbf{x}} = \frac{1}{N} x_i + \bar{\textbf{x}}^\prime - \frac{1}{N} \bar{\textbf{x}}^\prime =  \frac{1}{N} x_i + \frac{N-1}{N} \bar{\textbf{x}}^\prime.
\label{eq:x_bar}
\end{equation}
The impact of $x_i$ being missing in the dataset and imputed by the mean is the result of subtracting \eqref{eq:cov} from \eqref{eq:cov_prime}. Substituting \eqref{eq:x_bar_diff} yields:
\begin{align}
\cov(\textbf{x}^\prime, \textbf{y}) - \cov (\textbf{x}, \textbf{y})   & = 
\frac{1}{N}\left(
    - (x_i - \bar{\textbf{x}})(y_i - \bar{\textbf{y}}) + \underset{j\neq i}{\sum_{j = 1}^{N}} (\bar{\textbf{x}} - \bar{\textbf{x}}^\prime) (y_j - \bar{\textbf{y}})
\right)
\nonumber \\
& = \frac{1}{N}\left(
    (y_i - \bar{\textbf{y}}) (\bar{\textbf{x}}-x_i) 
    + \frac{1}{N} (x_i-\bar{\textbf{x}}^\prime)
    \underset{j\neq i}{\sum_{j = 1}^{N}}
    (y_j - \bar{\textbf{y}})
\right).
\label{eq:cov_diff}
\end{align}
Note also that: 
\begin{equation}
    \underset{j\neq i}{\sum_{j = 1}^{N}}
    (y_j - \bar{\textbf{y}}) = N\bar{\textbf{y}} - y_i - (N-1)\bar{\textbf{y}} = \bar{\textbf{y}} - y_i.
    \label{eq:y_y_bar}
\end{equation}
Substituting \eqref{eq:y_y_bar} into \eqref{eq:cov_diff}:
$$
\cov(\textbf{x}^\prime, \textbf{y}) - \cov (\textbf{x}, \textbf{y})  = 
\frac{1}{N} (y_i - \bar{\textbf{y}})\left(
    \frac{1}{N}(\bar{\textbf{x}}^\prime - x_i) + \bar{\textbf{x}} - x_i.
\right)
$$
Finally, substituting \eqref{eq:x_bar} into the latter equation:
\begin{align}
    \cov(\textbf{x}^\prime, \textbf{y}) - \cov (\textbf{x}, \textbf{y}) 
    & = \frac{1}{N} (y_i - \bar{\textbf{y}})\left(
        \frac{1}{N}\bar{\textbf{x}}^\prime + \frac{N-1}{N} \bar{\textbf{x}}^\prime - x_i
    \right)
    \nonumber \\
    & = \frac{1}{N} (y_i - \bar{\textbf{y}})(\bar{\textbf{x}}^\prime - x_i),
\end{align}
yields the sought result.

\end{document}